\documentclass{article}

\usepackage{iclr2027_conference,times}

\usepackage[utf8]{inputenc}
\usepackage[T1]{fontenc}
\usepackage{amsfonts}
\usepackage{amsmath}
\usepackage{amssymb}
\usepackage{colortbl}
\usepackage{amsthm}
\usepackage{booktabs}
\usepackage{longtable}
\usepackage{fancyvrb}
\usepackage{etoolbox}
\usepackage{graphicx}
\usepackage{hyperref}
\usepackage{mathtools}
\usepackage{microtype}
\usepackage{tabularx}
\usepackage{url}
\usepackage{xcolor}
\usepackage{array}
\usepackage{float}
\usepackage{placeins}
\usepackage{algorithm}
\usepackage{algorithmic}
\usepackage{paralist}

\definecolor{artifactlink}{rgb}{0.10,0.25,0.55}
\newcommand{\micon}[1]{%
  \IfFileExists{icons/#1}{%
    \raisebox{-0.25\height}{\includegraphics[height=1em]{icons/#1}}%
  }{}%
}
\newcommand{\artifactlink}[2]{\href{#1}{\textcolor{artifactlink}{\textbf{#2}}}}
\newcommand{\artifactfootnote}[1]{%
  \begingroup
  \renewcommand{\thefootnote}{}%
  \footnotetext{#1}%
  \endgroup
}
\hypersetup{
  hypertexnames=false,
  hidelinks,
  pdftitle={Measuring and Mitigating Solution Mode Collapse in RLVR},
  pdfauthor={Liv G. d'Aliberti, Marwa Abdulhai, Sofiia Druchyna, Peter Henderson, Manoel Horta Ribeiro}
}
\newcommand{\enforcehalffinalproseline}{%
  \setlength{\parfillskip}{0pt plus .20\linewidth}}
\AtBeginDocument{\enforcehalffinalproseline}
\AtBeginEnvironment{abstract}{\enforcehalffinalproseline}
\makeatletter
\apptocmd{\@lbibitem}{\looseness=-2}{}{}
\apptocmd{\@bibitem}{\looseness=-2}{}{}
\makeatother

\newcolumntype{Y}{>{\raggedright\arraybackslash}X}
\newcolumntype{P}[1]{>{\raggedright\arraybackslash}p{#1}}

\newcommand{\E}{\mathbb E}

\usepackage{wrapfig}
\newcommand{\mb}{\textsc{ModeBench}}
\newcommand{\pmd}{\texttt{PCMD}}
\newcommand{\neff}{\ensuremath{N^{(2)}_{\mathrm{eff}}}}
\newcommand{\distk}{\texttt{distinct@8}}

\newcommand{\MDfrontierpmdmean}{.232}
\newcommand{\MDsmallbestpmd}{.570}

\newcommand{\MDsmallbestpassall}{.273}

\newcommand{\MDfrontierpass}{.991}

\newcommand{\MDfrontierdeployments}{7}

\newcommand{\MDklFreeNames}{\texttt{Python} and \texttt{Pantry}}

\definecolor{bettergreen}{RGB}{224,242,226}

\definecolor{domGraph}{HTML}{E8F1FA}
\definecolor{domCountdown}{HTML}{E7FBF6}
\definecolor{domPython}{HTML}{EFFBE7}
\definecolor{domMathIR}{HTML}{E7FBEE}
\definecolor{domPantry}{HTML}{E7ECFB}

\newtheorem{lemma}{Lemma}[section]
\newtheorem{theorem}[lemma]{Theorem}

\theoremstyle{remark}

\theoremstyle{definition}

\theoremstyle{plain}

\newcommand{\providermark}[1]{\raisebox{-.18ex}{\includegraphics[height=1.75ex]{icons/#1}}\,}
\newcommand{\qwenmark}{\providermark{qwen.png}Qwen}
\newcommand{\falconmark}{Falcon}


\usepackage[disable,textsize=tiny]{todonotes}

\title{Measuring and Mitigating Solution Mode Collapse in RLVR}

\author{}

\iclrfinalcopy

\renewcommand{\headrulewidth}{0pt}

\makeatletter

\newcommand{\arxivauthor}[3]{%
  \parbox[t]{#1}{%
    \centering
    \textbf{#2}\\
    Princeton University\\
    {\footnotesize\texttt{#3}}%
  }%
}

\def\@maketitle{%
  \vbox{%
    \hsize\textwidth

    {\LARGE\sc \@title\par}

    \vskip 0.15in

    \noindent
    \hbox to \textwidth{%
      \hfil
      \arxivauthor{0.29\textwidth}
        {Liv G.~d'Aliberti}
        {od2961@princeton.edu}%
      \hspace{0.04\textwidth}%
      \arxivauthor{0.29\textwidth}
        {Marwa Abdulhai}
        {ma6249@princeton.edu}%
      \hspace{0.04\textwidth}%
      \arxivauthor{0.29\textwidth}
        {Sofiia Druchyna}
        {sd0937@princeton.edu}%
      \hfil
    }%

    \vskip 0.17in

    \noindent
    \hbox to \textwidth{%
      \hfil
      \arxivauthor{0.34\textwidth}
        {Peter Henderson}
        {peter.henderson@princeton.edu}%
      \hspace{0.08\textwidth}%
      \arxivauthor{0.34\textwidth}
        {Manoel Horta Ribeiro}
        {manoel@princeton.edu}%
      \hfil
    }%

    \vskip 0.22in
  }%
}

\makeatother

\begin{document}

\maketitle
\raggedbottom

\begin{abstract}
A language model (LM) can usually answer the same question in more than one way, but reinforcement learning with verifiable rewards (RLVR) is indifferent to which correct answer a model produces. A solution will earn the same reward whether it is the thousandth copy of a familiar answer or one the model has never produced before.
Yet, there is potential value in having the model retain multiple correct solutions as it is trained. For instance, multiple modes may give users a choice and provide problem-solving strategies that improve overall model performance.
Here, we introduce \mb{}, a benchmark of multi-solution tasks in which the verifier returns both correctness and mode discovered. 
We then use \mb{} to measure how solution diversity changes under RLVR post-training.
We find that RLVR post-training concentrates probability onto fewer correct modes even as accuracy holds or improves, and moreover, that frontier models are already highly concentrated.
We then introduce our solution, Re:Max, which stores one verified example per discovered mode in a replay buffer and trains on those stored modes uniformly. 
A solution found once is, therefore, practiced as often as one found repeatedly.
Across three model scales, two RL objectives, and harder task constructions, replay improves both how often a policy succeeds and how many different ways it can succeed.

\end{abstract}

\artifactfootnote{%
  \micon{gh_logo.png}\;%
  \artifactlink{https://github.com/liv-daliberti/modeBench}%
  {ModeBench benchmark}\qquad
  \micon{gh_logo.png}\;%
  \artifactlink{https://github.com/liv-daliberti/remax}%
  {Re:Max code}\qquad
  \micon{huggingface.png}\;%
  \artifactlink{https://huggingface.co/od2961/maxent-grpo-models}%
  {Models and checkpoints}}

\begin{figure}[!htbp]
\centering
\includegraphics[width=0.9\linewidth]{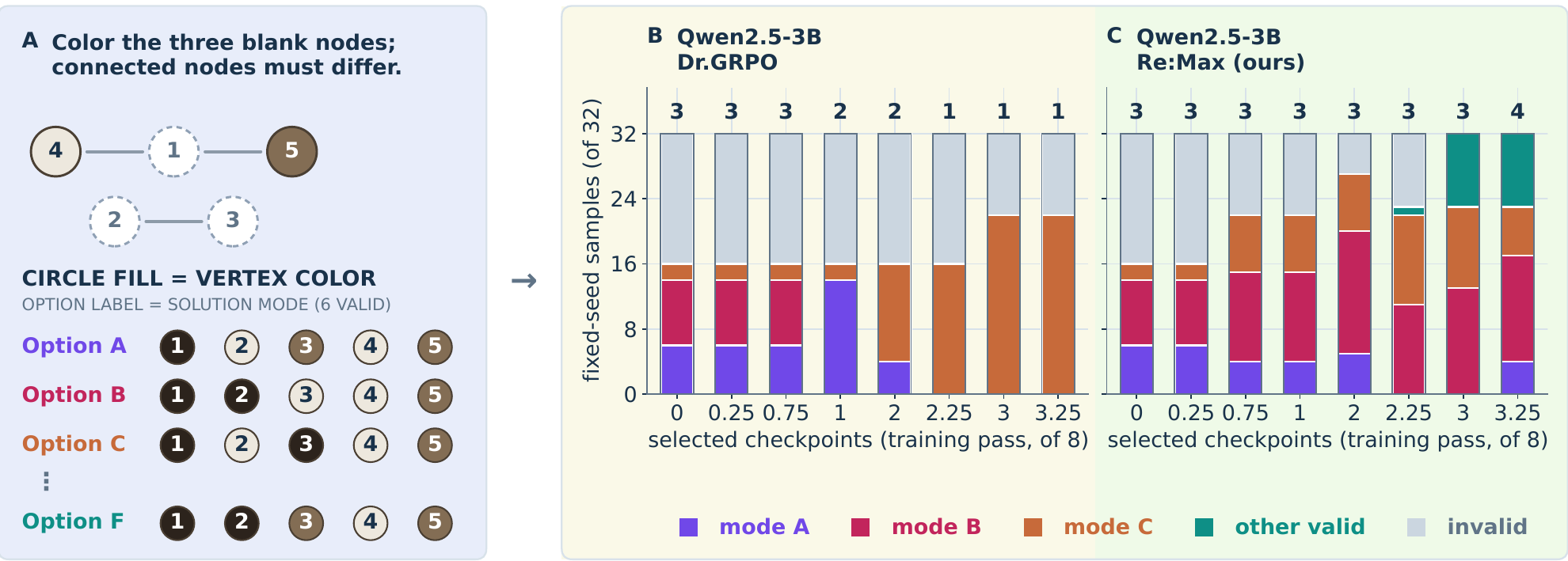}
\caption{\textbf{Motivation.}
Consider the \texttt{Graph} coloring problem shown in \textbf{A} (left), with six valid solutions.
We evaluate Qwen2.5-3B across training, tracking the number and diversity of correct solutions. At each training pass, we show 32 sampled questions, eight from each of four matched runs.
Training with Dr. GRPO (\textbf{B}) improves accuracy and leads to a single solution dominating the output. In contrast, Re:Max (\textbf{C}) increases both accuracy and diversity of correct solutions.
}
\label{fig:story}
\end{figure}

\section{Introduction}
\label{sec:introduction}

Training and evaluation of reinforcement learning with verifiable rewards (RLVR) algorithms are driven entirely by model accuracy.
However, what makes models useful at deployment depends on having different candidates to choose from, i.e. test-time search and reranking \citep{brown2024monkeys,snell2024testtime,cobbe2021training,wang2023selfconsistency}.
As RLVR improves accuracy, what happens to the diversity of correct solutions?

RLVR commonly uses automatically verifiable outcome rewards~\citep{guo2025deepseekr1}.
These rewards do not distinguish among different correct solutions.
In binary-verifier formulations, all verified-correct solutions receive the same task reward, regardless of which solution mode they belong to.
Consequently, improving correctness does not determine how the model should weight alternative correct solutions.
Fig.~\ref{fig:story} shows what this can look like on a graph-coloring problem with six valid answers: under Dr.GRPO \citep{liu2025understanding}, the policy becomes more accurate while its correct outputs concentrate onto a single mode.

Mode collapse, or narrowing of the output space, has been observed in alignment \citep{kirk2024understanding,cui2025entropy}, creative generation \citep{doshi2024creative}, brainstorming and ideation \citep{anderson2024homogenization}, public argument \citep{kim2026argument}, and LM-assisted writing \citep{abdulhai2026distort}.
Studying mode collapse, however, requires deciding when two correct outputs constitute genuinely different solutions.
Approaches based on token-level diversity \citep{li2016diversity,zhu2018texygen} are computationally inexpensive, but can mistake superficial reformulations for distinct solutions.
Semantic grouping methods \citep{farquhar2024detecting,zheng2023judging} instead use a second model to judge whether outputs are meaningfully different. 
LLM-as-judge approaches both (1) add an inference cost per generation and (2) rely on the judge to determine diversity.
To avoid this tradeoff, we study settings where a solution identity can be determined directly by an executable verifier.
This gives us an inexpensive, deterministic, verifiable measure of diversity that is insensitive to superficial wording.

Here, we investigate the prevalence and impact of mode collapse in problem-solving tasks under RLVR post-training.
Our contributions are threefold:

\begingroup
\setlength{\leftmargini}{1.1em}
\begin{compactenum}
    \item \mb{}: We contribute a benchmark for empirically studying mode collapse in tasks allowing for multiple correct solutions, ranging from code evaluation to constrained planning.
    \mb{} provides both output correctness and a unique identity for each solution mode.
    The benchmark spans five executable domains, each with five difficulty levels (Sec.~\ref{sec:modebench}).

    \item We conduct comprehensive experiments on \mb{} across 4 open-weight model families and \MDfrontierdeployments{} frontier models.
    Across open-weight models and at 3 model scales, RLVR substantially reduces solution-mode diversity across domains under both Dr.GRPO and GRPO.
    Further, in inference, all frontier models we measured reached fewer than $1.5$ effective modes out of up to 355 unique solutions (Sec.~\ref{sec:results-collapse} and App.~\ref{app:hosted-mode-diversity}).

    \item We propose \emph{mode replay}, which stores one verified exemplar per mode and rehearses all modes equally.
    Added to Dr.GRPO \citep{liu2025understanding}, it raises accuracy and pairwise correct mode diversity across 3 model scales and all 5 \mb{} domains.
    Across the five-domain aggregate, our replay methods outperform 8 RLVR comparators on both correctness and conditional solution diversity.
    Added to MaxRL \citep{tajwar2026maxrl} (which we dub as Re:Max), it adds diversity and accuracy across \mb{} difficulty levels (Sections~\ref{sec:method} and~\ref{sec:results-retention}--\ref{sec:results-levels}).
\end{compactenum}
\endgroup

Overall, our results quantify the severity of mode collapse and suggest that incorporating mode replay as part of RLVR post-training can lead to both increased diversity of solutions and improved model performance.

\section{\mb: Measuring Successful Alternatives}
\label{sec:modebench}
\label{sec:collapse}

\begin{figure}[!t]
  \centering
  \includegraphics[width=\linewidth]{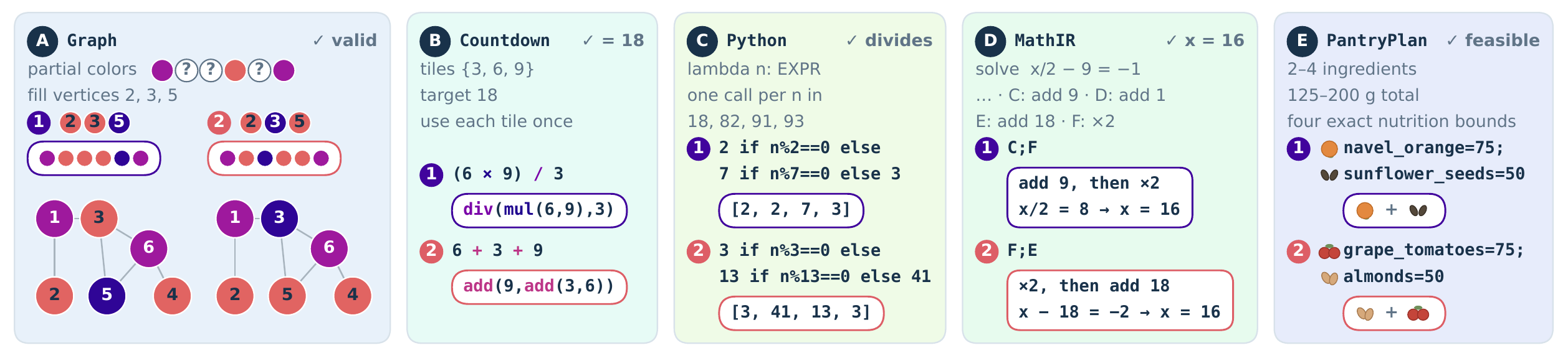}
  \caption{\textbf{\mb{} comprises questions with many acceptable solutions across domains.}
  (1) and (2) identify different solution modes across the five domains, which include Graph coloring, a normalized computation tree, executable Python code, algebraic problems and constraint optimization.
  }
  \label{fig:modebench-examples}
\end{figure}


Building a benchmark to study solution mode collapse is challenging for two reasons.
First, each problem must allow for several valid discrete solutions.
Second, given a model's response, we must be able to determine which solution mode was reached.
For example, in Fig.~\ref{fig:modebench-examples}D, we show algebra problems in our benchmark.
In this example, the same algebraic solution can be reached through two different series of steps.

Ensuring these two properties hold is domain-specific.
Fig.~\ref{fig:modebench-examples} illustrates a subset of solution modes across the considered domains; e.g., in \texttt{Countdown} and \texttt{MathIR}, modes mean the model reached the same answer by a different route, whereas in \texttt{Graph} and \texttt{Pantry}, modes mean a different valid answer altogether.
We provide the certified or enumerated solution support used for each domain in Table~\ref{tab:tasks}.
Refer to App.~\ref{app:domain-overview} and App.~\ref{app:prompts} for per-domain rules and prompts for dataset generation.

For each domain, we provide a verifier that determines both correctness and solution-mode identity.
Precisely, for a prompt $x$ with specification $s_x$, the verifier maps an output $y$ either to failure $\bot$ or to a solution mode $V(y,s_x)\in\mathcal C_x$.
Two accepted outputs are the same mode if and only if their verifier outputs agree.
The learner observes only modes for responses it generates and is never given the complete set $\mathcal C_x$ or its size.
Some domains additionally provide certified or enumerated support used for validation and analysis; these catalogues need not exhaust every solution accepted by the verifier.

\begin{figure}[b]
  \centering
  \includegraphics[width=0.95\linewidth]{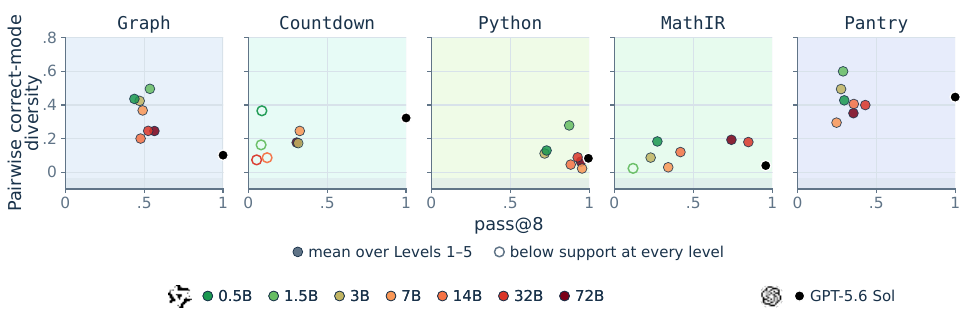}
  \caption{\textbf{Higher performance is not associated with increased solution mode diversity.} We show the relationship between a performance-related metric (\texttt{pass@}8; $x$-axis) and our proposed solution mode diversity metric (\pmd{}; $y$-axis) for GPT-5.6 Sol (in black), and Qwen2.5-Instruct models of various sizes (where parameter count is indicated by color). Accuracy and model size do not reliably predict solution-mode diversity.}
  \label{fig:base-levels-all-scales}
\end{figure}

\subsection{Success and verified-mode diversity}
\label{sec:collapse-metrics}

To measure performance on \mb{}, we must consider two distinct properties: how often the output is accurate, and, given that it is right, which of the available right answers has been given.
A benchmark reporting only accuracy cannot capture the failure we study.

The obvious way to measure modal diversity is to count distinct correct answers in $K$ samples.
We define $P$ as the chance a sample is correct and $q_c$ as how that correct probability is divided among the valid answers:
\vspace{-0.05em}
\begin{equation}
 \underbrace{\Pr[D_K>0]}_{\texttt{pass@}K}=1-(1-P)^K,
 \qquad
 \underbrace{\E[D_K]}_{\texttt{distinct@}K}
 =\sum_c[1-(1-Pq_c)^K],
 \label{eq:main-occupancy}
\end{equation}
\vspace{-1em}

where $D_K$ counts the distinct correct answers.
However, both metrics increase with $P$, and thus a policy that is right more often may have higher $\texttt{pass@K}$ and $\texttt{distinct@}K$ while producing less diverse solutions.

We therefore propose an alternative metric, denoted as pairwise correct-mode diversity:
\begin{equation}
 \pmd = 1-\sum_c q_c^2
 = \Pr[Z_1\ne Z_2\mid \text{both responses verified}],
 \label{eq:main-pmd}
\end{equation}
\vspace{-1.2em}

measuring the probability of different solution modes conditional on two correct responses.
A \pmd{} of zero thus implies every correct answer belongs to the same solution mode, and an even spread over $m$ yields a \pmd{} of $1-1/m$, e.g., if answers spread evenly across two possible solution modes, then the \pmd{} is $1/2$.
We note that when infinite solution modes are possible, the maximum \pmd{} converges to 1, and that when the model outputs no correct solutions, we cannot calculate the metric.
We only calculate \pmd{} when a model produces a minimum number of correct solutions; App.~\ref{app:metric-details} gives that threshold.

Fig.~\ref{fig:base-levels-all-scales} shows how \pmd{} captures diversity of responses independently of accuracy.
GPT-5.6 Sol has near-perfect \texttt{pass@8} for all five tasks considered, yet much weaker models in the Qwen2.5-Instruct family have higher \pmd{}.
In this work, we define a \textit{solution mode collapse} as a decline in \pmd{} during RLVR.

\begin{wrapfigure}{r}{0.56\linewidth}
  \centering
  \includegraphics[width=\linewidth]{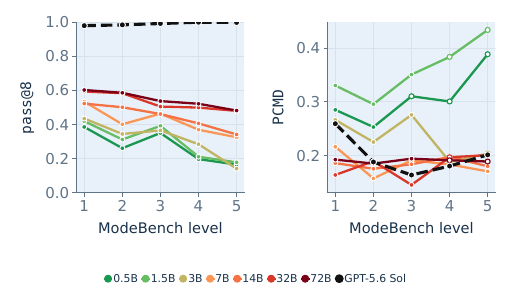}
  \caption{\textbf{Correctness and solution-mode diversity vary across benchmark levels.} Across increasingly difficult \mb{} levels, \texttt{pass@8} generally declines, while \pmd{} follows different patterns across models.}
  \label{fig:levels}
\end{wrapfigure}

\subsection{Impact of task difficulty on mode collapse}
\label{sec:levels-design}

We construct \mb{} with multiple difficulty levels so that we can study how solution-mode diversity changes as problems become harder, while keeping the task-defined solution mode fixed.
For each domain, we include five difficulty levels, e.g. for our \texttt{Graph} domain we allow more hidden vertices, and for \texttt{Countdown} we give the model more operands.
Each difficulty level holds constant (1) our definition of mode, (2) our approximate distribution of modes, (3) our verifier, and (4) the sample size of 384 train and 128 test problems.
To ensure (2) holds, and moreover, that the problems we sample are in fact more challenging, we fit each level with consideration to \texttt{pass@1} and \texttt{pass@8} of a reference model, e.g. Level 4 was tuned such that a Qwen2.5-Instruct-7B model has the same \texttt{pass@1} and \texttt{pass@8} scores as a Qwen2.5-Instruct-0.5B on Level 1.


Across benchmark levels, solution-mode diversity varies with task difficulty and model scale, including in settings where correctness remains near saturation (Fig.~\ref{fig:levels}).
This is most noticeable when frontier models are evaluated.
Even though correctness is saturated, we are able to use \pmd{} to isolate solution mode collapse (App.~\ref{app:hosted-mode-diversity}), and this pattern appears across \mb{}.
For example, frontier models answer $\MDfrontierpass$ of questions from Level 1--3 correctly under normalized grading, but their mean \pmd{} is only $\MDfrontierpmdmean$, below the $\MDsmallbestpmd$ of the most diverse sub-2B model, which answers $\MDsmallbestpassall$ of its prompts.

\begin{figure}[!t]
  \centering
  \includegraphics[width=\linewidth]{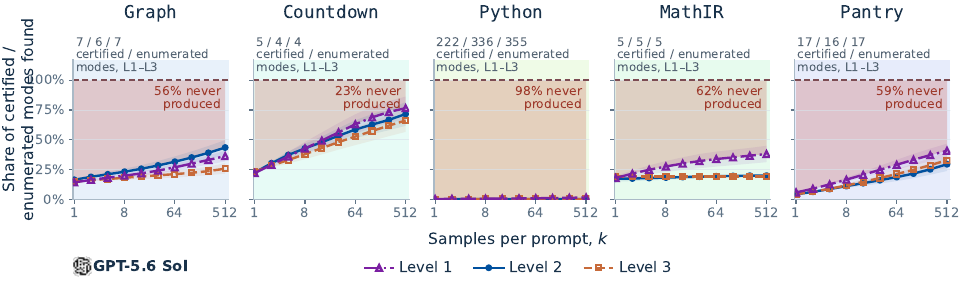}
  \vspace{-9pt}
  \caption{\textbf{Asking the model more times may not recover all certified or enumerated modes.}
  We test directly whether a larger sampling budget recovers the known modes.
  GPT-5.6 Sol, with medium reasoning on, answers 32 problems per domain and level with 512 responses each.
  Curves are distinct verified modes in random $k$-subsets as a share of the enumerated or certified modes available for those problems, and the shaded band is what no level ever produced.
  Shading around lines is the pointwise 95\% CIs using a bootstrap over problems.}
  \label{fig:gpt56-sampling-budget}
\end{figure}

\subsection{Concentration in frontier models}
\label{sec:hosted-concentration}

A common fix for solution mode collapse is to increase the sampling budget, as repeated sampling can substantially increase solution coverage at test time \citep{brown2024monkeys,yue2025rlvrlimit}.
But, across \mb{}, larger sampling budgets remain far below the enumerated or certified solution support in several domains.
For example, all 1,024 GPT-5.6 Sol draws on the 128 Level-3 \texttt{Graph} problems verify, yet eight draws yield just $1.16$ distinct colorings, and $95.1\%$ of correct pairs repeat one (App.~\ref{app:hosted-mode-diversity}).
Fig.~\ref{fig:gpt56-sampling-budget} follows \mb{} at three levels out to 512 draws per problem, and the curves flatten far below the modal set: an LM that has settled on two colorings does not find the rest even after being asked five hundred times.

\texttt{Python} is the extreme case.
Its problems admit several hundred distinct return vectors on average, and five hundred draws surface between one and four solutions, under two percent of all possible solutions, and under one percent at Levels 2 and 3.
Each doubling of the sampling budget buys a fraction of one mode, against gaps of several modes everywhere and several hundred in \texttt{Python}.
This problem becomes worse as we increase problem difficulty.
All \MDfrontierdeployments{} frontier models tested narrow in \pmd{} from Level 1 to Level 3 while staying near-perfect on correctness (Table~\ref{tab:hosted-cohort-mode-diversity}, App.~\ref{app:hosted-mode-diversity}).

We similarly check our local models over larger samplings, and we find quadrupling to $K=32$ draws per problem moves control \pmd{} by at most $.010$, and \texttt{MathIR} at $K=64$ draws is as concentrated as at $K=8$ (Apps.~\ref{app:decoding-objection} and~\ref{app:sampling-budget-ablation}), while disabling reasoning lowers correctness and verified-mode counts in every deployment tested, though its \pmd{} effect differs by deployment (Table~\ref{tab:hosted-reasoning-control}, App.~\ref{app:hosted-reasoning-control}).

We do, however, note that cheap inference interventions do exist.
For example, we find that because different models collapse onto different answers, mixing responses from two independent models covers more verified solution modes than either alone (App.~\ref{app:cross-model-overlap}).
Additionally, raising inference temperature can expose more sampled solution modes, sometimes at a cost in accuracy (Table~\ref{tab:gpt56-temperature-summary}, App.~\ref{app:hosted-temperature}).

\section{Re:Max: Learn from New Samples and Remember Successes}
\label{sec:method}

\begin{wrapfigure}{r}{0.6\linewidth}
  \centering
  \vspace{-0.1cm}
  \includegraphics[width=0.9\linewidth]{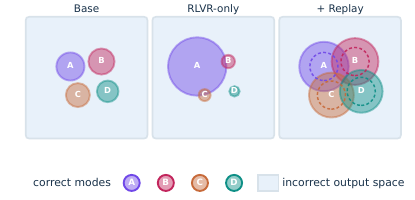}
  \caption{\textbf{RLVR early ``winner'' (A) becomes the predominant way to solve.} The base model supports several modes. RLVR-only training concentrates on one mode. Our replay approach keeps every buffered mode \& leaves room for more.}
  \label{fig:replay-bank-balance}
\end{wrapfigure}

\subsection{Theoretical intuition for mode collapse}
\label{sec:collapse-mechanism}

Solution mode collapse follows from two properties of the objective:
1) the binary verifier assigns the same reward to every correct answer, so the objective provides no incentive to preserve one correct mode over another; and 2) fresh on-policy updates directly reinforce modes that appear in the current rollout samples, which are drawn according to the policy's current preferences.
Together, these properties create a positive feedback loop, where a mode with even a small early advantage is sampled more often, receives more reinforcement, and therefore becomes even more likely to be sampled in subsequent updates.
Figure~\ref{fig:replay-bank-balance} illustrates this process.

Recent work formalizes this positive-feedback effect, showing that on-policy RL can amplify small probability differences among equally rewarded outcomes into ``winner-takes-all'' dynamics~\citep{sinha2026expected,lochab2026ucpo}.
We extend existing analysis to the verified solution-mode setting; App.~\ref{app:theory} gives the formal results and assumptions.
First, under an idealized solution-mode model over verifier-defined modes, the same winner-takes-all dynamic applies.
Second, binary rewards alone do not prefer mode concentration; the concentration arises from the optimization dynamics.
Third, we characterize why collapse becomes difficult to reverse under on-policy training.
We show that as a solution mode becomes rare, both its probability of appearing in a rollout and its expected gradient contribution approach zero.

\subsection{Replaying distinct verified successes}
\label{sec:replay}

These results motivate a simple intervention.
Once a correct solution mode has been discovered, continued training should not depend on the current policy re-discovering it.
We therefore use experience replay to retain verified successes and continue training on them \citep{lin1992selfimproving,DBLP:journals/corr/MnihKSGAWR13,rolnick2019replay}.
In our idealized fixed-buffer analysis, replay prevents a stored mode's probability from vanishing: once a mode is stored in the buffer, its probability remains bounded away from zero throughout subsequent training under the assumptions of Theorem~\ref{thm:replay-retention}.


For each training prompt $x$, we maintain an initially empty, bounded buffer $\mathcal B_x$.
The first time a verified mode appears, we store one exemplar $e_c$ (subsequent rediscoveries do not modify the buffer).
A deterministic schedule revisits nonempty buffers throughout training.
For an exemplar of token length $|e_c|$, the replay loss is

\vspace{-2mm}
\begin{equation}
 L_{\mathrm{replay}}(x)=-\frac{1}{|\mathcal B_x|}
 \sum_{c\in\mathcal B_x}\frac{\log\pi_\theta(e_c\mid x)}{|e_c|}.
 \label{eq:lm-verified-replay}
\end{equation}
\vspace{-2mm}

Uniform weighting breaks the link between discovery frequency, how often a mode appears in fresh rollouts, and rehearsal frequency, how often a stored mode is replayed.
Once a mode enters the buffer, it is rehearsed as often as every other retained mode, so frequently sampled modes receive no additional replay.

Adding this loss to the fresh update from standard MaxRL \citep{tajwar2026maxrl} gives our primary method, Re:Max.
This principle is modular and can be applied to any RLVR algorithm.
For example, applying it to Dr.GRPO \citep{liu2025understanding} yields Re:Dr.
Solution modes discovered during training do not enter the on-policy reward group and contribute only through the separate replay loss.
This separation allows one to vary the RLVR objective and the replay mechanism independently.
Finally, our replay buffers contain only responses to training prompts; held-out evaluation therefore measures changes in the policy itself rather than retrieval or recitation from the buffer (App.~\ref{app:algorithm}).

Fig.~\ref{fig:verified-support-story} illustrates how replay changes learning dynamics.
In the standard RLVR setting, fresh updates are driven by modes present in the current sampling, which can cause previously discovered modes to disappear.
Re:Max adds an additional incentive for the policy to continuously generate previously found solution modes that yield correct solutions.
We note that this is fundamentally different from approaches that encourage diversity by modifying the on-policy learning signal to favor rare or diverse outcomes \citep{song2025outcome,li2026setpo}, as we do not use mode identity to modify the task reward or advantage.

\begin{figure}[b]
  \centering
  \includegraphics[width=\linewidth]{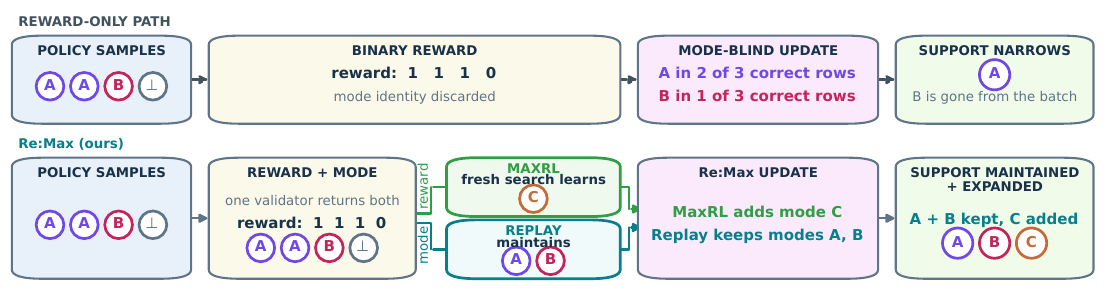}
  \caption{\textbf{Replay preserves successes that on-policy learning may stop seeing.}
  MaxRL learns from the policy's current samples, while replay stores one verified exemplar per discovered mode and revisits those exemplars with equal weight.
  Together, these two learning paths allow Re:Max to preserve previously discovered solution modes while continuing to explore for new ones.}
  \label{fig:verified-support-story}
\end{figure}

\section{Experimental Design}
\label{sec:experiments}
\label{sec:shared-protocol}

In our main experiment, we examine whether replay improves correctness and solution-mode diversity across different RLVR objectives.
We compare the impact of our replay component against both Dr.GRPO \citep{liu2025understanding} and MaxRL \citep{tajwar2026maxrl}.
First, on all five \mb{} Level 1 train/test sets we carry out five seeds of post-training using Qwen2.5-0.5B, Falcon3-1B, and Qwen2.5-3B.
We train for eight passes with $G=16$ new rollouts.
We note that replay does not increase this sampling budget.
Instead, it performs an additional training update using at most $M=16$ previously collected exemplars from the replay buffer for each visited prompt (Table~\ref{tab:run-contract}).
Second, we conduct the same experiment on \mb{} Levels 2 and 3 using only Qwen2.5-0.5B.
Third, we compare our method with other approaches, including frequency-weighted replay, UCPO \citep{lochab2026ucpo}, RLEP \citep{zhang2025rlep}, GAPO \citep{anschel2025groupaware}, SetPO \citep{li2026setpo}, and fixed Semantic-MaxEnt under a specific setting (Qwen2.5-0.5B, five seeds, Level 1).


We additionally evaluate \MDfrontierdeployments{} hosted models: GPT-5.6 Sol \citep{openai2026gpt56}, GPT-5.4 \citep{openai2026gpt54}, Grok 4.3 \citep{xai2026grok43}, Kimi K3 \citep{kimiteam2026kimik3}, DeepSeek V4 Pro \citep{deepseek2026v4pro}, Claude Opus 5 \citep{anthropic2026opus5}, and Claude Opus 4.8 \citep{anthropic2026opus48} on \mb{} Levels 1--3, collecting eight responses per prompt and grading them through a fixed formatting normalizer so that presentation differences are not counted as verified alternatives (App.~\ref{app:hosted-comparison}).
We carry out the same inference-only evaluation on Levels 1--5 across the available scales and benchmark levels of SmolLM2 \citep{allal2025smollm2}, Qwen2.5 \citep{yang2024qwen25}, Falcon3 \citep{falconllm2024falcon3}, and OLMo-2 \citep{olmo20242olmo2furious}.
These experiments are inference-only; no replay or training is applied.

\section{Results}
\label{sec:results}

Our results are three-fold.
First, we document solution-mode collapse under RLVR on \mb{} (Section~\ref{sec:results-collapse}).
Second, we find that replay reverses this collapse without sacrificing correctness (Section~\ref{sec:results-retention}), outperforming the 8 evaluated RLVR comparators on the five-domain aggregate.
Third, we show that replay gains persist across difficulty levels (Section~\ref{sec:results-levels}), KL anchoring, and domain changes.

\subsection{RLVR concentrates solution modes}
\label{sec:results-collapse}

\begin{wrapfigure}{r}{0.5\textwidth}
  \centering
  \includegraphics[width=0.5\textwidth]{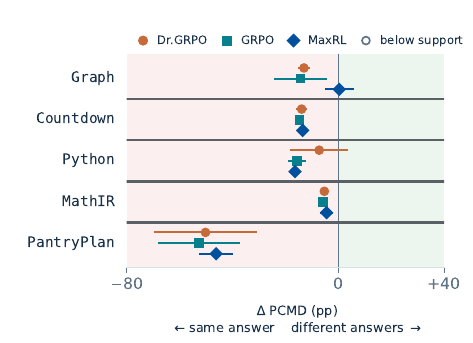}
  \caption{\textbf{Training reduces solution mode diversity.} We compare \pmd{} before and after RLVR on Qwen2.5-3B, over \mb{} Level 1 test set, under three objectives, five domains, five seeds. The figure shows the change in \pmd{} from before to after. Marks are means over training seeds with 95\% intervals. In 14 of 15 comparisons, we find fewer solution modes after training. See Fig.~\ref{fig:concentration-all-scales} for other model scales.}
  \label{fig:concentration-story}
\end{wrapfigure}

\textbf{RLVR replaces diverse correct attempts with the same repeated answers.}
We compare each policy's solution-mode distribution before and after training using \pmd{}.
Across Dr.GRPO, GRPO, and MaxRL results, we find fewer solution modes after RLVR post-training in 14 out of 15 comparisons for Qwen2.5-3B (Fig.~\ref{fig:concentration-story}), with the same pattern appearing across model scales and difficulty levels (Fig.~\ref{fig:concentration-all-scales}).
The starkest case is \texttt{Python}, where Qwen2.5-3B goes from solving 73 of 128 held-out prompts in multiple ways to as few as 22 after Dr.GRPO post-training.

\textbf{As outputs concentrate, rollouts under RLVR lose training signal.}
Dr.GRPO and MaxRL learn from a rollout group only when it contains a mix of correct and incorrect responses.
Thus, we audit every rollout group during training to track how often it provides a non-zero learning signal.
We find that Re:Max systematically increases the percentage of rollout groups with a learning signal (App.~\ref{app:degenerate-groups}).

\subsection{Replay improves both accuracy and diversity}
\label{sec:results-retention}

\begin{table}[t]
  \centering
  \caption{\textbf{Compared objectives under Qwen2.5-0.5B, \mb{} Level 1.} All differences are relative to Dr.GRPO; avg weights the five domains equally within seed (5 seeds). Bold marks a column's best. An em-dash means that there are few more-than-one solution pairs to evaluate, and a superscript gives the paired seed count where it falls below five. Red-to-green shading scales to each column's largest magnitude.}
  \label{tab:direct-comparator-matrix}
  \scriptsize
  \setlength{\tabcolsep}{1.4pt}
  \renewcommand{\arraystretch}{0.95}
  \begin{tabular}{@{}l*{12}{r}@{}}
    \toprule
    & \multicolumn{6}{c}{\textbf{$\Delta$ \texttt{pass@8}}}
      & \multicolumn{6}{c}{\textbf{$\Delta$ \pmd{}}} \\
    \cmidrule(lr){2-7}\cmidrule(l){8-13}
    & \texttt{Graph}
    & \shortstack{\texttt{Count}\\\texttt{down}}
    & \texttt{Python}
    & \texttt{MathIR}
    & \texttt{Pantry}
    & Avg
    & \texttt{Graph}
    & \shortstack{\texttt{Count}\\\texttt{down}}
    & \texttt{Python}
    & \texttt{MathIR}
    & \texttt{Pantry}
    & Avg \\
    \midrule
    Re:Dr \textit{(ours)} & \cellcolor[HTML]{A9D6A0}$\mathbf{+.646}$ & \cellcolor[HTML]{D1E3A3}$\mathbf{+.191}$ & \cellcolor[HTML]{C1DEA2}$+.356$ & \cellcolor[HTML]{ABD7A0}$\mathbf{+.285}$ & \cellcolor[HTML]{CDE2A3}$+.207$ & \cellcolor[HTML]{AED8A0}$+.337$ & \cellcolor[HTML]{A9D6A0}$\mathbf{+.559}$ & \cellcolor[HTML]{ADD7A0}$+.480$ & \cellcolor[HTML]{DEE7A4}$+.082$ & \cellcolor[HTML]{DDE7A4}$\mathbf{+.017}$ & \cellcolor[HTML]{D0E3A3}$+.334$ & \cellcolor[HTML]{AFD8A0}$+.294$ \\
    Re:Max \textit{(ours)} & \cellcolor[HTML]{ACD7A0}$+.621$ & \cellcolor[HTML]{D2E3A3}$+.185$ & \cellcolor[HTML]{A9D6A0}$\mathbf{+.509}$ & \cellcolor[HTML]{ADD7A0}$+.277$ & \cellcolor[HTML]{CFE2A3}$+.199$ & \cellcolor[HTML]{A9D6A0}$\mathbf{+.358}$ & \cellcolor[HTML]{AED8A0}$+.525$ & \cellcolor[HTML]{A9D6A0}$\mathbf{+.504}$ & \cellcolor[HTML]{A9D6A0}$\mathbf{+.249}$ & \cellcolor[HTML]{F0EDA5}$+.005$ & \cellcolor[HTML]{D2E3A3}$+.317$ & \cellcolor[HTML]{A9D6A0}$\mathbf{+.320}$ \\
    \midrule
    GAPO & \cellcolor[HTML]{E3E9A4}$+.175$ & \cellcolor[HTML]{F6EFA6}$+.011$ & \cellcolor[HTML]{F7F0A6}$+.009$ & \cellcolor[HTML]{F7E2A2}$-.054$ & \cellcolor[HTML]{BEDDA2}$\mathbf{+.279}$ & \cellcolor[HTML]{E6EAA5}$+.084$ & \cellcolor[HTML]{E5EAA5}$+.131$ & \cellcolor[HTML]{C3DEA2}$+.340$ & \cellcolor[HTML]{F8F0A6}$.000$ & \cellcolor[HTML]{F8F0A6}$.000$ & \cellcolor[HTML]{A9D6A0}$\mathbf{+.653}$ & \cellcolor[HTML]{C1DEA2}$+.225$ \\
    SetPO & \cellcolor[HTML]{F0EDA5}$+.064$ & \cellcolor[HTML]{F7E6A3}$-.050$ & \cellcolor[HTML]{ABD7A0}$+.497$ & \cellcolor[HTML]{EAEBA5}$+.051$ & \cellcolor[HTML]{D8E5A4}$+.154$ & \cellcolor[HTML]{D8E6A4}$+.143$ & \cellcolor[HTML]{F3EEA6}$+.039$ & \cellcolor[HTML]{F7F0A6}$+.007$ & \cellcolor[HTML]{F4EFA6}$+.013$ & \cellcolor[HTML]{F3EEA6}$+.003$ & \cellcolor[HTML]{D8E5A4}$+.266$ & \cellcolor[HTML]{E8EBA5}$+.066$ \\
    MaxRL (no replay) & \cellcolor[HTML]{DEE7A4}$+.213$ & \cellcolor[HTML]{E3E9A4}$+.101$ & \cellcolor[HTML]{F8F0A6}$.000$ & \cellcolor[HTML]{F6DFA1}$-.066$ & \cellcolor[HTML]{F6EFA6}$+.009$ & \cellcolor[HTML]{EDECA5}$+.051$ & \cellcolor[HTML]{ECECA5}$+.086$ & \cellcolor[HTML]{F6EFA6}$+.015$ & \cellcolor[HTML]{F8F0A6}$.000$ & \cellcolor[HTML]{F8EFA6}$.000$ & \cellcolor[HTML]{F8F0A6}$.000$ & \cellcolor[HTML]{F3EEA6}$+.020$ \\
    RLEP-Dr & \cellcolor[HTML]{ECECA5}$+.098$ & \cellcolor[HTML]{E3E9A4}$+.103$ & \cellcolor[HTML]{F8F0A6}$.000$ & \cellcolor[HTML]{D6E5A3}$+.126$ & \cellcolor[HTML]{F8F0A6}$.000$ & \cellcolor[HTML]{EAEBA5}$+.066$ & \cellcolor[HTML]{F3EEA6}$+.035$ & \cellcolor[HTML]{F6EFA6}$+.012$ & \cellcolor[HTML]{F8F0A6}$.000$ & \cellcolor[HTML]{F8EFA6}$-.001$ & \cellcolor[HTML]{F8F0A6}$.000$ & \cellcolor[HTML]{F6EFA6}$+.009$ \\
    UCPO & \cellcolor[HTML]{F8F0A6}$-.003$ & \cellcolor[HTML]{DDE7A4}$+.133$ & \cellcolor[HTML]{F8F0A6}$.000$ & \cellcolor[HTML]{F0EDA5}$+.031$ & \cellcolor[HTML]{EDECA5}$+.053$ & \cellcolor[HTML]{EFEDA5}$+.043$ & \cellcolor[HTML]{F8F0A6}$-.003^{\,2}$ & \cellcolor[HTML]{F5EFA6}$+.019$ & \cellcolor[HTML]{F8F0A6}$.000^{\,2}$ & \cellcolor[HTML]{F8EFA6}$-.001$ & \cellcolor[HTML]{F8F0A6}$.000^{\,2}$ & \cellcolor[HTML]{F6EFA6}$+.009$ \\
    Semantic-MaxEnt & \cellcolor[HTML]{F8F0A6}$+.002$ & \cellcolor[HTML]{DBE6A4}$+.141$ & \cellcolor[HTML]{F8F0A6}$.000$ & \cellcolor[HTML]{F5D19C}$-.121$ & \cellcolor[HTML]{CEE2A3}$+.201$ & \cellcolor[HTML]{EEEDA5}$+.045$ & \cellcolor[HTML]{F7F0A6}$+.004$ & \cellcolor[HTML]{F5EFA6}$+.020$ & \cellcolor[HTML]{F8F0A6}$.000$ & \cellcolor[HTML]{F7F0A6}$+.001$ & \textemdash{} & \cellcolor[HTML]{F6EFA6}$+.006$ \\
    GRPO & \cellcolor[HTML]{F3EEA6}$+.037$ & \cellcolor[HTML]{D7E5A4}$+.159$ & \cellcolor[HTML]{F8F0A6}$.000$ & \cellcolor[HTML]{F7E3A2}$-.052$ & \cellcolor[HTML]{F3EEA6}$+.023$ & \cellcolor[HTML]{F1EEA5}$+.034$ & \cellcolor[HTML]{F7F0A6}$+.008$ & \cellcolor[HTML]{F8F0A6}$+.001$ & \cellcolor[HTML]{F8F0A6}$.000$ & \cellcolor[HTML]{F6EFA6}$+.001$ & \cellcolor[HTML]{F8F0A6}$.000$ & \cellcolor[HTML]{F8F0A6}$+.002$ \\
    \midrule
    Untrained checkpoint & \cellcolor[HTML]{E8EBA5}$+.132$ & \cellcolor[HTML]{F0A58F}$-.383$ & \cellcolor[HTML]{F5D79E}$-.172$ & \cellcolor[HTML]{F0A58F}$-.292$ & \cellcolor[HTML]{A9D6A0}$+.380$ & \cellcolor[HTML]{F7E2A2}$-.067$ & \cellcolor[HTML]{E9EBA5}$+.108$ & \textemdash{} & \textemdash{} & \cellcolor[HTML]{CDE2A3}$+.027$ & \cellcolor[HTML]{B0D8A1}$+.591$ & \cellcolor[HTML]{BCDCA1}$+.242$ \\
    \bottomrule

  \end{tabular}
\end{table}

\textbf{Replay improves accuracy and diversity across RLVR algorithms and model scales.}
We add our proposed verified-mode replay mechanism to Dr.GRPO and MaxRL and compare each against its matched no-replay control across five \mb{} domains and three model scales.
Across these comparisons, we find that Re:Max and Re:Dr (ours) shift the outcome toward both higher \texttt{pass@8} and greater solution-mode diversity (Fig.~\ref{fig:maxrl-factorial}; Table~\ref{tab:cross-scale-terminal-effects}).

\textbf{Mode-diversity gains come from both retaining successes and balancing replay.}
To isolate the role of mode balancing, we compare uniform replay with a frequency-weighted variant that rehearses modes in proportion to their discovery rate.
Holding replay buffer construction, replay advantage scheduling, and loss update unchanged, we find that uniform weighting preserves greater verified diversity, most clearly in \texttt{Graph}, \texttt{Countdown}, and \texttt{Pantry}.
For example, the post-trained LM recovers roughly 0.6 additional modes per prompt in \texttt{Graph}, while also increasing \texttt{pass@8} by 6.6 percentage points.
Across domains, these gains come without an aggregate difference in correctness (App.~\ref{app:frequency-replay-progress}, App.\ Fig.~\ref{fig:replay-key-weighting}).

\begin{wrapfigure}{r}{0.6\textwidth}
  \centering
  \vspace{-0.9\baselineskip}
  \includegraphics[width=0.55\textwidth]{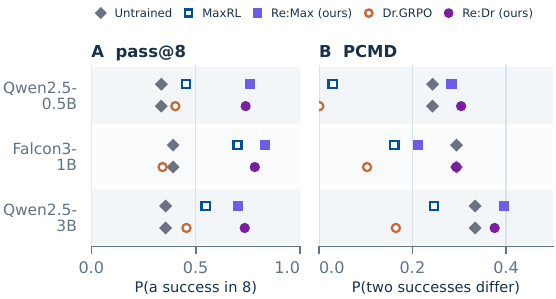}
  \caption{\textbf{Replay leads to higher accuracy and solution mode diversity.} Diamonds are initial LMs; open/filled markers are training without/with replay. Averages all five domains within seed. (App.\ Fig.~\ref{fig:maxrl-factorial-scale-extensions} provides per-task results).}
  \label{fig:maxrl-factorial}
\end{wrapfigure}


\textbf{Alternative diversity mechanisms produce isolated rather than consistent gains.}
We compare Re:Max and Re:Dr (ours) against UCPO \citep{lochab2026ucpo}, RLEP \citep{zhang2025rlep}, GAPO \citep{anschel2025groupaware}, SetPO \citep{li2026setpo}, and fixed Semantic-MaxEnt (Table~\ref{tab:direct-comparator-matrix}, App.~\ref{app:ucpo-progress}).
Some methods produce substantial gains in individual domains, e.g., GAPO reaches $2.9\times$ its control's effective number of verified modes in \texttt{Pantry}.
Yet gains are domain-specific, while Re:Max improves both metrics across domains, raising \texttt{pass@8} by $90\%$ relative to Dr.GRPO with $1.9\times$ its effective verified modes (Table~\ref{tab:direct-comparator-matrix}; App.~\ref{app:supporting}).

\subsection{Replay remains effective beyond the base training setting}
\label{sec:results-levels}

\textbf{Replay's gains persist on harder task constructions.}
We repeat the replay-versus-control comparison within each of the first three \mb{} levels, holding the level and RLVR objective fixed.
Re:Dr and Re:Max improve average \texttt{pass@8} and \pmd{} under Dr.GRPO and MaxRL, respectively, at every level (Fig.~\ref{fig:level2-admission}).
At Level 2, for instance, \texttt{Graph} rises from $1.04$ to $1.55$ effective verified modes under Dr.GRPO, and from $1.13$ to $1.48$ under MaxRL.
Gains vary across domains, but remain positive in aggregate even at Level 3 (App.~\ref{app:level2-interim}).


\textbf{Replay leaves more usable alternative correct solutions when the task changes.}
We ask whether Re:Max and Re:Dr have a downstream impact on inference.
Ideally, replay should make a model more able to give a correct answer when the original most-common answer is no longer available.
To test this, we take a portfolio of eight generated answers, remove one option from the task, e.g. an ingredient in \texttt{Pantry} or a color in \texttt{Graph}, and re-check under the modified task.
Across all 5 domains, Re:Max and Re:Dr are more likely to give a valid alternate answer.
When no answer is valid, Re:Max and Re:Dr find a new valid answer in fewer additional attempts than RLVR or tuned-temperature sampling (App.~\ref{app:portfolio-withdrawals}).

\begin{wrapfigure}{l}{0.55\textwidth}
  \centering
  \includegraphics[width=0.45\textwidth]{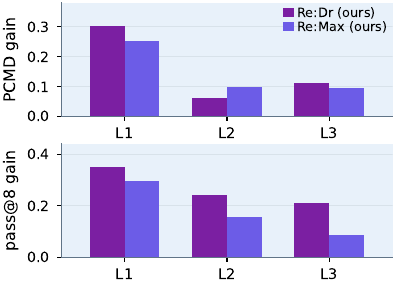}
  \caption{\textbf{Replay gains persist across difficulty levels.}
  Bars show the gain from adding replay to Dr.GRPO (Re:Dr) and MaxRL (Re:Max) at each \mb{} level, averaged across domains for Qwen2.5-0.5B. Replay improves both \pmd{} (top) and \texttt{pass@8} (bottom) at all three levels.}
  \label{fig:level2-admission}
\end{wrapfigure} 

\textbf{Reference KL and replay preserve different sources of diversity.}
We explore why Re:Max and Re:Dr gain solution diversity relative to the base model.
A potential explanation is that the model retains the diversity of its base.
We test this hypothesis by anchoring the policy to its base using KL divergence.
We sweep the $\beta$-coefficient of a KL penalty to the initial policy and compare with Re:Max and Re:Dr.
The anchor is limited by the diversity already present in its reference, whereas Re:Max and Re:Dr are limited by the modes their buffers discover.
Averaged over the three domains with a reportable initial \pmd{}, matching diversity with a KL anchor costs model correctness, although in \MDklFreeNames{} the reference is already broad enough for the anchor to win (Fig.~\ref{fig:reference-kl-plane}).

\begin{wrapfigure}{r}{0.45\textwidth}
  \centering
  \vspace{-0.7\baselineskip}
  \includegraphics[width=\linewidth]{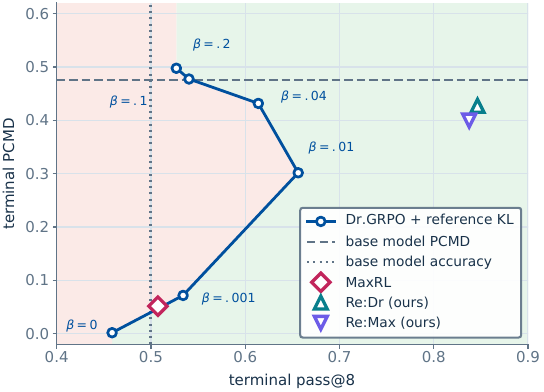}
  \caption{\textbf{Replay and KL anchoring preserve different diversity.}
  Means over the 3 domains with reportable initial \pmd{}, for Qwen2.5-0.5B. The KL anchor forms a curve because we sweep $\beta$-KL coefficients. Increasing $\beta$ preserves more diversity, but at a lower \texttt{pass@8}.
  Green indicates higher \texttt{pass@8} than the anchor at the same diversity (App.~\ref{app:kl-measured}).}
  \label{fig:reference-kl-plane}
\end{wrapfigure}

\section{Related Work}
\label{sec:related}

\textbf{Measuring solution diversity under alignment}
RLHF and reasoning RL can reduce output diversity, entropy, and high-budget coverage \citep{kirk2024understanding,dang2025diversity,cui2025entropy,yue2025rlvrlimit}.
Semantic clustering infers it from language \citep{farquhar2024detecting}, which is general but depends on a judge.
\mb{} instead defines solution identity through the same executable check that is used to assign reward.
This gives us a narrower but cheaper measure of solution diversity, shared on training and evaluation.
Preserving alternatives matters because sampling-based selection relies on generating multiple viable proposals \citep{brown2024monkeys,snell2024testtime,cobbe2021training,lightman2023lets}.

\textbf{Expanding diversity.}
Outcome bonuses, set prediction, distribution matching, and uniform-correct objectives all reshape the learning signal~\citep{song2025outcome,chen2025passktraining,li2025divergence,li2026setpo,li2026distribution,lochab2026ucpo}, while MaxRL combines objectives across sampling budgets \citep{tajwar2026maxrl}.
Many existing approaches still operate on what the current policy samples, so rare modes may receive little direct signal.
Replay can break this dependence on rediscovery.
Standard replay methods~\citep{lin1992selfimproving,rolnick2019replay,zhang2025rlep} revisit past successes in proportion to how often they occur.
Instead, we homogeneously balance replay across solution modes, an idea related to quality-diversity methods \citep{mouret2015illuminating,ecoffet2021return}.


\section{Conclusion}
\label{sec:conclusion}
\label{sec:discussion}
\label{sec:limitations}

We contribute \mb{}, a benchmark for studying solution mode diversity in RLVR, and Re:Max, a replay-based method that preserves discovered solution modes during RLVR.
Across objectives, model scales, and task difficulties, we find that replay can retain solution diversity while improving correctness.
Our results suggest that solution diversity may itself be a target for optimization.
Future work can move beyond preserving discovered modes towards training models to produce desired distributions over ways to solve a problem.

\textbf{Limitations.}
We acknowledge that replay retains only what sampling finds.
We also recognize that mode count says nothing about the ``worth'' of additional solutions.
Our experiments are designed such that the solution-equivalence relation is executable by construction, so future work is needed to evaluate the benefit and feasibility of replay in more open-ended solution spaces.


\label{sec:main-end}

\subsection*{Acknowledgments}
\label{sec:acknowledgements}

The authors are pleased to acknowledge that the work reported on in this paper
was substantially performed using the Princeton Research Computing resources at
Princeton University. Princeton Research Computing is a consortium of groups
including the Princeton Institute for Computational Science and Engineering
(PICSciE) and Research Computing at Princeton University.

This work began in COS 514, \textit{Fundamentals of Deep Learning}, taught by
Sanjeev Arora in Fall 2025, and was further developed in COS 598D,
\textit{Empirical Research Methods for CS}, taught by Manoel Horta Ribeiro in
Spring 2026. Liv G.~d'Aliberti's work was supported in part by a second-year
teaching assistantship from Princeton University's Department of Computer
Science.

We thank Cannoli, the Humans \& Machines Lab dog and lab member, for consistent
morale support. Liv G.~d'Aliberti is especially grateful to Julia Galliker and
to Arielle d'Aliberti, whom Liv married while this paper was being written, for
their love and support.

\subsubsection*{AI Use Statement}
\label{sec:ai-use}

Generative AI assisted manuscript editing, mathematical analysis, and supporting code.
The authors reviewed and verified the resulting content and take responsibility for the final manuscript.

\subsubsection*{Ethics Statement}

The benchmark and experiment-specific data are synthetic; this study collects no human-subject or personal data.
More verified modes are not inherently better: a distinct execution may be risky or costly even when it passes the validator, and \pmd{} does not measure downstream value.
Systems that evaluate or optimize breadth should report correctness, invalidity, and application-specific costs alongside mode counts.

\subsubsection*{Reproducibility Statement}

Every reported number is computed from a versioned result file.
App.~\ref{app:reproducibility} documents evaluation provenance and links the released datasets, checkpoints, and analysis resources.
The AI Use Statement above discloses generative-AI assistance.


\clearpage
\phantomsection\label{sec:references}
\bibliographystyle{iclr2027_conference}
\bibliography{example_paper}

\long\def\crossscaletablematerial{%

\begin{table}[!bp]
  \centering
  \caption{\textbf{Terminal correctness and solution diversity with and without replay.}
  Dr.GRPO and Re:Dr after eight training passes. Correctness uses five paired
  seeds except Falcon \texttt{Countdown} ($n=4$). Each diversity column averages its
  own eligible seeds, requiring at least 30 prompts with two correct responses
  per seed; superscripts give counts below five. Dashes indicate no eligible
  estimate. Arm-specific means need not subtract to the paired effects in
  Table~\ref{tab:cross-scale-terminal-effects}, which uses a 20-prompt threshold
  per arm. Red-to-green shading uses the same scale for both arms within
  each metric.}

\label{tab:core-terminal-endpoints}
  \scriptsize
  \setlength{\tabcolsep}{5pt}
  \begin{tabular}{@{}ll*{4}{r}@{}}
    \toprule
    \textbf{Model} & \textbf{Domain} &
      \multicolumn{2}{c}{\textbf{Dr.GRPO}} &
      \multicolumn{2}{c}{\textbf{Re:Dr}} \\
    \cmidrule(lr){3-4}\cmidrule(l){5-6}
      & & \textbf{\texttt{pass@8}} & \textbf{\pmd{}}
      & \textbf{\texttt{pass@8}} & \textbf{\pmd{}} \\
    \midrule
        \qwenmark{}2.5-0.5B & \texttt{Graph} coloring & \cellcolor[HTML]{F5D79E}.323 & \cellcolor[HTML]{F0A58F}.001 & \cellcolor[HTML]{ABD7A0}.969 & \cellcolor[HTML]{A9D6A0}.560 \\
     & \texttt{Countdown} & \cellcolor[HTML]{F8EFA6}.480 & \cellcolor[HTML]{F0A790}.007 & \cellcolor[HTML]{DBE6A4}.672 & \cellcolor[HTML]{BDDDA2}.487 \\
     & \texttt{Python} factors & \cellcolor[HTML]{F3BF97}.172 & \textemdash{} & \cellcolor[HTML]{F2EEA6}.528 & \cellcolor[HTML]{F0A58F}.000$^{\,2}$ \\
     & \texttt{MathIR} & \cellcolor[HTML]{F4EFA6}.512 & \cellcolor[HTML]{F0A58F}.001 & \cellcolor[HTML]{C6E0A2}.796 & \cellcolor[HTML]{F1AA90}.018 \\
     & \texttt{PantryPlan} & \cellcolor[HTML]{F3EEA6}.522 & \cellcolor[HTML]{F0A58F}.000 & \cellcolor[HTML]{D1E3A3}.729 & \cellcolor[HTML]{E9EBA5}.334 \\
    \addlinespace[2pt]
    \falconmark{}3-1B & \texttt{Graph} coloring & \cellcolor[HTML]{DDE7A4}.656 & \cellcolor[HTML]{F6DCA0}.204 & \cellcolor[HTML]{C5DFA2}.809 & \cellcolor[HTML]{CEE2A3}.428 \\
     & \texttt{Countdown}$^{\,4}$ & \cellcolor[HTML]{F5D79E}.325 & \cellcolor[HTML]{F3C097}.102 & \cellcolor[HTML]{E1E8A4}.632 & \cellcolor[HTML]{F4C99A}.134 \\
     & \texttt{Python} factors & \cellcolor[HTML]{F0A58F}.000 & \textemdash{} & \cellcolor[HTML]{A9D6A0}.979 & \cellcolor[HTML]{F2B594}.060 \\
     & \texttt{MathIR} & \cellcolor[HTML]{F2B493}.095 & \textemdash{} & \cellcolor[HTML]{D0E3A3}.736 & \cellcolor[HTML]{F1AB91}.022 \\
     & \texttt{PantryPlan} & \cellcolor[HTML]{E1E8A4}.633 & \cellcolor[HTML]{F0A58F}.000 & \cellcolor[HTML]{CEE2A3}.751 & \cellcolor[HTML]{EDEDA5}.317 \\
    \addlinespace[2pt]
    \qwenmark{}2.5-3B & \texttt{Graph} coloring & \cellcolor[HTML]{D9E6A4}.682 & \cellcolor[HTML]{F6E1A1}.224 & \cellcolor[HTML]{C0DEA2}.836 & \cellcolor[HTML]{D6E5A3}.399 \\
     & \texttt{Countdown} & \cellcolor[HTML]{E5EAA5}.609 & \cellcolor[HTML]{F0A890}.012 & \cellcolor[HTML]{CFE2A3}.746 & \cellcolor[HTML]{F3C398}.112 \\
     & \texttt{Python} factors & \cellcolor[HTML]{F1B092}.069 & \textemdash{} & \cellcolor[HTML]{DDE7A4}.659 & \cellcolor[HTML]{F3BF97}.098 \\
     & \texttt{MathIR} & \cellcolor[HTML]{F5D09C}.278 & \cellcolor[HTML]{F0A58F}.001 & \cellcolor[HTML]{D8E6A4}.686 & \cellcolor[HTML]{F0A990}.014 \\
     & \texttt{PantryPlan} & \cellcolor[HTML]{E0E8A4}.638 & \cellcolor[HTML]{F3C198}.105 & \cellcolor[HTML]{CFE2A3}.744 & \cellcolor[HTML]{E4E9A4}.351 \\
    \bottomrule

  \end{tabular}
\end{table}

}

\appendix
\onecolumn
\raggedbottom

\section*{Appendix Organization}
\label{app:organization}

The appendix is organized into eight parts, with sections and page numbers listed below.

\vspace{0.6em}

\begingroup
\providecommand{\apxpart}[1]{\par\vspace{2pt}\noindent\begingroup\parfillskip=0pt plus 1fil\relax\textbf{#1}\par\endgroup}
\providecommand{\apxdots}{\nobreak\leaders\hbox{\hspace{0.26em}.\hspace{0.26em}}\hfill}
\providecommand{\apxline}[3]{\par\noindent\begingroup
  \rightskip=2.1em plus 1fil\relax \parfillskip=-\rightskip
  \hangindent=#1\hangafter=1 \leavevmode
  #2\apxdots\nobreak\makebox[2.1em][r]{#3}\par\endgroup}
\providecommand{\apxsec}[3]{\apxline{2.4em}{\makebox[2.2em][l]{#1}#2}{#3}}
\providecommand{\apxsub}[3]{\apxline{4.0em}{\hspace*{1.4em}\makebox[2.6em][l]{#1}#2}{#3}}

\scriptsize
\linespread{0.97}\selectfont
\setlength{\parindent}{0pt}
\setlength{\parskip}{0pt}

\apxpart{Part I. Scope and measurement}
\apxsec{\ref{app:metric-details}}{Metric and Objective Details}{\pageref{app:metric-details}}
\apxsub{\ref{app:metric-sampling}}{Measuring success and diversity}{\pageref{app:metric-sampling}}
\apxsub{\ref{app:reported-endpoints}}{Reported endpoints and eligibility}{\pageref{app:reported-endpoints}}
\apxsub{\ref{app:replay-metric-alignment}}{Why balance solution modes in replay}{\pageref{app:replay-metric-alignment}}

\apxpart{Part II. Benchmark and method}
\apxsec{\ref{app:data}}{Benchmark and Evaluation Protocol}{\pageref{app:data}}
\apxsub{\ref{app:mode-keying}}{Solution-mode definitions}{\pageref{app:mode-keying}}
\apxsub{\ref{app:data-materialization}}{Dataset splits and validation}{\pageref{app:data-materialization}}
\apxsub{\ref{app:python}}{\texttt{Python} execution and solution modes}{\pageref{app:python}}
\apxsub{\ref{app:data-levels}}{Matched level constructions}{\pageref{app:data-levels}}
\apxsub{\ref{app:base-levels-all-scales}}{Performance across benchmark levels}{\pageref{app:base-levels-all-scales}}
\apxsub{\ref{app:repro-run-contract}}{Training and evaluation settings}{\pageref{app:repro-run-contract}}

\apxsec{\ref{app:prompts}}{Exact Domain Prompts}{\pageref{app:prompts}}
\apxsub{\ref{app:prompt-graph}}{\texttt{Graph} coloring}{\pageref{app:prompt-graph}}
\apxsub{\ref{app:prompt-countdown}}{\texttt{Countdown}}{\pageref{app:prompt-countdown}}
\apxsub{\ref{app:prompt-python}}{\texttt{Python} factors}{\pageref{app:prompt-python}}
\apxsub{\ref{app:prompt-mathir}}{\texttt{MathIR} action menu}{\pageref{app:prompt-mathir}}
\apxsub{\ref{app:prompt-pantry}}{\texttt{PantryPlan}}{\pageref{app:prompt-pantry}}

\apxsec{\ref{app:algorithm}}{Replay Algorithm}{\pageref{app:algorithm}}
\apxsub{\ref{app:binary-maxrl}}{Fresh task advantages}{\pageref{app:binary-maxrl}}
\apxsub{\ref{app:algorithm-update}}{Optimizer update}{\pageref{app:algorithm-update}}
\apxsub{\ref{app:algorithm-bank}}{Buffer construction and scheduling}{\pageref{app:algorithm-bank}}
\apxsub{\ref{app:algorithm-control}}{Matched control}{\pageref{app:algorithm-control}}
\apxsub{\ref{app:algorithm-identity}}{Verifier requirements and applicability}{\pageref{app:algorithm-identity}}
\apxsub{\ref{app:algorithm-quality}}{Exemplar quality and length}{\pageref{app:algorithm-quality}}
\apxsub{\ref{app:semantic-maxent}}{Semantic-MaxEnt implementation}{\pageref{app:semantic-maxent}}

\apxpart{Part III. Training results}
\apxsec{\ref{app:supporting}}{Replay Effects Across Scales}{\pageref{app:supporting}}
\apxsub{\ref{app:scale-matrix}}{Terminal comparisons across scales}{\pageref{app:scale-matrix}}
\apxsub{\ref{app:qwen7b-level1}}{Level 1 at Qwen2.5-7B}{\pageref{app:qwen7b-level1}}
\apxsub{\ref{app:level2-interim}}{Results on harder benchmark levels}{\pageref{app:level2-interim}}
\apxsub{\ref{app:factorial-training-curves}}{Diversity throughout training}{\pageref{app:factorial-training-curves}}
\apxsub{\ref{app:ucpo-progress}}{Supporting comparators}{\pageref{app:ucpo-progress}}
\apxsub{\ref{app:diversity-comparators}}{Diversity-preserving comparators}{\pageref{app:diversity-comparators}}
\apxsub{\ref{app:kl-measured}}{Reference KL tradeoffs}{\pageref{app:kl-measured}}
\apxsub{\ref{app:frequency-replay-progress}}{Uniform versus fresh-frequency replay}{\pageref{app:frequency-replay-progress}}

\apxsec{\ref{app:ablations}}{Concentration and Replay Mechanisms}{\pageref{app:ablations}}
\apxsub{\ref{app:fresh-concentration}}{Fresh 64-response evaluation}{\pageref{app:fresh-concentration}}
\apxsub{\ref{app:conditional-concentration}}{Conditional-concentration robustness}{\pageref{app:conditional-concentration}}
\apxsub{\ref{app:pre-rl-regression}}{Diversity before and after RLVR-only training}{\pageref{app:pre-rl-regression}}
\apxsub{\ref{app:degenerate-groups}}{Sparse task gradients and output repetition}{\pageref{app:degenerate-groups}}
\apxsub{\ref{app:active-controls}}{Replay gains with active and tuned controls}{\pageref{app:active-controls}}
\apxsec{\ref{app:code}}{Human-Authored Programming Evaluation}{\pageref{app:code}}
\apxsec{\ref{app:standard-math}}{Standard Mathematical Reasoning Evaluation}{\pageref{app:standard-math}}

\apxpart{Part IV. What diversity represents}
\apxsec{\ref{app:prompt-hint-ablation}}{Prompt-Guidance Ablation}{\pageref{app:prompt-hint-ablation}}
\apxsec{\ref{app:decoding-objection}}{Inference-Time Decoding Controls}{\pageref{app:decoding-objection}}

\apxsec{\ref{app:sampling-budget-ablation}}{Sampling-Budget Ablation}{\pageref{app:sampling-budget-ablation}}
\apxsub{\ref{sec:sampling-budget-design}}{Sampling-budget design}{\pageref{sec:sampling-budget-design}}
\apxsub{\ref{app:sampling-budget-effects}}{Discovery at larger sampling budgets}{\pageref{app:sampling-budget-effects}}
\apxsub{\ref{app:gpt56-all-levels-discovery}}{Discovery at 512 draws}{\pageref{app:gpt56-all-levels-discovery}}

\apxsec{\ref{app:inference-followups}}{Inference-Time Alternatives and Mode Definitions}{\pageref{app:inference-followups}}
\apxsub{\ref{app:cross-model-overlap}}{Complementarity across hosted models}{\pageref{app:cross-model-overlap}}
\apxsub{\ref{app:coarser-keys}}{Sensitivity to solution-mode identity}{\pageref{app:coarser-keys}}

\apxsec{\ref{app:portfolio-withdrawals}}{Portfolio Survival under Withdrawn Options}{\pageref{app:portfolio-withdrawals}}
\apxsub{\ref{app:withdrawal-severity}}{Two withdrawals at once}{\pageref{app:withdrawal-severity}}
\apxsub{\ref{app:withdrawal-decoding}}{Decoding settings as a substitute}{\pageref{app:withdrawal-decoding}}
\apxsub{\ref{app:withdrawal-recovery}}{Recovery after a withdrawal}{\pageref{app:withdrawal-recovery}}

\apxpart{Part V. Hosted deployment analyses}
\apxsec{\ref{app:hosted-concentration}}{Inference-Only Concentration in GPT-5.6 Sol}{\pageref{app:hosted-concentration}}

\apxsec{\ref{app:hosted-comparison}}{Comparison across Hosted Deployments}{\pageref{app:hosted-comparison}}
\apxsub{\ref{app:hosted-mode-diversity}}{Hosted-model concentration}{\pageref{app:hosted-mode-diversity}}
\apxsub{\ref{app:hosted-python-prompt}}{\texttt{Python} prompt-interface sensitivity}{\pageref{app:hosted-python-prompt}}
\apxsub{\ref{app:hosted-reasoning-control}}{Reasoning-control comparison}{\pageref{app:hosted-reasoning-control}}
\apxsub{\ref{app:hosted-temperature}}{Temperature sensitivity across deployments}{\pageref{app:hosted-temperature}}
\apxsub{\ref{app:hosted-provider-outcomes}}{Native provider outcomes}{\pageref{app:hosted-provider-outcomes}}

\apxpart{Part VI. Mechanism and efficiency}
\apxsec{\ref{app:theory}}{Mode Collapse and Verified Replay}{\pageref{app:theory}}
\apxsub{\ref{app:theory-grpo-mean}}{A mode-level model of RLVR}{\pageref{app:theory-grpo-mean}}
\apxsub{\ref{app:theory-gradient-availability}}{Rare modes become difficult to recover}{\pageref{app:theory-gradient-availability}}
\apxsub{\ref{app:theory-replay}}{Replay preserves stored modes}{\pageref{app:theory-replay}}
\apxsub{\ref{app:theory-evidence-map}}{Scope of the analysis}{\pageref{app:theory-evidence-map}}

\apxsec{\ref{app:replay-cost}}{Compute and Storage Overhead}{\pageref{app:replay-cost}}
\apxsub{\ref{app:compute-matching}}{Control design and compute matching}{\pageref{app:compute-matching}}
\apxsub{\ref{app:compute-accounting}}{Replay storage and runtime}{\pageref{app:compute-accounting}}
\apxsub{\ref{app:compute-budget-control}}{Longer-horizon control}{\pageref{app:compute-budget-control}}
\apxsub{\ref{app:compute-budget-groups}}{Wider-group control}{\pageref{app:compute-budget-groups}}
\apxsub{\ref{app:mode-agnostic-replay}}{Single-slot replay}{\pageref{app:mode-agnostic-replay}}

\apxpart{Part VII. Related-method comparison}
\apxsec{\ref{app:related-extended}}{Comparison with Diversity-Preserving RLVR Methods}{\pageref{app:related-extended}}

\apxpart{Part VIII. Data and reproducibility}
\apxsec{\ref{app:reproducibility}}{Data and Reproducibility}{\pageref{app:reproducibility}}
\apxsub{\ref{app:source-integrity}}{Evaluation pairing}{\pageref{app:source-integrity}}
\apxsub{\ref{app:repro-artifacts}}{Released datasets and checkpoints}{\pageref{app:repro-artifacts}}
\apxsub{\ref{app:repro-recovery}}{Analysis resources}{\pageref{app:repro-recovery}}

\par
\endgroup

\clearpage
\phantomsection
\begin{center}\large\textbf{Part I. Scope and measurement}\end{center}
\addcontentsline{toc}{part}{Part I. Scope and measurement}

\section{Metric and Objective Details}
\label{app:metric-details}

Correctness and verified-mode diversity are distinct properties of a policy.
For a prompt $x$ with reference specification $s_x$, the verifier maps a response to either failure or a verified solution mode:
\[
  V(\cdot,s_x):\mathcal Y\rightarrow \mathcal C_x\cup\{\bot\}.
\]
Here $\mathcal C_x$ is the set of valid solution modes and $\bot$ denotes failure.
The learner observes a solution mode only for a response it generates; neither the full set $\mathcal C_x$ nor its size is available during training.

For $Y\sim\pi_\theta(\cdot\mid x)$, define the probability of a correct response and the conditional distribution over the solution mode of a verified correct response:
\[
 P_\theta^+(x)=\Pr[V(Y,s_x)\neq\bot],
 \qquad
 p_\theta^+(c\mid x)=\Pr[V(Y,s_x)=c\mid V(Y,s_x)\neq\bot].
\]
We use the abbreviations $P$ and $q_c$ from the main text when the policy
and prompt are clear. The conditional distribution is defined only when
$P>0$. If $P=0$, \texttt{pass@}$K$, \texttt{distinct@}$K$, and
\texttt{mean@}$K$ are all zero, whereas \pmd{} is undefined.

We use \emph{mode collapse} to describe increasing concentration over correct solution modes, particularly when correctness remains stable or improves.
Transferring probability from a less probable correct mode to a more probable one decreases $\pmd{}=1-\sum_c q_c^2$.
Accordingly, \pmd{} measures concentration among sampled correct modes rather than literal extinction of alternatives; finite evaluations characterize observed policy breadth.

Estimating \pmd{} requires two verified responses per prompt.
We distinguish well-supported estimates with at least 30 eligible prompts from estimates based on fewer prompts; paired training comparisons require at least 20 eligible evaluation prompts in each arm (App.~\ref{app:reported-endpoints}).

\subsection{Measuring success and diversity}
\label{app:metric-sampling}

For a prompt, let $n_m$ be the number of verified responses assigned to solution mode $m$, and let $n_+=\sum_m n_m$ be the total number of verified responses.
When $n_+\ge2$, we estimate pairwise mode diversity as
\begin{equation}
  \widehat{\pmd}
  =
  1-
  \frac{\sum_m \binom{n_m}{2}}
       {\binom{n_+}{2}}
  =
  1-
  \frac{\sum_m n_m(n_m-1)}
       {n_+(n_+-1)}.
  \label{eq:main-pmd-estimator}
\end{equation}
Thus \pmd{} is the probability that two independently sampled correct responses represent different solution modes.
Under independent sampling from a fixed policy, this estimator is unbiased conditional on $n_+\ge2$.

We average prompt-level \pmd{} estimates with equal weight across eligible prompts rather than pooling response pairs across prompts.
This prevents easier prompts, which produce more verified pairs, from receiving greater weight.
Local evaluations generally use multiple response groups per prompt, while hosted evaluations use eight responses per prompt.

For $K$ sampled responses, we additionally report
\[
  \texttt{pass@}K
  =
  \Pr(\text{at least one response is correct}),
\]
\[
\texttt{distinct@}K
  =
  \E[\text{number of distinct verified solution modes}],
\]
and \texttt{mean@}$K$, the mean fraction of correct responses.
Unlike \pmd{}, \texttt{distinct@}$K$ depends on both correctness and how correct probability is distributed across solution modes.

\pmd{} is the Gini--Simpson diversity index applied to the conditional distribution over correct solution modes~\citep{simpson1949measurement}.
For interpretation, we also use the effective number of solution modes
\begin{equation}
  \neff
  =
  \frac{1}{1-\pmd},
\end{equation}
the inverse Simpson index or Hill number of order two~\citep{hill1973diversity,jost2006entropy}.
A value of $\neff=k$ corresponds to the diversity of a uniform distribution over $k$ solution modes.

\subsection{Reported endpoints and eligibility}
\label{app:reported-endpoints}

We report \texttt{pass@}$K$ for correctness, \texttt{distinct@}$K$ for sampled solution breadth, and \pmd{} for diversity conditional on producing a correct response.
Unlike \texttt{distinct@}$K$, \pmd{} depends only on the conditional distribution over correct solution modes and therefore separates solution diversity from overall correctness.
A model can consequently achieve high \texttt{pass@8} while repeatedly producing the same correct solution mode.

A prompt-level \pmd{} estimate requires at least two verified responses.
For evaluation summaries, filled markers indicate at least 30 eligible prompts, hollow markers indicate 5--29, and estimates based on fewer than five prompts are omitted.
Missing \pmd{} values therefore indicate insufficient verified responses rather than zero diversity.

For paired training comparisons, both arms must contain at least 20 eligible prompts.
Unless otherwise stated, each arm is averaged over its own eligible prompt set, so differences can reflect both the conditional solution distribution and which prompts satisfy the eligibility criterion.

We also report
$\texttt{distinct@}K-\texttt{pass@}K$
as the expected number of verified solution modes beyond the first.
This quantity requires no knowledge of the full set of valid solution modes but, unlike \pmd{}, remains coupled to correctness.

\subsection{Why balance solution modes in replay}
\label{app:replay-metric-alignment}

Replay gives each stored solution mode equal weight, regardless of how often that mode was discovered.
A simple probability decomposition clarifies the effect of this choice.
For a buffer $\mathcal B_x$ containing $k$ solution modes, let $p(m)$ denote the probability of mode $m$, let $P_{\mathcal B}$ be their total probability mass, and let $U_k$ be uniform over the buffered modes.
Then
\begin{equation}
  -\frac{1}{k}\sum_{m\in\mathcal B_x}\log p(m)
  =
  -\log P_{\mathcal B}
  +
  \log k
  +
  \mathrm{KL}\!\left(
    U_k
    \,\middle\|\,
    p(m\mid m\in\mathcal B_x)
  \right).
  \label{eq:replay-bank-decomposition}
\end{equation}

The decomposition therefore shows that the replay objective favors both greater total probability on buffered solution modes and a more balanced distribution among them.
At fixed buffered probability mass, balancing the modes increases the expected number of distinct buffered modes observed in repeated samples, with the maximum attained by the uniform allocation.
This provides the motivation for weighting discovered solution modes equally rather than in proportion to how frequently fresh rollouts rediscover them.

The implemented replay loss operates on one stored exemplar per solution mode, so the decomposition gives the mode-level motivation for the training objective.
The fresh-rollout frequency ablation in App.~\ref{app:frequency-replay-progress} then tests this balancing principle directly in neural training.

\clearpage
\phantomsection
\begin{center}\large\textbf{Part II. Benchmark and method}\end{center}
\addcontentsline{toc}{part}{Part II. Benchmark and method}

\section{Benchmark and Evaluation Protocol}
\label{app:data}
\label{app:experimental-design}
\label{app:domain-overview}
\label{app:data-inventory}

\mb{} uses executable verifiers to determine both correctness and solution-mode identity.
Invalid responses receive reward zero, while verified responses receive reward one and are assigned to a solution mode.
The learner observes only modes that it generates and is never given the complete set or number of valid modes.
These quantities are not used to schedule replay or determine when training stops.
The benchmark implementation is available in the
\href{https://github.com/liv-daliberti/modeBench}{public \mb{} repository},
and the frozen datasets are released on
\href{https://huggingface.co/datasets/od2961/ModeBench}{Hugging Face};
the training experiments in this paper use Levels~1--3.

\begin{table}[!htbp]
  \centering
  \caption{\textbf{Executable solution modes in the five \mb{} domains.}
  The verifier determines both correctness and solution-mode identity.
  \emph{Certified modes} gives the mean and range of the Level-1 evaluation catalogue; these catalogues need not enumerate every accepted solution.}
  \label{tab:tasks}
  \scriptsize
  \setlength{\tabcolsep}{5pt}
  \renewcommand{\arraystretch}{1.08}
  \begin{tabularx}{\linewidth}{@{}l
      >{\hsize=1.35\hsize}Y
      >{\hsize=.9\hsize}Y
      >{\hsize=1.0\hsize}Y
      r@{}}
    \toprule
    \textbf{Domain} & \textbf{Generated answer / check} &
    \textbf{Solution mode} & \textbf{Equivalent responses} &
    \textbf{Certified modes} \\
    \midrule
    \texttt{Graph} coloring &
    Complete assignment; check fixed colors and edges &
    Color vector &
    Spacing only &
    \shortstack{6.27\\{\scriptsize 4--12}} \\
    \addlinespace[2pt]
    \texttt{Countdown} &
    Parse operands and execute exactly &
    Normalized AST &
    Commutative order &
    \shortstack{4.52\\{\scriptsize 2--8}} \\
    \addlinespace[2pt]
    \texttt{Python} factors &
    Execute restricted lambda externally &
    Return vector &
    Programs with the same return vector &
    \shortstack{229.44\\{\scriptsize 16--1{,}512}} \\
    \addlinespace[2pt]
    \texttt{MathIR} &
    Execute bounded action program &
    State trajectory &
    Programs with the same states &
    \shortstack{5.00\\{\scriptsize 5}} \\
    \addlinespace[2pt]
    \texttt{PantryPlan} &
    Check ingredient allocation and all constraints &
    Ingredient support &
    Quantity and ordering aliases &
    \shortstack{18.19\\{\scriptsize 8--40}} \\
    \bottomrule
  \end{tabularx}
\end{table}

\subsection{Solution-mode definitions}
\label{app:mode-keying}

The five domains instantiate different notions of a distinct correct solution.
In \texttt{Countdown} and \texttt{MathIR}, solution modes correspond to different executable routes to an answer.
In \texttt{Python}, they correspond to different return vectors on the fixed inputs specified by the prompt.
In \texttt{Graph} and \texttt{PantryPlan}, they correspond to different valid answers.
These definitions identify verified solution behavior rather than latent reasoning processes.

\paragraph{Sensitivity to the mode definition.}
The chosen equivalence relation determines which verified responses count as the same solution mode.
App.~\ref{app:coarser-keys} therefore repeats the analysis under coarser task-relevant definitions.
Most replay effects persist, although \texttt{Python} is more sensitive to the equivalence relation than the other domains.
Training always uses the benchmark's original solution-mode definitions; the coarser analysis is used only as a robustness check.

\paragraph{Executable-answer setting.}
Training prompts request executable answers rather than unrestricted reasoning traces, allowing solution-mode identity to be determined without a learned judge or embedding-similarity threshold.
The experiments therefore measure diversity among executable solutions, not diversity in latent reasoning or unrestricted chain-of-thought.
Evaluating those broader notions would require a separate definition and validation of solution identity.

\subsection{Dataset splits and validation}
\label{app:data-materialization}

Each domain uses disjoint training and evaluation splits (Table~\ref{tab:run-contract}), with additional development splits for \texttt{PantryPlan} and Level~2 (App.~\ref{app:data-levels}).
\texttt{Python} and \texttt{MathIR} each contain 384 training and 128 evaluation problems.
Every \texttt{Python} problem is validated to admit at least two distinct correct return vectors, and every certified \texttt{MathIR} trajectory is checked using the same interpreter applied to model outputs.

\paragraph{\texttt{PantryPlan}.}
\label{app:data-pantry}
\texttt{PantryPlan} contains 384 training, 64 development, and 128 evaluation prompts across four problem families.
Evaluation prompts admit a mean of 18.19 feasible ingredient supports, ranging from 8 to 40.
The verifier uses exact arithmetic to check inventory, serving increments, mass, nutrient bounds, and dietary restrictions, and defines the solution mode by the set of ingredients used.
These checks establish satisfaction of the benchmark constraints rather than taste, preparation feasibility, or food safety.

\subsection{\texttt{Python} execution and solution modes}
\label{app:python}
\label{app:python-language}

The \texttt{Python} verifier accepts restricted expressions of the form \texttt{lambda n: EXPR}, limited to 240 characters and 64 AST nodes.
Expressions may use integer arithmetic, comparisons, Boolean operators, and conditionals, but may not use calls, imports, attributes, comprehensions, containers, assignments, loops, or identifiers other than \texttt{n}.
Level-1 prompts contain four fixed inputs, and every problem admits at least two valid solution modes.

\paragraph{External execution.}
\label{app:python-process}
Candidate expressions execute in an isolated Python process with empty \texttt{\_\_builtins\_\_} and fixed time limits.
Malformed expressions, execution failures, timeouts, non-integer outputs, and failed divisor checks receive zero reward.

\paragraph{Solution-mode identity.}
\label{app:python-identity}
For fixed inputs $n_1,\ldots,n_J$, a verified response produces integers $d_1,\ldots,d_J$ satisfying
\[
    1<d_j<n_j,
    \qquad
    d_j\mid n_j
    \qquad
    \text{for all }j.
\]
Its solution mode is the ordered return vector $(d_1,\ldots,d_J)$, so two programs represent the same mode exactly when they return the same vector on all benchmark inputs.
The exact number of valid return vectors is used only to validate the dataset and is never exposed to the learner.

\subsection{Matched level constructions}
\label{app:data-levels}

Levels~2--5 increase task complexity while preserving the verifier and solution-mode definition within each domain.
Level~2 supports the harder-task training experiments, while Levels~3--5 support scale-calibrated evaluations.

\paragraph{Level 2.}
Level~2 contains 384 training, 128 development, and 128 evaluation problems, with no shared problem identities across splits or with Level~1.
Its training-set distribution of certified solution-mode counts is matched to the Level-1 construction reference.
The evaluation populations are not identically distributed across levels in every domain, so cross-level differences should not be interpreted as the isolated causal effect of difficulty.
Within each level, replay and control conditions use the same prompts, seeds, fresh objective, and training budget.

\begin{table}[!htbp]
  \centering
  \caption{\textbf{Level~2 increases task complexity while preserving solution-mode identity.}
  The verifier and definition of a solution mode are unchanged from Level~1.}
  \label{tab:level2-structure}
  \small
  \begin{tabularx}{.96\linewidth}{@{}lYY@{}}
    \toprule
    \textbf{Domain} & \textbf{Level 1} & \textbf{Level 2} \\
    \midrule
    \texttt{Graph} coloring & Three hidden vertices & Four hidden vertices \\
    \texttt{Countdown} & Three operands & Four operands, with values up to 14 \\
    \texttt{Python} factors & Four executed cases, all at most 96 & Four cases, including a value above 96 \\
    \texttt{MathIR} & One-sided linear families & Variable-on-both-sides families \\
    \texttt{PantryPlan} & Base dietary and composition constraints & More active exclusions and tighter composition constraints \\
    \bottomrule
  \end{tabularx}
\end{table}

\paragraph{Scale-calibrated Levels 3--5.}
Levels~3--5 are constructed so that progressively larger fixed models face approximately the same correctness difficulty as Qwen2.5-0.5B on Level~1.
Level~3 uses Qwen2.5-3B, Level~4 uses Qwen2.5-7B, and Level~5 uses Qwen2.5-14B.
Construction parameters are selected on development problems using \texttt{pass@1} and \texttt{pass@8}, with tolerances of $.04$ and $.08$ from the Level-1 reference; solution diversity is not part of the calibration objective.
Level~5 satisfies these criteria in all five domains, while Level~4 satisfies them in all domains except \texttt{MathIR}.
We therefore include Level-4 \texttt{MathIR} where relevant but do not treat it as difficulty-matched to the Level-1 reference.

\subsection{Performance across benchmark levels}
\label{app:base-levels-all-scales}

We evaluate untrained checkpoints across all five benchmark levels to characterize how correctness and solution diversity vary with task construction before RLVR training.
Each local-model evaluation uses 128 held-out prompts per domain and level, with four independent groups of eight responses per prompt.
For \texttt{pass@8} and \distk{}, we average the four groups within each prompt, while \pmd{} is estimated from the pooled responses.

Higher levels generally reduce correctness, although the relationship is not strictly monotone for every model and domain.
Solution diversity does not vary in lockstep with correctness or model scale, so benchmark difficulty cannot be inferred from \pmd{} alone.
These measurements describe the benchmark populations and are separate from the matched training comparisons, which always compare replay and control under the same evaluation conditions.

\begin{figure}[!htbp]
  \centering
  \includegraphics[width=\linewidth]{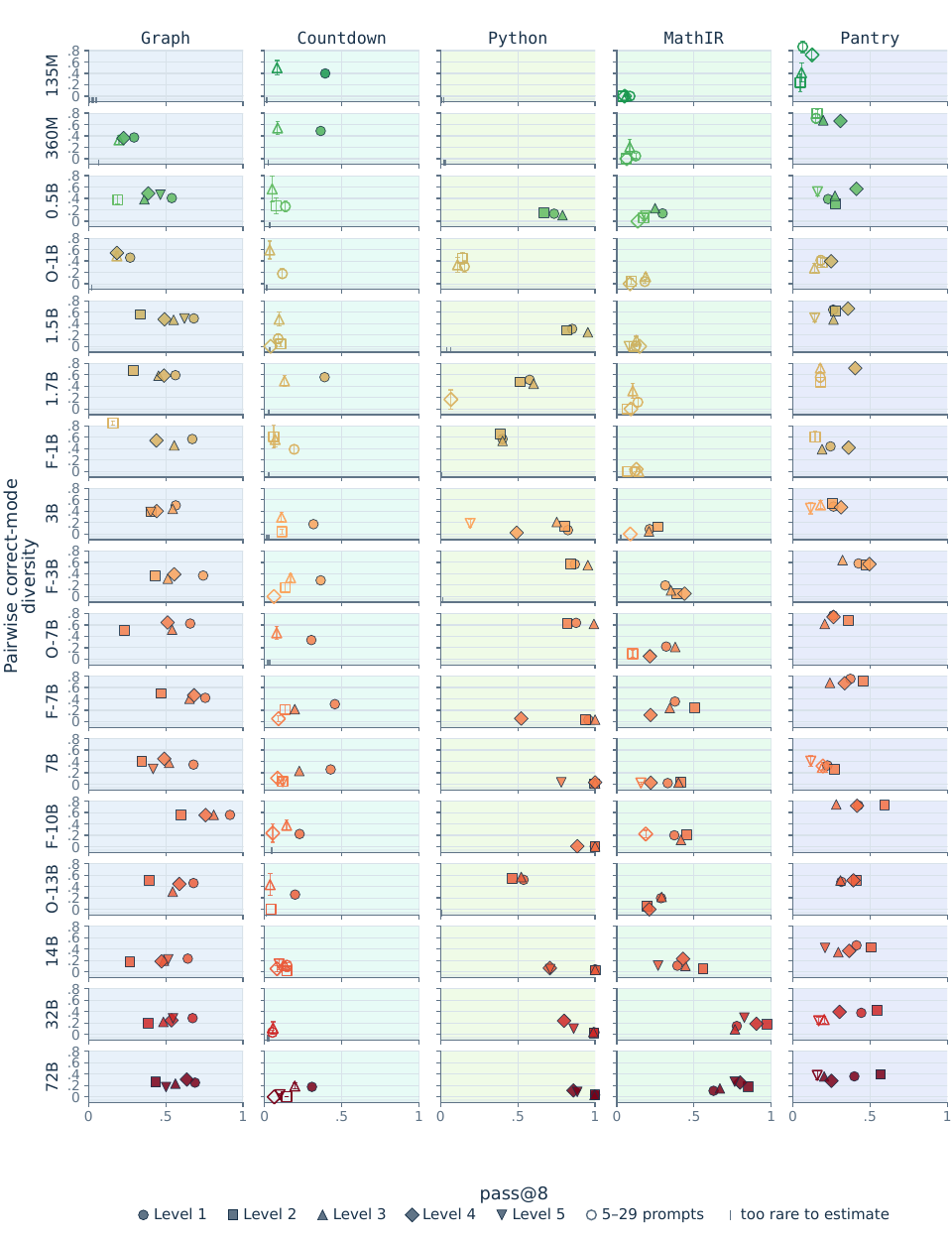}
  \caption{\textbf{Correctness and conditional solution diversity vary separately across benchmark levels.}
  Each point is an untrained model evaluated on one domain and level, with \texttt{pass@8} on the horizontal axis and \pmd{} on the vertical axis.
  Filled markers indicate sufficient verified responses for the primary diversity estimate, while hollow markers indicate limited support.}
  \label{fig:level-construction}
\end{figure}

\subsection{Training and evaluation settings}
\label{app:repro-run-contract}
\label{app:comparison-design}

\begin{table}[!htbp]
  \centering
  \caption{\textbf{Training and evaluation settings for the matched replay comparisons.}
  Replay and control use the same initialization, fresh-rollout budget, optimizer updates, prompts, and evaluation procedure.}
  \label{tab:run-contract}
  \small
  \begin{tabularx}{\linewidth}{@{}P{.22\linewidth}Y@{}}
    \toprule
    \textbf{Component} & \textbf{Setting} \\
    \midrule
    Data &
    384 training and 128 disjoint evaluation prompts per domain and level; evaluation prompts are never used for training or replay. \\
    Fresh rollouts &
    $G=16$, temperature $1$, top-$p=1$, with one prompt group per optimizer update. \\
    Training horizon &
    Eight passes over the training set, totaling 3,072 optimizer updates, with one PPO epoch per batch. \\
    Evaluation &
    Four groups of eight responses per prompt at temperature $1$ and top-$p=1$ at passes $0,.5,\ldots,8$; final comparisons use the last checkpoint. \\
    Replay &
    At most $M=16$ solution modes per prompt, one stored exemplar per mode, deterministic round-robin buffer selection, and $\lambda_{\mathrm{replay}}=.10$. \\
    Training seeds &
    Five seeds per model--domain condition, with paired comparisons using common seeds whenever available. \\
    Learning rate &
    Qwen2.5-0.5B and Falcon3-1B use a constant $2\times10^{-7}$ learning rate; Qwen2.5-3B uses cosine decay from $10^{-7}$ to $10^{-8}$ with 10\% warmup. \\
    Optimizer &
    Adam in \texttt{bf16} with $\beta_1=.9$, $\beta_2=.95$, zero weight decay, gradient-norm clipping at $1.0$, and gradient checkpointing. \\
    Policy update &
    Parameters are synchronized after each update, each rollout is used once, and the symmetric PPO clipping threshold is $\varepsilon=.2$. \\
    Response tokens &
    \texttt{Graph}/\texttt{Countdown}: 192; \texttt{Python}: 192/512/192 at 0.5B/1B/3B; \texttt{MathIR}: 64/128/64; \texttt{PantryPlan}: 8. \\
    \bottomrule
  \end{tabularx}
\end{table}

Replay and control additionally execute the same validation, buffer construction, scheduling, and exemplar-scoring code, with the replay gradient set to zero in the control.
They therefore match optimizer updates and fresh-rollout budgets but not total computation or wall-clock time (App.~\ref{app:compute-matching}).

For each seed, \texttt{pass@8} and \texttt{distinct@8} are averaged over four eight-response groups for each of the 128 evaluation prompts.
Evaluation responses are excluded from training and replay and do not affect checkpoint selection or stopping.

Final-checkpoint comparisons use paired training seeds.
When all five paired seeds are available, we report nominal 95\% Student-$t$ intervals; comparisons with fewer seeds report their sample size and mean.
Diversity comparisons additionally apply the eligibility criteria stated with each analysis.

\FloatBarrier
\section{Exact Domain Prompts}
\label{app:prompts}

Each domain's training prompt contains one system message and one user
message, without few-shot examples. Training and evaluation use the same
template within each method. The examples below reproduce the first prompt
from each Level-1 evaluation split in Qwen's chat format. Falcon3 uses the
same message text with its own \texttt{<|system|>} and \texttt{<|user|>}
markers. Models evaluated before benchmark training use their tokenizer's
native chat template (App.~\ref{app:base-levels-all-scales}).
App.~\ref{app:data-levels} describes the different problems and prompting
conditions at Levels 2--5.

The examples preserve the prompt text. A continuation line indented by two
spaces indicates a display line break; replacing that break and indentation
with one space recovers the original text. In the \texttt{PantryPlan}
example, a line break without this indentation is part of the prompt.

\filbreak
\subsection{\texttt{Graph} coloring}
\label{app:prompt-graph}

\noindent 384 training and 128 evaluation prompts.

\begin{Verbatim}[fontsize=\scriptsize,frame=single,framesep=4pt,rulecolor=\color{black!25},samepage=true]
<|im_start|>system
Return only the final answer inside \boxed{}. Do not explain.<|im_end|>
<|im_start|>user
Color this graph with vertices 1 through 6. The undirected edges are: (1,2), (1,3),
  (3,5), (3,6), (4,6), (5,6). Assign each vertex one color from {1, 2, 3} so that no
  edge has the same color on both endpoints. A partial coloring is 2??1?2, where ?
  means uncolored. Keep every shown digit fixed. There are exactly 3 question marks.
  Return exactly one final answer and no explanation: exactly 3 digits, each one 1,
  2, or 3, for the missing positions from left to right, inside \boxed{}.<|im_end|>
<|im_start|>assistant
\end{Verbatim}

\filbreak
\subsection{\texttt{Countdown}}
\label{app:prompt-countdown}

\noindent 384 training and 128 evaluation prompts.

\begin{Verbatim}[fontsize=\scriptsize,frame=single,framesep=4pt,rulecolor=\color{black!25},samepage=true]
<|im_start|>system
Return only the final answer inside \boxed{}. Do not explain.<|im_end|>
<|im_start|>user
Using the numbers [3, 6, 9], create an arithmetic expression that equals 18. Use
  each given number exactly once. You may use +, -, *, /, and parentheses. Return
  exactly one final answer and no explanation: the expression inside
  \boxed{}.<|im_end|>
<|im_start|>assistant
\end{Verbatim}

\filbreak
\subsection{\texttt{Python} factors}
\label{app:prompt-python}

\noindent 384 training and 128 evaluation prompts.

\begin{Verbatim}[fontsize=\scriptsize,frame=single,framesep=4pt,rulecolor=\color{black!25},samepage=true]
<|im_start|>system
Return only the final answer inside \boxed{}. Do not explain.<|im_end|>
<|im_start|>user
Write one pure Python function with the exact form lambda n: EXPR. An external
  Python tool will call it once for each n in [18, 82, 91, 93]. For every call,
  return an integer d satisfying 1 < d < n and n % d == 0. Different valid return
  vectors are different solution modes. You may use integer literals, n, +, -, *,
  //, %, comparisons, Boolean operators, and conditional expressions; calls,
  imports, attributes, containers, and other names are forbidden. Return exactly the
  one-line lambda inside \boxed{} and no explanation.<|im_end|>
<|im_start|>assistant
\end{Verbatim}

\filbreak
\subsection{\texttt{MathIR} action menu}
\label{app:prompt-mathir}

\noindent 384 training and 128 evaluation prompts.

\begin{Verbatim}[fontsize=\scriptsize,frame=single,framesep=4pt,rulecolor=\color{black!25},samepage=true]
<|im_start|>system
Return only the final answer inside \boxed{}. Do not explain.<|im_end|>
<|im_start|>user
MathIR linear-menu-v1. Solve for x by choosing an executable action sequence. Every
  selected action is applied exactly to BOTH sides of the current equation, followed
  by exact simplification. Bindings: a=2, b=-9, c=-1. Initial equation: x/a + b = c.
  Action menu: A means div(a); B means add(b); C means sub(b); D means sub(c); E
  means sub(mul(a,b)); F means mul(a). Choose 1-4 action IDs. Illegal operations,
  repeated states, and paths that do not finish with x isolated are rejected. Return
  only the IDs separated by semicolons inside the answer box; for example formatting
  only: \boxed{A;C}. Do not return x's numeric value.<|im_end|>
<|im_start|>assistant
\end{Verbatim}

\subsection{\texttt{PantryPlan}}
\label{app:prompt-pantry}

\noindent 384 training and 128 evaluation prompts.

\begin{Verbatim}[fontsize=\scriptsize,frame=single,framesep=4pt,rulecolor=\color{black!25},samepage=false]
<|im_start|>system
Return only six binary digits and no other text. Each digit maps to the
  corresponding Pantry row in printed order. A trusted environment will choose exact
  quantities on precisely the selected rows.<|im_end|>
<|im_start|>user
Create one feasible high fiber snack from this frozen pantry.
Nutrition values are per 100 g. Quantities are exact integer grams.

Pantry:
- almonds: available=100g, step=25g, minimum-if-used=50g; energy=584kcal,
  protein=21.5g, fiber=10.8g, sodium=0.0mg; tags=tree_nut
- banana: available=125g, step=25g, minimum-if-used=50g; energy=97.0kcal,
  protein=0.74g, fiber=4.62g, sodium=0.0mg; tags=none
- carrots: available=150g, step=25g, minimum-if-used=50g; energy=45.0kcal,
  protein=0.941g, fiber=3.1g, sodium=86.6mg; tags=none
- grape_tomatoes: available=100g, step=25g, minimum-if-used=50g; energy=27.0kcal,
  protein=0.83g, fiber=2.1g, sodium=6.0mg; tags=none
- navel_orange: available=150g, step=25g, minimum-if-used=50g; energy=47.0kcal,
  protein=0.91g, fiber=2.0g, sodium=9.0mg; tags=none
- sunflower_seeds: available=125g, step=25g, minimum-if-used=50g; energy=571kcal,
  protein=18.9g, fiber=7.22g, sodium=0.0mg; tags=none

Use 2 to 4 ingredients.
Forbidden tags: none
Targets:
- mass_g: minimum 125, maximum 200
- energy_kcal: minimum 239.4, maximum 432.25
- protein_g: minimum 6.423
- fiber_g: minimum 3.478
- sodium_mg: maximum 37.6

Choose the ingredient support as a six-bit mask in the exact Pantry row order. Bit 1
  includes that row and bit 0 excludes it.<|im_end|>
<|im_start|>assistant
\end{Verbatim}

\section{Replay Algorithm}
\label{app:algorithm}

Verified replay augments learning from fresh rollouts with a likelihood loss on stored verified solutions.
Each prompt $x$ has a buffer $\mathcal B_x$ containing at most $M$ exemplars, with one exemplar per discovered solution mode.
At each update, a round-robin schedule selects one nonempty buffer and gives each stored solution mode equal weight in the replay loss.
The fresh objective determines whether the resulting method is Re:Dr or Re:Max; the replay mechanism is otherwise identical.

\subsection{Fresh task advantages}
\label{app:binary-maxrl}

For a group of $G$ fresh responses with binary rewards $R_i$ and $R=\sum_i R_i$, Dr.GRPO uses
\[
    A_i^{\mathrm{Dr}}
    =
    R_i-\frac{R}{G},
\]
while MaxRL uses
\[
    A_i^{\mathrm{MaxRL}}
    =
    \begin{cases}
        GR_i/R-1, & R>0,\\
        0, & R=0.
    \end{cases}
\]
In a mixed group, MaxRL assigns advantage $-1$ to failures and gives each successful response greater weight when successes are scarce~\citep{tajwar2026maxrl}.
Both objectives assign zero fresh-task advantage when every response in the group receives the same reward.
Stored exemplars do not receive these fresh advantages and instead contribute through the replay loss defined below.

\subsection{Optimizer update}
\label{app:algorithm-update}

\begin{algorithm}[H]
  \caption{\textbf{Verified replay update.} The fresh objective determines Re:Dr or Re:Max.}
  \label{alg:verified-replay-update}
  \small
  \begin{algorithmic}[1]
    \REQUIRE Policy $\pi_\theta$; prompt $x$ and verifier specification $s_x$; group size $G$; fresh objective $o\in\{\mathrm{Dr.GRPO},\mathrm{MaxRL}\}$; replay weight $\lambda_{\mathrm{replay}}$; buffer capacity $M$; buffers $\mathcal B$; round-robin cursor $\sigma$
    \ENSURE Updated policy, buffers, and cursor
    \STATE Draw the on-policy group $\mathcal G=\{y_i\sim\pi_\theta(\cdot\mid x)\}_{i=1}^G$
    \FOR{$i=1,\ldots,G$}
      \STATE Assign solution mode $m_i\gets V(y_i,s_x)$
      \STATE Set $R_i\gets\mathbf 1\{m_i\neq\bot\}$
    \ENDFOR
    \STATE Compute fresh task advantages $A_i^o$ from $\{R_i\}_{i=1}^G$
    \STATE After validating the full group, retain one exemplar for each newly discovered solution mode $m_i\neq\bot$ in $\mathcal B_x$, subject to capacity $M$
    \IF{at least one replay buffer is nonempty}
      \STATE $(z,\sigma')\gets\textsc{NextBuffer}(\mathcal B,\sigma)$ using the fixed round-robin schedule
      \STATE For each stored exemplar $e_m\in\mathcal B_z$, compute $\ell_m=|e_m|^{-1}\log\pi_\theta(e_m\mid z)$
      \STATE $L_{\mathrm{replay}}\gets-|\mathcal B_z|^{-1}\sum_{m\in\mathcal B_z}\ell_m$
    \ELSE
      \STATE $L_{\mathrm{replay}}\gets0$
      \STATE $\sigma'\gets\sigma$
    \ENDIF
    \STATE $\theta'\gets\textsc{Update}\!\left(\theta,L_{\mathrm{PPO}}(A^o)+\lambda_{\mathrm{replay}}\frac{G-1}{G^2}L_{\mathrm{replay}}\right)$
  \end{algorithmic}
\end{algorithm}

\subsection{Buffer construction and scheduling}
\label{app:algorithm-bank}

Each training prompt has its own replay buffer.
A verified response is eligible for storage once the verifier assigns it a valid solution mode.
The full rollout group is validated before the buffer is updated, and the buffer stores at most one exemplar for each discovered solution mode.
When multiple responses first instantiate the same mode, we retain the lexicographically smallest normalized response-token sequence, making exemplar selection independent of rollout order.
Each buffer contains at most $M=16$ solution modes, with no eviction; modes discovered after the buffer is full receive no replay support.

At each update, a deterministic round-robin schedule selects one nonempty buffer.
The policy scores every stored exemplar in that buffer by teacher forcing on its original prompt, and each exemplar contributes its mean response-token log probability excluding prompt and padding tokens.
The replay loss is the negative mean of these exemplar scores, weighted by $\lambda_{\mathrm{replay}}(G-1)/G^2$ (Algorithm~\ref{alg:verified-replay-update}).
The factor $(G-1)/G^2$ matches the contribution of a single successful response in a centered Dr.GRPO group of size $G$.

Replay therefore uses only solution modes discovered during training and does not require the full set of valid solutions or any evaluation outputs.

\subsection{Matched control}
\label{app:algorithm-control}

Replay and control use the same validation, buffer construction, scheduling, exemplar scoring, optimizer updates, and fresh-rollout budget.
The control executes the same replay machinery but sets the replay gradient to zero, so the replay learning signal is the only difference in the objective.

As training progresses, the policies can generate different responses and therefore accumulate different buffers and replay workloads.
App.~\ref{app:compute-matching} quantifies this overhead and compares replay against controls with larger fresh-rollout budgets.

\subsection{Verifier requirements and applicability}
\label{app:algorithm-identity}

Replay requires a deterministic rule for assigning each verified response to a solution mode.
This identity must be computable during training at reasonable cost and must reflect the task's intended notion of distinct solutions.
In \mb{}, the executable verifier determines both correctness and solution-mode identity, so buffer construction requires neither a learned judge nor an embedding-similarity threshold.

Verified replay applies whenever training-time verification can assign correct responses to task-relevant solution modes.
Final-answer-only graders collapse distinct derivations of the same answer into one mode, so preserving derivational alternatives requires a finer identity rule.
Consistent with this distinction, the single-slot replay control retains part, but not all, of Re:Dr's diversity benefit (App.~\ref{app:mode-agnostic-replay}).

\texttt{MathIR} demonstrates a path-based identity by defining modes through executed state trajectories.
More open-ended tasks could analogously use a validated learned judge or another structured equivalence rule.
Executable identities in \mb{} provide a deterministic reference setting for studying this mechanism before extending it to learned mode definitions.

\subsection{Exemplar quality and length}
\label{app:algorithm-quality}

Verification establishes that each retained exemplar satisfies the task constraints and assigns it to a solution mode.
Replay then gives discovered modes equal buffer representation without introducing a separate learned ranking over exemplar quality.

Each solution mode occupies one buffer entry, and its exemplar contributes its mean response-token log probability,
\[
    \ell_m
    =
    |e_m|^{-1}\log\pi_\theta(e_m\mid z).
\]
Thus longer exemplars do not receive additional buffer entries or loss weight proportional to their token count.
Mean-token normalization gives each stored exemplar equal weight in the replay loss, but it does not imply that the trained policy assigns equal probability to the corresponding solution modes.

We apply no additional response-length or quality filters beyond the generation limits and executable verifier, keeping the replay target aligned with the benchmark's verified notion of a valid alternative.

\subsection{Semantic-MaxEnt implementation}
\label{app:semantic-maxent}

Fixed Semantic-MaxEnt adds a frequency-based bonus to fresh verified responses.
For each prompt, it tracks solution modes observed in previous successful rollouts and among the other successful responses in the current group.
Let $n_x(a)$ count previous successful responses with solution mode $a$, and let $m_{-i}(a)$ count successful responses with that mode among the other members of the current group.
Define
\[
    N_x=\sum_a n_x(a),
    \qquad
    S_{-i}=\sum_a m_{-i}(a),
\]
and let $\mathcal K_i$ contain the solution modes observed in either source.
We estimate
\begin{equation}
  \hat p_i(a)
  =
  \frac{n_x(a)+m_{-i}(a)+1}
       {N_x+S_{-i}+|\mathcal K_i|+1},
  \qquad
  \hat p_i(\textsc{unseen})
  =
  \frac{1}
       {N_x+S_{-i}+|\mathcal K_i|+1}.
  \label{eq:open-set-predictor}
\end{equation}
A newly observed solution mode receives the probability assigned to the \textsc{unseen} category.

For a successful response with solution mode $a_i$, the added advantage is
\begin{equation}
  A^{\rm sem}_{i}
  =
  \eta
  \frac{
    \min\{-\log\hat p_i(a_i),\kappa\}
    -
    \E_{\hat p_i}
    [\min\{-\log\hat p_i(a),\kappa\}]
  }{\kappa},
  \qquad
  \eta=.10,\quad \kappa=5.
  \label{eq:legacy-semantic-advantage}
\end{equation}
The same bonus is applied at every token of the response, while unsuccessful responses receive no bonus.
Counts from the current group are added only after scoring, so a response does not affect its own bonus.

This mechanism favors underrepresented verified solution modes only when they appear in fresh rollouts.
Unlike replay, it does not optimize the likelihood of stored exemplars whose solution modes are absent from the current group.
Because $\hat p_i$ is a frequency-based predictor rather than the model's current conditional distribution over correct solution modes, Semantic-MaxEnt serves as a fresh-rollout frequency baseline rather than the same objective as stored-mode replay.

\clearpage
\phantomsection
\begin{center}\large\textbf{Part III. Training results}\end{center}
\addcontentsline{toc}{part}{Part III. Training results}

\section{Replay Effects Across Scales}
\label{app:supporting}

\begin{figure}[!htbp]
  \centering
  \includegraphics[width=\linewidth]{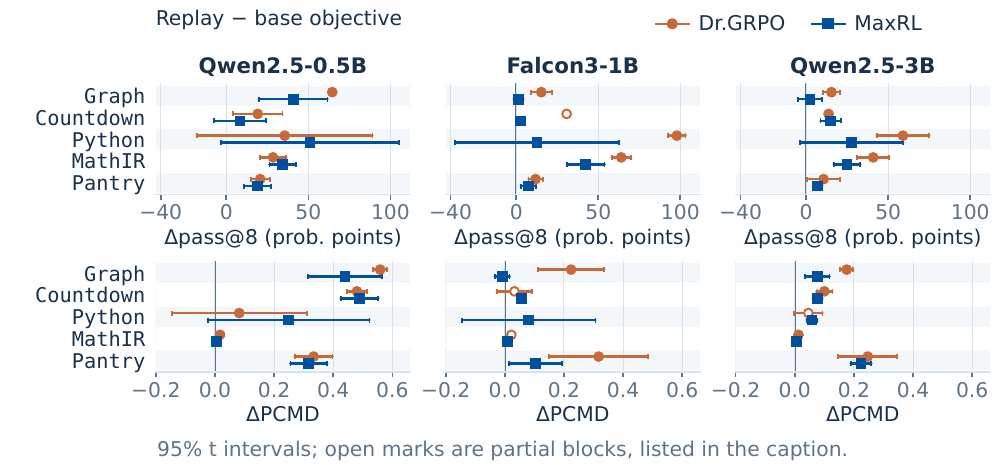}
  \caption{\textbf{Replay gains persist across model scales.}
  Points show terminal Re:Dr minus Dr.GRPO and Re:Max minus MaxRL effects on \texttt{pass@8} and \pmd{} across three model scales and five domains.
  Filled markers use five paired seeds, hollow markers use fewer, and bars show nominal 95\% paired intervals where available.}
  \label{fig:replay-factorial-effects}
\end{figure}

\subsection{Terminal comparisons across scales}
\label{app:scale-matrix}
\label{app:scale-snapshot}
\label{app:maxrl-scale-extensions}

Replay and control conditions share initialization, prompts, optimizer settings, training updates, fresh-rollout budget, verifier, and evaluation requests.
The control executes the same replay machinery with zero replay gradient, although total compute is not matched because replay processes stored exemplars.

Across model scales, replay improves both correctness and conditional solution diversity on average.
Table~\ref{tab:cross-scale-terminal-effects} summarizes the paired cross-domain effects at the terminal checkpoint.

\begin{table}[!b]
  \centering
  \caption{\textbf{Replay improves mean correctness and conditional diversity across model scales.}
  Entries are replay minus the matched non-replay objective after eight training passes.
  \texttt{pass@8} averages all five domains within matched seeds.
  \pmd{} averages a fixed set of domains with defined paired estimates for all contributing seeds.
  Falcon3-1B Re:Dr \texttt{pass@8} uses four common seeds; every other entry uses five.}
  \label{tab:cross-scale-terminal-effects}
  \small
  \setlength{\tabcolsep}{7pt}
  \begin{tabular}{@{}llrr@{}}
    \toprule
    Model & Comparison & $\Delta$\texttt{pass@8} & $\Delta$\pmd{} \\
    \midrule
    Qwen2.5-0.5B & Re:Dr $-$ Dr.GRPO & $+.337$ & $+.294$ \\
    Falcon3-1B & Re:Dr $-$ Dr.GRPO & $+.442$ & $+.271$ \\
    Qwen2.5-3B & Re:Dr $-$ Dr.GRPO & $+.279$ & $+.134$ \\
    Qwen2.5-0.5B & Re:Max $-$ MaxRL & $+.307$ & $+.300$ \\
    Falcon3-1B & Re:Max $-$ MaxRL & $+.133$ & $+.049$ \\
    Qwen2.5-3B & Re:Max $-$ MaxRL & $+.154$ & $+.088$ \\
    \bottomrule
  \end{tabular}
\end{table}

\paragraph{Domain-level effects.}
\label{app:retention-numerical-detail}
Re:Dr increases mean \pmd{} over Dr.GRPO in all 14 measurable model--domain comparisons, while Re:Max exceeds MaxRL in 14 of 15.
The domain-level pattern is therefore broad across models and objectives.
Effect sizes vary, with Falcon3-1B \texttt{Graph} providing the lone near-zero Re:Max diversity comparison and Qwen2.5-3B showing positive Re:Max effects in all five domains.

\begin{figure}[!htbp]
  \centering
  \includegraphics[width=.98\linewidth]{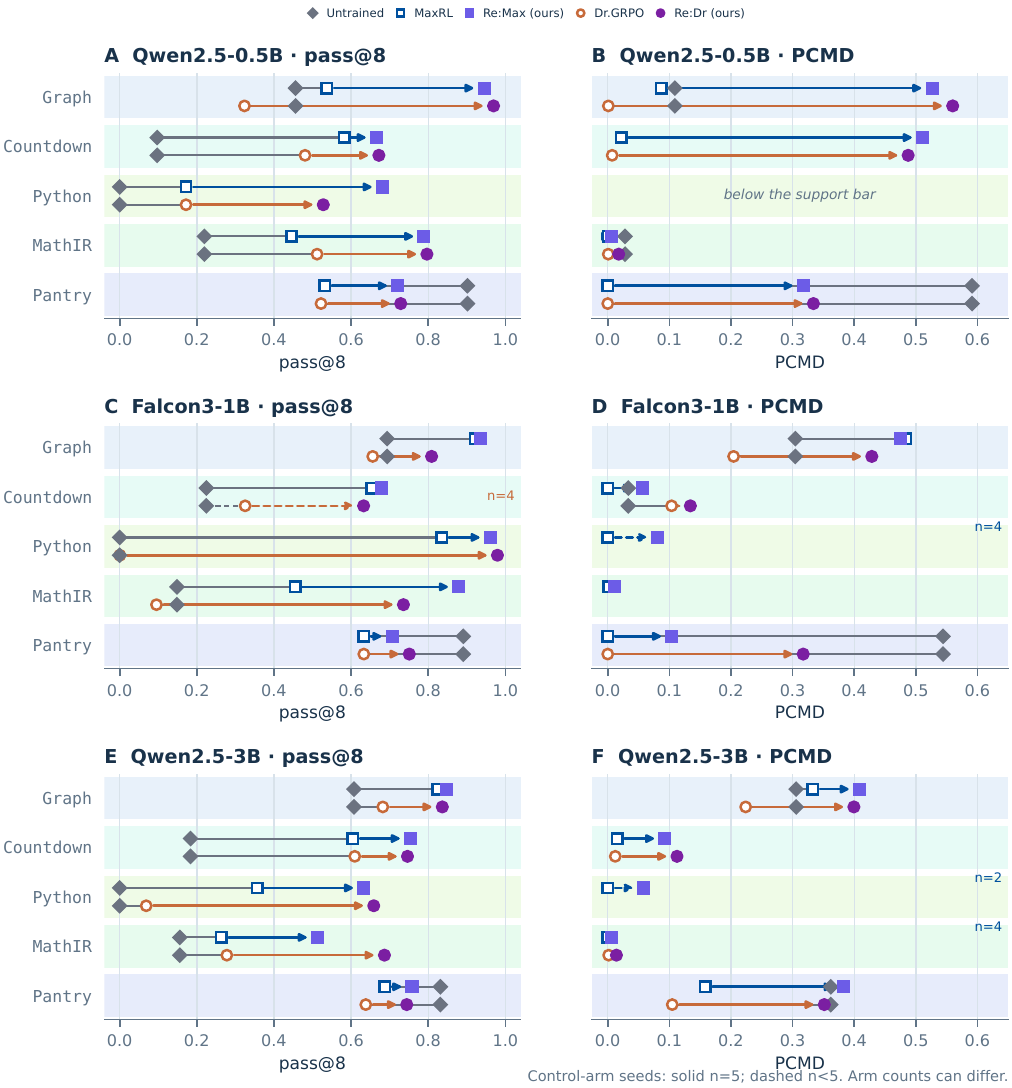}
  \caption{\textbf{Replay gains vary by domain and model scale.}
  Final \texttt{pass@8} and \pmd{} across three model scales for Dr.GRPO, MaxRL, and their replay variants.
  Markers show seed means; estimates with insufficient verified responses are omitted.}
  \label{fig:maxrl-factorial-scale-extensions}
\end{figure}

Each paired \pmd{} contrast uses prompts with sufficient verified responses in both arms, so eligibility is determined within the corresponding model--domain comparison.
The reported effects therefore compare matched conditional solution distributions within each evaluable cell.

\paragraph{Cross-domain summary.}
\label{app:factorial-numerical-detail}
Positive mean replay effects appear under both fresh objectives at all three model scales.
The domain-level results in Fig.~\ref{fig:maxrl-factorial-scale-extensions} show that the magnitude of these gains depends on the task and fresh objective.

\subsection{Level 1 at Qwen2.5-7B}
\label{app:qwen7b-level1}

We evaluate MaxRL and Re:Max on Qwen2.5-7B for \texttt{Countdown}, \texttt{Graph}, \texttt{MathIR}, and \texttt{PantryPlan} using the same training horizon, prompts, verifier, evaluation procedure, replay schedule, and training seed (70) within each matched pair.
Each policy trains for 3,072 updates and is evaluated on 128 held-out prompts.

\begin{table}[!htbp]
  \centering
  \caption{\textbf{The Qwen2.5-7B scale extension separates three large replay effects from one near-null comparison.}
  Values are terminal measurements after 3,072 updates from one matched seed per domain.
  Parentheses give the number of evaluation prompts for which \pmd{} is defined; an undefined value indicates that too few verified responses are available to estimate conditional diversity.}
  \label{tab:qwen7b-level1}
  \small
  \setlength{\tabcolsep}{6pt}
  \begin{tabular}{@{}lrrrr@{}}
    \toprule
    & \multicolumn{2}{c}{\texttt{pass@8}} & \multicolumn{2}{c}{\pmd{}} \\
    \cmidrule(lr){2-3}\cmidrule(l){4-5}
    Domain & MaxRL & Re:Max & MaxRL & Re:Max \\
    \midrule
    \texttt{Countdown}  & .631 & .805 & .001 (81) & .138 (104) \\
    \texttt{Graph}      & .396 & .818 & .055 (51) & .353 (104) \\
    \texttt{MathIR}     & .000 & .744 & -- (0)    & .017 (96) \\
    \texttt{PantryPlan} & .760 & .779 & .384 (97) & .348 (101) \\
    \bottomrule
  \end{tabular}
\end{table}

Across the four domains, mean terminal \texttt{pass@8} increases from $.447$ for MaxRL to $.787$ for Re:Max, but the aggregate masks substantial domain heterogeneity.
The terminal correctness gains are $+.174$ on \texttt{Countdown}, $+.422$ on \texttt{Graph}, and $+.744$ on \texttt{MathIR}, compared with $+.019$ on \texttt{PantryPlan}.
Among the three domains with measurable conditional diversity in both arms, the corresponding \pmd{} changes are $+.137$, $+.298$, and $-.036$.

\paragraph{\texttt{Countdown}.}
Both methods begin at \texttt{pass@8} $.47$.
MaxRL rises to approximately $.62$ by update 960 and then remains near that level, ending at $.631$ with \pmd{} $.001$.
Re:Max continues improving after the control plateaus and ends at \texttt{pass@8} $.805$ with \pmd{} $.138$.
Thus replay changes both the correctness trajectory and the concentration of verified solutions in this domain.

\paragraph{\texttt{Graph}.}
The first evaluation occurs at update 192, where both arms have \texttt{pass@8} $.68$; both reach approximately $.70$ by update 576.
MaxRL then falls to $.39$ by update 960 and remains near $.40$ through the terminal checkpoint.
Re:Max instead remains above $.80$ over the later evaluations and ends at $.818$.
At the terminal checkpoint, \pmd{} is $.055$ for MaxRL and $.353$ for Re:Max, with 51 and 104 eligible prompts, respectively.

\paragraph{\texttt{MathIR}.}
MaxRL improves initially, rising from the untrained model's \texttt{pass@8} of $.227$ to a peak of $.365$ at update 576.
It then falls below the untrained model by update 768, reaching $.217$, and is zero at update 960 and at every subsequent evaluation.
At the final checkpoint, none of its 4{,}096 sampled responses is verified across 128 prompts and four groups of eight responses per prompt, so its \pmd{} is undefined.
Re:Max, trained on the same prompts and seed, shows no such failure: its \texttt{pass@8} rises at every evaluation through update 1344 and thereafter remains within $.01$ of its running maximum, ending at $.744$, more than three times the untrained model.
Its terminal \pmd{} of $.017$ is near the floor, so the replay effect on \texttt{MathIR} is a preservation of correctness rather than breadth.

\paragraph{\texttt{PantryPlan}.}
The untrained model already reaches \texttt{pass@8} $.88$.
Both methods improve slightly at first and then decline, ending at $.760$ for MaxRL and $.779$ for Re:Max.
Conditional diversity remains comparatively high in both arms, with terminal \pmd{} of $.384$ and $.348$, respectively.
The two methods therefore finish in approximately the same correctness--diversity regime on \texttt{PantryPlan}, unlike the large separations in the other three domains.

The Qwen2.5-7B extension therefore shows that replay is not mechanically beneficial in every domain: it changes the outcome when the matched MaxRL policy narrows or loses correctness, while producing little terminal separation in the high-initial-correctness \texttt{PantryPlan} setting.
Because each domain uses one matched seed, these runs establish a scale extension and domain-level contrast rather than a model-size trend.

\subsection{Results on harder benchmark levels}
\label{app:level2-interim}

We evaluate the four primary methods on Level~2 using Qwen2.5-0.5B.
Correctness comparisons use five matched seeds, while \pmd{} is reported only when both arms contain enough prompts with multiple verified responses.

\begin{figure}[!htbp]
  \centering
  \includegraphics[width=\linewidth]{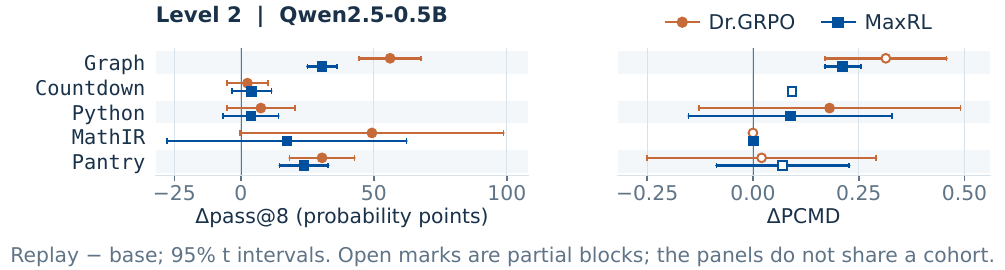}
  \caption{\textbf{Replay improves Level-2 \texttt{Graph} accuracy and conditional solution diversity.}
  Points show terminal Re:Dr minus Dr.GRPO and Re:Max minus MaxRL effects on \texttt{pass@8} and \pmd{}.
  Bars show nominal 95\% paired intervals where at least two eligible seed pairs are available.}
  \label{fig:replay-level2-effects}
\end{figure}

\paragraph{Level-2 \texttt{Graph}.}
\label{app:levels-numerical-detail}
Replay improves both correctness and conditional solution diversity substantially on \texttt{Graph}.
After eight training passes, Dr.GRPO and MaxRL reach \texttt{pass@8} of $.202$ and $.434$, compared with $.763$ and $.739$ for Re:Dr and Re:Max.
Among the eligible paired seeds, \pmd{} increases from $.039$ to $.353$ for Re:Dr versus Dr.GRPO and from $.112$ to $.325$ for Re:Max versus MaxRL.
The corresponding paired \pmd{} gains are $+.314$ and $+.212$.

Beyond \texttt{Graph}, replay also improves \texttt{PantryPlan} correctness under both objectives and yields positive \texttt{Python} diversity point estimates, while \texttt{MathIR} remains concentrated under all four methods.
Level~2 therefore shows that replay gains persist on harder constructions, with effect size varying by domain.
Because the Level-2 task populations and some prompt guidance differ from Level~1, cross-level contrasts summarize benchmark regimes rather than isolate difficulty alone.

\subsection{Diversity throughout training}
\label{app:factorial-training-curves}

We examine \pmd{} throughout training to distinguish terminal differences from transient variation.
Conditional-diversity estimates are reported only when enough held-out prompts contain at least two verified responses, so the contributing prompts and seeds can vary across checkpoints.

\begin{figure}[!htbp]
  \centering
  \includegraphics[width=\linewidth]{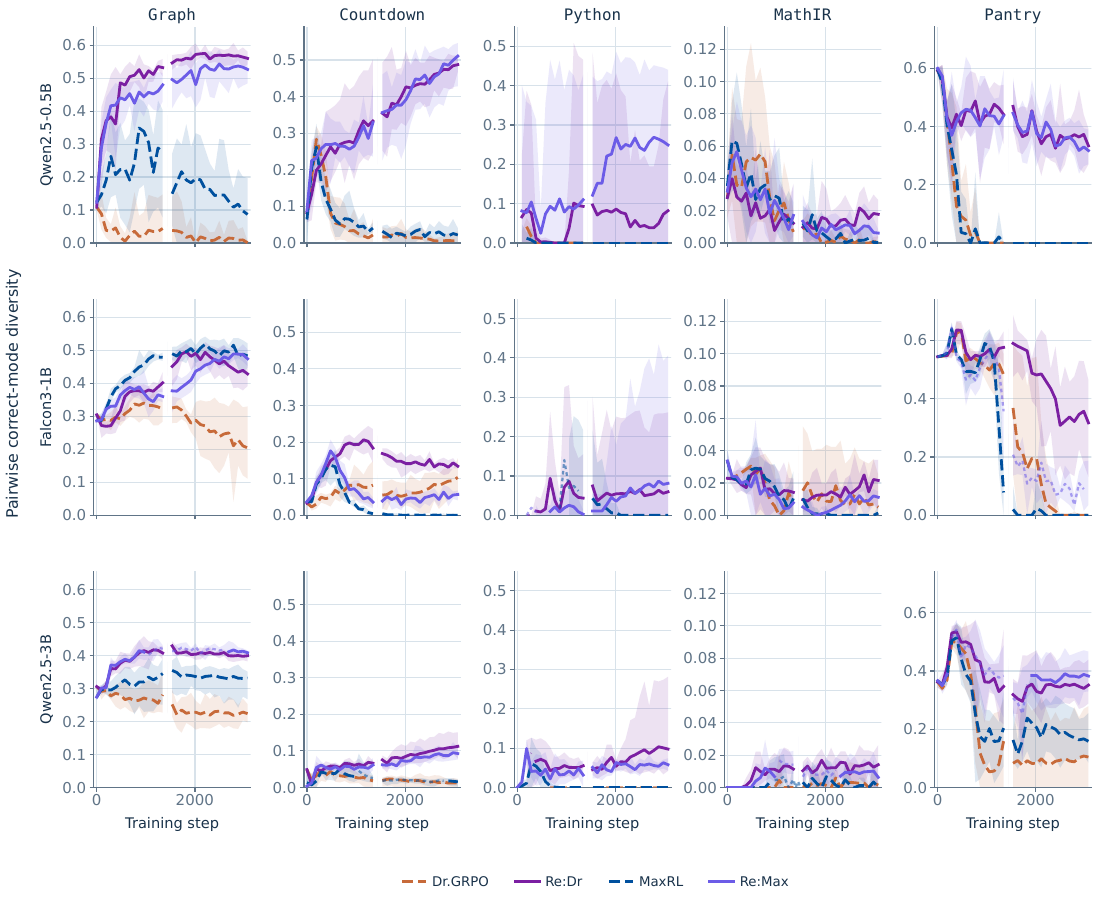}
  \caption{\textbf{Replay preserves conditional solution diversity throughout training.}
  \pmd{} across eight training passes for Dr.GRPO, Re:Dr, MaxRL, and Re:Max at three model scales.
  Lines show means over eligible seeds and bands show seed ranges; gaps indicate insufficient verified responses.}
  \label{fig:factorial-training-modes}
\end{figure}

Across scales, the largest separation appears in \texttt{PantryPlan}, where the no-replay objectives approach zero terminal diversity at the smaller scales while replay retains multiple correct solution modes.
The trajectories show that the terminal replay advantage reflects sustained differences during training rather than only a final-checkpoint fluctuation.

\begin{figure}[!htbp]
  \centering
  \includegraphics[width=\linewidth]{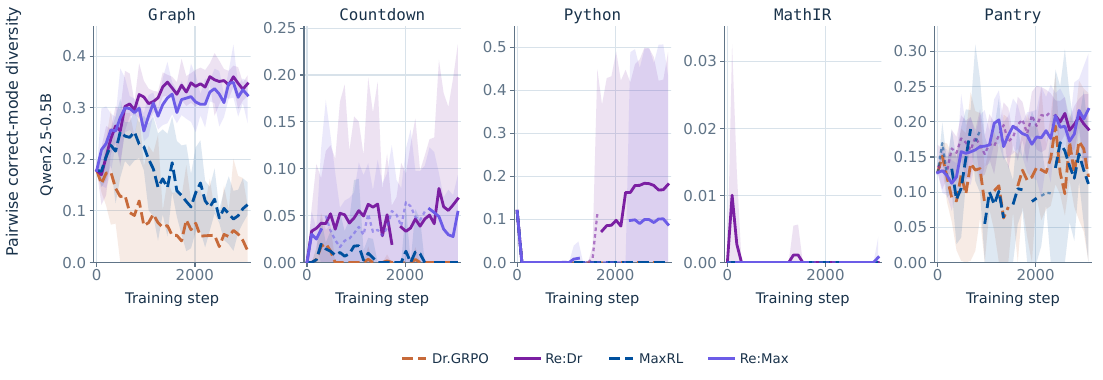}
  \caption{\textbf{Replay also preserves diversity on harder Level-2 tasks.}
  Qwen2.5-0.5B \pmd{} across training for the four primary methods.
  Replay produces the clearest separation on \texttt{Graph}, while \texttt{MathIR} remains concentrated under all four methods.}
  \label{fig:level2-training-pmd}
\end{figure}

\subsection{Supporting comparators}
\label{app:ucpo-progress}

GRPO changes the within-group reward normalization, UCPO redistributes fresh on-policy advantage, and sparse RLEP-Dr replays verified successes without balancing them by solution mode.
We evaluate GRPO across three model scales, while UCPO and RLEP-Dr cover the two smaller scales.
Fixed Semantic-MaxEnt adds a frequency-based bonus to fresh verified responses (App.~\ref{app:semantic-maxent}).

For direct comparator \pmd{} contrasts, a prompt is eligible only when both methods produce at least two verified responses, and each seed must contain at least 30 such prompts.
Some comparator cells therefore have fewer reportable seeds, and no \texttt{Python} comparison meets this threshold at every scale.

\paragraph{RLEP-Dr pool composition.}
\phantomsection
\label{app:rlep-pool}
RLEP-Dr replays verified trajectories from a frozen experience pool without balancing solution modes.
Among eligible prompts, its pools contain only $1.00$--$1.02$ solution modes on average for Qwen2.5-0.5B and $1.00$--$1.40$ for Falcon3-1B.
Two-row replay draws therefore repeat the same solution mode with probability $.89$--$1.00$ across the evaluated domains.
By contrast, the corresponding Re:Dr buffers at Qwen2.5-0.5B average $1.15$--$6.88$ stored solution modes across domains.
RLEP-Dr therefore provides a useful replay-of-success baseline whose experience pools contain little multi-mode coverage.

\paragraph{RLEP-Dr with an online experience pool.}
\phantomsection
\label{app:rlep-online-pool}
To separate experience-pool coverage from the RLEP-Dr update rule, we evaluate an online-pool variant on Qwen2.5-0.5B Level~1.
The update remains unchanged: each batch contains 16 fresh responses and two replay responses, replay uses a single baseline, draws preserve the empirical frequency of stored successes, and a prompt becomes eligible for replay after two verified successes.
The only intervention is the source of replay experience: the pool begins empty and accumulates the learner's own verified rollouts during training.

\begin{table}[!htbp]
  \centering
  \caption{\textbf{An online experience pool recovers part, but not all, of Re:Dr's conditional-diversity gain.}
  Results use five paired seeds per domain against the matched Dr.GRPO control.
  The effect column is online-pool RLEP-Dr minus control, and replay share is the fraction of updates receiving replay from pass 2 onward.}
  \label{tab:rlep-online-pool}
  \small
  \setlength{\tabcolsep}{5pt}
  \begin{tabular}{@{}lrrrrr@{}}
    \toprule
    Domain & Control \pmd{} & Online RLEP \pmd{} & Re:Dr \pmd{} & $\Delta$\pmd{} & Replay share \\
    \midrule
    \texttt{Graph}      & .003 & .344 & .711 & +.340 & .71 \\
    \texttt{Countdown}  & .017 & .054 & .538 & +.037 & .66 \\
    \texttt{Python}     & .000 & .000 & .111 &  .000 & .20 \\
    \texttt{MathIR}     & .000 & .013 & .018 & +.013 & .58 \\
    \texttt{PantryPlan} & .000 & .254 & .484 & +.254 & .93 \\
    \midrule
    Mean                 & .004 & .133 & .372 & +.129 & .62 \\
    \bottomrule
  \end{tabular}
\end{table}

Across the 25 paired domain--seed comparisons, the online pool increases \pmd{} over the control by $+.129$ ($[+.068,+.189]$), with positive effects in 17 of 25 pairs.
Mean \texttt{pass@8} increases from $.444$ for the control to $.583$, a paired effect of $+.139$ ($[+.035,+.243]$).
The online-pool variant therefore recovers approximately $35\%$ of Re:Dr's aggregate $+.368$ conditional-diversity improvement over the same control.

The gain is concentrated in \texttt{Graph} and \texttt{PantryPlan}, with smaller changes in \texttt{Countdown} and \texttt{MathIR} and no measurable conditional-diversity gain in \texttt{Python}.
Replay is active on only $.20$ of \texttt{Python} updates from pass 2 onward because relatively few prompts produce the two verified successes required for eligibility.
More generally, the RLEP replay coefficient becomes small as prompts are solved, ranging from approximately $.01$ to $.48$ in these runs, whereas Re:Dr retains a fixed replay weight.
Broadening the experience pool therefore explains part of the original RLEP-Dr gap, but most of the aggregate difference to Re:Dr remains when RLEP-Dr keeps its frequency-preserving replay rule and adaptive replay weighting.

\paragraph{Comparator settings.}
\phantomsection
\label{app:comparator-settings}
Comparator hyperparameters follow the stated method configurations: UCPO uses $\omega=.2$, RLEP-Dr mixes 16 fresh with 2 replay responses, SetPO uses $\lambda=1$, Semantic-MaxEnt uses $\eta=.10$ and $\kappa=5$, and our replay methods use $\lambda_{\mathrm{replay}}=.10$ with capacity $M=16$.
GAPO has no free coefficient and its reward is affinely mapped to the control reward range.
We do not perform a benchmark-specific hyperparameter search for these comparators, so the comparison evaluates the stated methods under a common experimental budget.

\paragraph{Interpreting zero diversity.}
An undefined \pmd{} comparison indicates insufficient verified responses, whereas $.000$ denotes measured zero conditional diversity.
In \texttt{Python}, several methods achieve high \texttt{pass@8} while producing only one verified solution mode per eligible prompt, illustrating that correctness can remain high even when conditional solution diversity is zero.

\begin{figure}[!htbp]
  \centering
  \includegraphics[width=\linewidth]{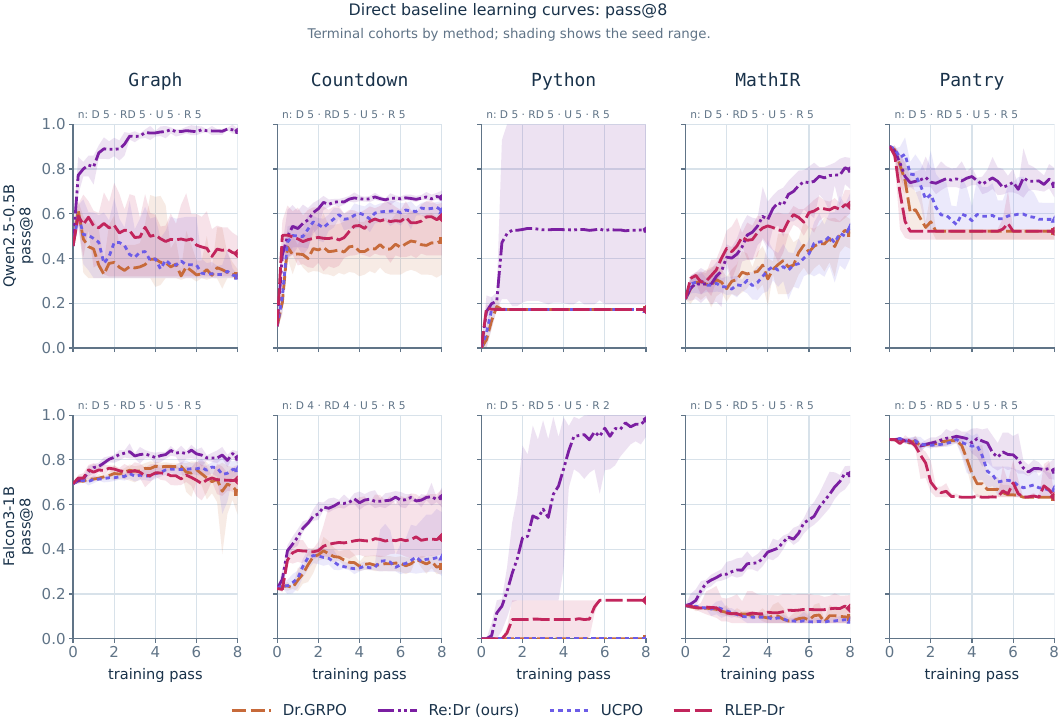}
  \caption{\textbf{Re:Dr has the highest terminal mean \texttt{pass@8} in every displayed comparison.}
  Learning curves compare Dr.GRPO, Re:Dr, UCPO, and sparse RLEP-Dr across two model scales and five domains.
  Lines show seed means and bands show seed ranges.}
  \label{fig:direct-baseline-accuracy-curves}
\end{figure}

\subsection{Diversity-preserving comparators}
\label{app:diversity-comparators}

We compare against GAPO~\citep{anschel2025groupaware} and SetPO~\citep{li2026setpo}, which explicitly encourage diversity within the current rollout group.
Both are evaluated on Level-1 Qwen2.5-0.5B across five domains and five seeds using a matched Dr.GRPO control.
GAPO reweights verified rollouts according to within-group solution-mode frequency, while SetPO rewards each rollout for its marginal contribution to semantic set diversity.

GAPO increases \pmd{} substantially on \texttt{Countdown} ($+.340$) and \texttt{PantryPlan} ($+.653$), while SetPO's clearest diversity gain is on \texttt{PantryPlan} ($+.266$).
Against the shared Dr.GRPO control, the five-domain mean \pmd{} effects are $+.225$ for GAPO, $+.066$ for SetPO, and $+.320$ for Re:Max.
Re:Max has the largest five-domain mean effect, while GAPO is strongest on \texttt{PantryPlan}.
SetPO also substantially improves \texttt{Python} correctness, raising \texttt{pass@8} from $.338$ to $.834$.

These methods encourage diversity using responses available in the current rollout group, whereas replay balances verified solution modes accumulated across training.
This distinction is especially relevant when the number of valid solution modes greatly exceeds the rollout-group size.

\subsection{Reference KL tradeoffs}
\label{app:kl-measured}
\label{app:kl-case}

We compare replay with Dr.GRPO using a reference-policy KL penalty at $\beta\in\{.001,.01,.04,.1,.2,.3\}$ on Qwen2.5-0.5B Level~1.
The reference is the initial policy, and all conditions use the same number of optimizer updates and fresh rollouts.
KL additionally evaluates the reference policy, while replay scores stored exemplars.

Reference KL can preserve substantial conditional diversity, but its correctness tradeoff depends strongly on the domain and coefficient.
In \texttt{PantryPlan} at $\beta=.3$, KL reaches $\pmd{}=.931$ and \texttt{pass@8} $=.825$, compared with $.557$ and $.797$ for Re:Dr.
The initial \texttt{PantryPlan} policy has $\pmd{}=.91$, consistent with the reference providing a strong diversity target.

The tradeoff is different in \texttt{MathIR}.
At $\beta=.3$, KL reaches $\pmd{}=.099$ but reduces \texttt{pass@8} to $.207$, compared with $.039$ and $.816$ for Re:Dr.
In \texttt{Graph}, the tested KL coefficients remain below the replay conditions on the reported correctness--diversity tradeoff.

Reference KL anchors the policy to a fixed initial reference, whereas replay rehearses verified solution modes discovered during training.
The comparison reveals different tradeoffs: KL can preserve high diversity when the reference itself is broad, while replay achieves a stronger correctness--diversity balance in \texttt{Graph} and \texttt{MathIR} under the tested settings.

\subsection{Uniform versus fresh-frequency replay}
\label{app:frequency-replay-progress}

This ablation keeps the replay mechanism fixed and changes only how stored solution modes are weighted.
Uniform replay assigns equal weight to retained solution modes, whereas frequency-weighted replay assigns weight in proportion to how often each mode appears in verifier-positive fresh rollouts.
The buffer construction and update rule are otherwise unchanged, although the resulting policies can accumulate different buffers during training.

We report $B_8=\texttt{distinct@8}-\texttt{pass@8}$, the number of extra verified solution modes per prompt beyond the first success, and $P_8=\texttt{pass@8}$.
Positive contrasts favor uniform weighting.
Paired-seed intervals use bootstrap resampling of the matched training seeds.

At Qwen2.5-0.5B, uniform weighting increases extra verified solution modes by $+.318$ per prompt over the five-domain aggregate ($[+.272,+.358]$), with positive within-seed averages in all five seeds.
The corresponding \texttt{pass@8} effect is $+.010$ ($[-.119,+.139]$), indicating that the diversity gain does not come with a detectable correctness cost.
The diversity effect is driven primarily by \texttt{Graph}, \texttt{Countdown}, and \texttt{PantryPlan}.

The same direction appears in several larger-model comparisons.
Uniform weighting increases $B_8$ on Falcon3-1B \texttt{Graph} by $+.233$ ($[+.182,+.320]$), on Qwen2.5-3B \texttt{Graph} by $+.079$ ($[+.011,+.135]$), and on Qwen2.5-3B \texttt{PantryPlan} by $+.207$ ($[+.038,+.377]$).
The corresponding correctness effects are not consistently distinguishable from zero.

This comparison isolates the role of balancing replay across stored solution modes.
Within the same replay procedure, equal weighting preserves more verified alternatives than weighting modes according to their fresh-rollout frequency.

\begin{figure}[!htbp]
  \centering
  \includegraphics[width=\linewidth]{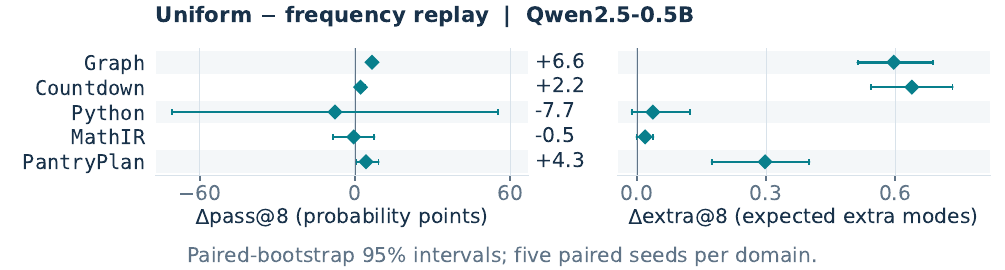}
  \caption{\textbf{Uniform weighting preserves more verified alternatives.}
  Uniform minus frequency-weighted replay at Qwen2.5-0.5B, Level~1.
  Extra modes are $B_8=\texttt{distinct@8}-\texttt{pass@8}$.
  Bars show nominal 95\% paired-seed bootstrap intervals.}
  \label{fig:replay-key-weighting}
\end{figure}

\section{Concentration and Replay Mechanisms}
\label{app:ablations}

\subsection{Fresh 64-response evaluation}
\label{app:fresh-concentration}

We reevaluate initial and terminal checkpoints on the 128 held-out Level-1 \texttt{Graph} and \texttt{PantryPlan} prompts using 64 newly generated responses per prompt.
The evaluation covers Qwen2.5-0.5B, Falcon3-1B, and Qwen2.5-3B, with five training seeds for Dr.GRPO, Re:Dr, MaxRL, and Re:Max.
For each contrast, \pmd{} is computed on prompts with at least two verified responses in both conditions, and the five paired seed effects are averaged equally.

Dr.GRPO reduces \pmd{} in all six model--domain comparisons and all 30 seed effects.
Replay increases \pmd{} relative to its matched trained control under both fresh objectives in all twelve model--domain comparisons and all 60 seed effects.
Thus the replay diversity advantage persists under a substantially larger fresh evaluation sample and jointly eligible prompt populations.

\subsection{Conditional-concentration robustness}
\label{app:conditional-concentration}

As an additional sensitivity analysis, we recompute terminal replay--control differences in conditional concentration using two opposite assignments of disjoint sampling streams.
Negative $\Delta\widehat C$ means that replay is less concentrated among correct outputs than its matched control.
Eligibility is recomputed separately under each stream assignment.

\begin{table}[!b]
  \centering
  \caption{\textbf{Replay generally lowers conditional concentration under disjoint sampling splits.}
  Counts summarize Level-1 replay--control comparisons across three model scales, five domains, and two objectives.
  One Falcon3-1B \texttt{Python} comparison is undefined because too few verified responses are available.}
  \label{tab:conditional-concentration-robustness}
  \small
  \setlength{\tabcolsep}{7pt}
  \begin{tabular}{@{}lrrr@{}}
    \toprule
    Estimate & Replay lower & No change & Replay higher \\
    \midrule
    Pooled estimate & 27/29 & 1/29 & 1/29 \\
    Disjoint split 0 & 25/29 & 3/29 & 1/29 \\
    Disjoint split 1 & 24/29 & 3/29 & 2/29 \\
    \bottomrule
  \end{tabular}
\end{table}

The same qualitative pattern appears under both disjoint stream assignments.
At Level~2, seven of ten pooled replay--control comparisons favor replay, one is unchanged, and two \texttt{PantryPlan} comparisons favor the control; several cells have small eligible populations.
Together, these disjoint-sampling analyses reinforce the main \pmd{} result that replay generally reduces concentration among correct outputs.

\subsection{Diversity before and after RLVR-only training}
\label{app:pre-rl-regression}

We compare Dr.GRPO and GRPO before training and after eight passes across three model scales and five domains.
We decompose \texttt{distinct@8} into \texttt{pass@8} and the number of additional verified modes beyond the first success.

Across both objectives, the average number of additional verified modes decreases at every tested model scale.
The decline is approximately $98$--$99\%$ at Qwen2.5-0.5B and $42$--$72\%$ at the larger scales.
This loss of breadth can occur even when success improves: \texttt{pass@8} does not consistently decrease, and aggregate \texttt{distinct@8} can increase when gains in correctness outweigh losses in additional modes.

Where conditional diversity is measurable at both endpoints, \pmd{} shows the same general pattern.
Its domain-average decline is approximately $99\%$ at Qwen2.5-0.5B and $51$--$65\%$ at the larger scales, although individual domain--model combinations need not decrease.
Thus the reduction in observed solution breadth is not explained solely by changes in the probability of producing a correct response.

\begin{figure}[!htbp]
  \centering
  \begin{minipage}[t]{.49\linewidth}
    \centering
    \textbf{(A)}\\[1pt]
    \includegraphics[width=\linewidth]{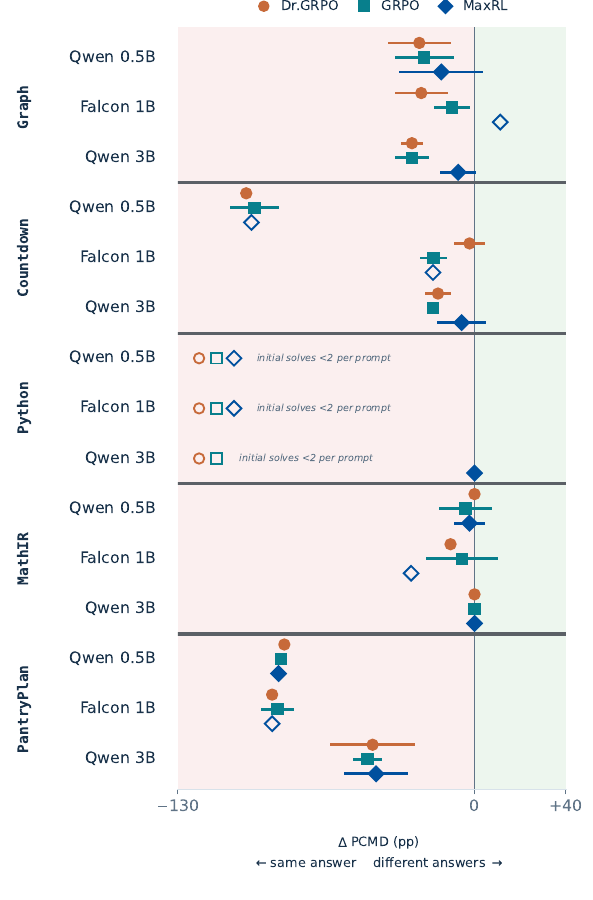}
  \end{minipage}\hfill
  \begin{minipage}[t]{.49\linewidth}
    \centering
    \textbf{(B)}\\[1pt]
    \includegraphics[width=\linewidth]{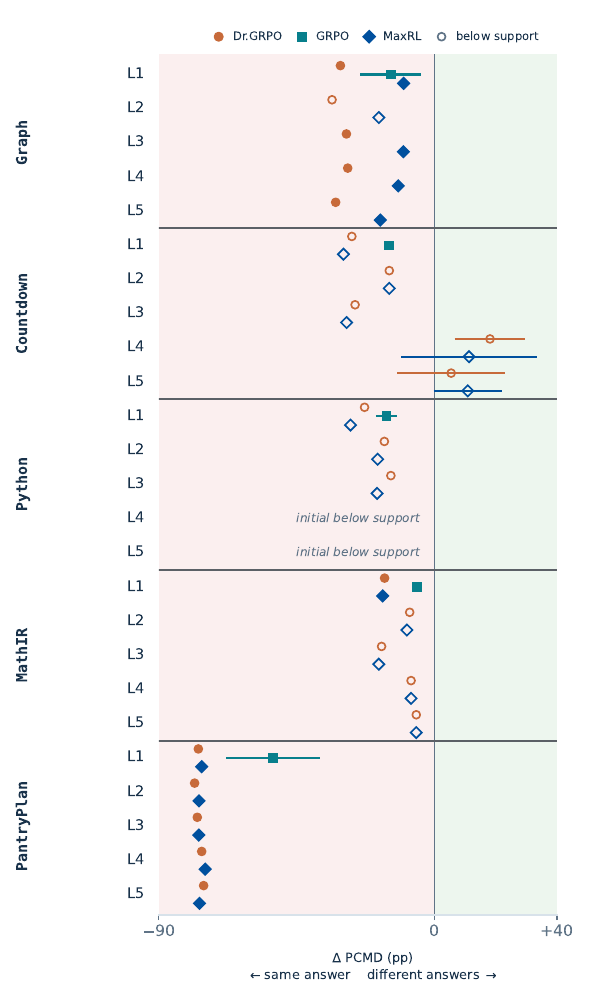}
  \end{minipage}
  \caption{\textbf{\texttt{Graph} and \texttt{PantryPlan} lose conditional solution diversity across scales and evaluation levels.}
  \textbf{(A)} Initial-to-final changes in \pmd{} at Level~1 across three model scales and no-replay objectives.
  \textbf{(B)} Level-1-trained policies evaluated on Levels~1--5, with higher levels measuring transfer.
  Filled marks show well-supported comparisons, open marks indicate limited support, and bars show nominal 95\% intervals where available.}
  \label{fig:concentration-all-scales}
  \label{fig:concentration-levels}
\end{figure}

\begin{figure}[t]
  \centering
  \includegraphics[width=\linewidth]{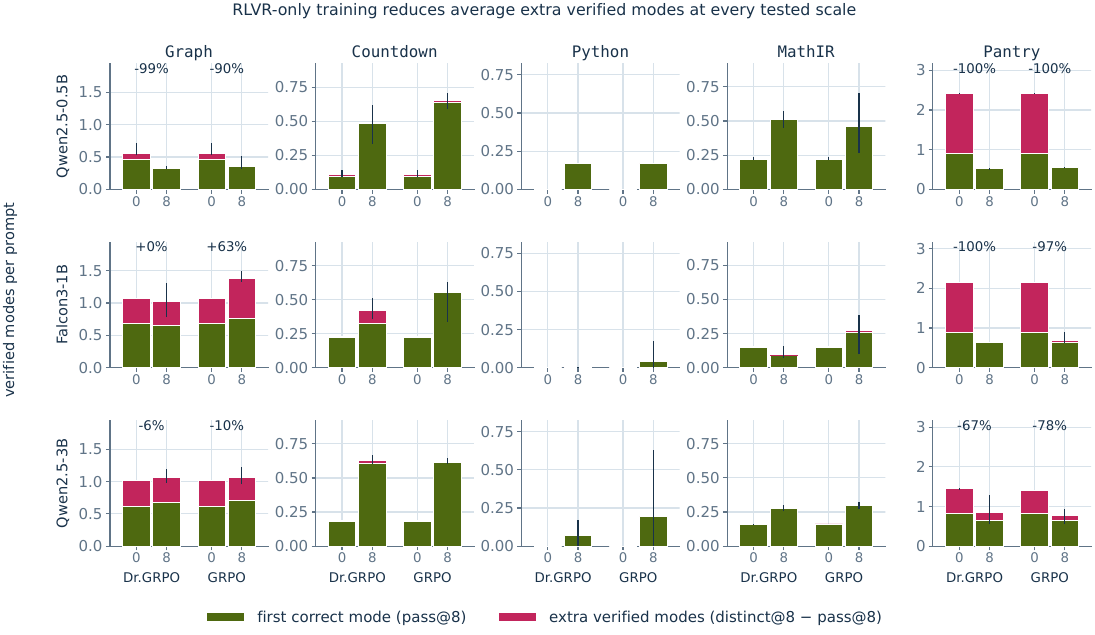}
  \caption{\textbf{RLVR-only training reduces additional verified modes across scales.}
  Bars decompose \texttt{distinct@8} into \texttt{pass@8} and additional verified modes before training and after eight passes for Dr.GRPO and GRPO across three model scales and five domains.
  Whiskers show ranges across five training seeds.}
  \label{fig:baseline-precheck}
\end{figure}

Initial evaluations are sampled separately for each training method, so small differences between their pass-0 estimates can reflect sampling variability.
These evaluations measure observed solution breadth; claims throughout concern sampled concentration rather than literal zero probability for unobserved modes.

\subsection{Sparse task gradients and output repetition}
\label{app:degenerate-groups}

For Dr.GRPO, a rollout group with identical rewards has zero fresh task-gradient contribution because all within-group advantages vanish.
We therefore measure the fraction of training groups containing both correct and incorrect responses across 3,072 Qwen2.5-0.5B updates.

\begin{table}[t]
  \centering
  \caption{\textbf{Fresh reward information is sparse in several control runs.}
  Entries are the fraction of Dr.GRPO rollout groups containing both correct and incorrect responses, averaged over five seeds.}
  \label{tab:mixed-groups}
  \small
  \begin{tabular}{@{}lr@{}}
    \toprule
    Domain & Mixed groups \\
    \midrule
    \texttt{Graph} & .144 \\
    \texttt{Countdown} & .140 \\
    \texttt{Python} & .015 \\
    \texttt{MathIR} & .364 \\
    \texttt{PantryPlan} & .123 \\
    \bottomrule
  \end{tabular}
\end{table}

The effect is most extreme in \texttt{Python}, where only $1.5\%$ of control groups are mixed and therefore $98.5\%$ have zero fresh task advantage.
Replay can still provide a likelihood gradient on these updates and can alter subsequent sampling so that future groups become informative.

The \texttt{Python} controls also exhibit strong output repetition.
At initialization, eight-response groups contain an average of 7.99 distinct completion strings per prompt, but all five control seeds later produce a single repeated completion for every prompt and remain at that level through the terminal evaluation.
At the end of training, eight-response groups contain 1.00 distinct completion per prompt for Dr.GRPO and 1.97 for Re:Dr.
The sampled-output collapse is therefore pronounced under the tested Dr.GRPO setting.
The decoding and group-size controls below test whether straightforward inference-time or sampling-budget changes recover the lost breadth.

\subsection{Replay gains with active and tuned controls}
\label{app:active-controls}

Control and replay runs share the learning-rate schedule, initialization, optimizer, prompt stream, update count, and fresh-rollout budget.
Per-response correctness increases from the initial to terminal checkpoint in all 15 control comparisons, confirming that the controls learn under these schedules.
Fresh reward information is nevertheless sparse in some settings, most notably Falcon3-1B \texttt{Python}, where only $.9\%$ of rollout groups contain both correct and incorrect responses.

Table~\ref{tab:active-controls} repeats the replay comparison on increasingly restrictive subsets of active controls.
Replay remains beneficial among the nine comparisons in which control \texttt{pass@8} increases and among the six comparisons with nonzero initial \texttt{pass@8} and at least $10\%$ mixed rollout groups.
All six comparisons in the latter subset have positive replay effects on both \texttt{pass@8} and \pmd{}.
Thus the replay advantage persists when the comparison is restricted to controls with clear evidence of task learning and nontrivial fresh reward signal.

\begin{table}[t]
  \centering
  \caption{\textbf{Replay gains persist when controls are learning.}
  Mean Re:Dr minus Dr.GRPO effects across increasingly restrictive subsets of the 15 model--domain comparisons.}
  \label{tab:active-controls}
  \small
  \setlength{\tabcolsep}{7pt}
  \begin{tabular}{@{}lrrr@{}}
    \toprule
    Control subset & $n$ & $\Delta$\texttt{pass@8} & $\Delta$\pmd{} \\
    \midrule
    All comparisons & 15 & +.352 & +.189 \\
    Increasing control \texttt{pass@8} & 9 & +.379 & +.118 \\
    Nonzero initial \texttt{pass@8}, $\geq10\%$ mixed groups & 6 & +.247 & +.136 \\
    \bottomrule
  \end{tabular}
\end{table}

\paragraph{Per-domain tuned controls.}
The active-control analysis above tests whether replay gains persist when the no-replay policy is demonstrably learning.
We compare replay against controls whose optimization settings are selected separately for each domain on correctness.
The tuning grid contains eight configurations spanning learning rate, rollout batching, and entropy regularization, evaluated on two selection seeds per domain.
The selected configuration maximizes terminal \texttt{pass@8}; ties within $.005$ are broken by \texttt{mean@8}.
Conditional solution diversity does not enter the selection criterion.
An independent replicate of the baseline configuration provides a second measurement of the same setting, yielding 90 measured domain--seed--configuration cells in total.
Configurations requiring four prompts per update are incompatible with the \texttt{PantryPlan} sampler, so the six such cells are excluded for that domain; the remaining 84 cells reach the terminal checkpoint.

\begin{table}[!htbp]
  \centering
  \caption{\textbf{Correctness-based tuning strengthens the no-replay control in four of five domains.}
  Settings are selected separately within each domain using terminal \texttt{pass@8}, with ties within $.005$ broken by \texttt{mean@8}; \pmd{} is not used for selection.}
  \label{tab:tuned-control-selection}
  \small
  \setlength{\tabcolsep}{7pt}
  \begin{tabular}{@{}lrrr@{}}
    \toprule
    Domain & Selected setting & Selected \texttt{pass@8} & Baseline \texttt{pass@8} \\
    \midrule
    \texttt{Graph}      & LR $5\times10^{-8}$ & .429 & .325 \\
    \texttt{Countdown}  & entropy $10^{-3}$    & .635 & .493 \\
    \texttt{Python}     & baseline              & .586 & .586 \\
    \texttt{MathIR}     & entropy $10^{-3}$    & .490 & .460 \\
    \texttt{PantryPlan} & LR $5\times10^{-8}$ & .600 & .531 \\
    \bottomrule
  \end{tabular}
\end{table}

Tuning increases the selected control's terminal correctness on \texttt{Graph}, \texttt{Countdown}, \texttt{MathIR}, and \texttt{PantryPlan}.
It can also preserve more conditional diversity in the control itself: the lower-learning-rate setting reaches \pmd{} $.062$ on \texttt{Graph} and $.077$ on \texttt{PantryPlan}, reducing the amount of headroom available to replay.

\texttt{Python} exhibits unusually high seed variability in the baseline estimate.
The two baseline seeds used in the comparison terminate at \texttt{pass@8} $.172$ and $1.000$, producing the reported mean of $.586$.
By contrast, the 16 Python cells in the tuning grid all terminate at $.172$--$.174$, and both cells in the independent baseline replicate terminate at $.172$.
Thus the selected two-seed baseline mean is dominated by one unusually high trajectory rather than by a systematic improvement from tuning.

\begin{table}[!htbp]
  \centering
  \caption{\textbf{Replay remains substantially above the selected controls in paired tuned-control comparisons.}
  Effects are Re:Dr minus the correctness-selected no-replay control on matched seeds; brackets give nominal 95\% paired intervals.}
  \label{tab:tuned-control-paired}
  \small
  \setlength{\tabcolsep}{6pt}
  \begin{tabular}{@{}lrrrr@{}}
    \toprule
    Domain & Control \pmd{} & Re:Dr \pmd{} & $\Delta$\pmd{} & $\Delta$\texttt{pass@8} \\
    \midrule
    \texttt{Graph} & .054 & .764 & $+.709\;[+.677,+.742]$ & $+.515\;[+.439,+.591]$ \\
    \texttt{PantryPlan} & .091 & .563 & $+.473\;[+.371,+.575]$ & $+.157\;[+.104,+.210]$ \\
    \bottomrule
  \end{tabular}
\end{table}

On both reported domains, every paired seed favors Re:Dr on conditional diversity.
The tuned controls are stronger than the canonical control on both correctness and, in these cases, their own solution breadth, yet the replay-trained policies remain substantially less concentrated and more successful at the terminal checkpoint.
This extends the active-control result from filtering for learning controls to directly optimizing the no-replay configuration for terminal correctness.

\section{Human-Authored Programming Evaluation}
\label{app:code}

We evaluate verified replay on human-authored constructive programming problems whose correctness checkers were not designed for this study.
We use 21 Codeforces problems from CodeContests+ and CodeContests-O with released output checkers and frozen hidden-input suites.
Eight problems are used for training and thirteen for development.
Correctness is determined by the released checker, while solution-mode identity is defined by the canonical witness signature produced across the frozen suite.
Because the suites are finite, these modes characterize behavior on the frozen checker suite rather than semantic program equivalence.

We train Qwen2.5-Coder-7B-Instruct in a $2\times2$ factorial crossing Dr.GRPO versus MaxRL with replay off versus on, using five paired seeds per arm and 128 updates.
Replay uses the same mechanism as in the main experiments: one verified exemplar per discovered solution mode, uniform rehearsal, and a capacity of 16 modes per prompt.
The controls execute the same buffer construction, scheduling, and exemplar scoring with the replay gradient set to zero.

Replay produces a small but consistent training-side increase in the fraction of fresh rollouts that verify.
Across matched seeds, the verified fraction of fresh rollouts is $+.014$ higher for Re:Dr than for Dr.GRPO, with a 95\% interval of $[+.005,+.023]$, and $+.014$ higher for Re:Max than for MaxRL, with interval $[-.002,+.031]$.
Pooling the ten matched replay--control seed pairs gives the same $+.014$ effect, with interval $[+.007,+.021]$ and a positive effect in nine of ten pairs.

The replay mechanism transfers unchanged to these human-authored constructive programming tasks and produces a small, consistent training-side correctness improvement.
This extends the method beyond the synthetic \mb{} task generators while retaining executable verification and solution-mode identity.

\section{Standard Mathematical Reasoning Evaluation}
\label{app:standard-math}

We evaluate whether the sampled-coverage effects observed on \mb{} extend to conventional mathematical reasoning benchmarks whose verifier checks only the final answer.
The evaluation covers MATH-500, AMC23, OlympiadBench, Minerva, AIME24, and AIME25, with 32 sampled responses per problem.
Unlike \mb{}, these benchmarks do not distinguish multiple correct derivations of the same final answer as different solution modes.
They therefore provide a complementary test of how replay changes the sampled correctness distribution outside the executable-mode setting.

\paragraph{Re:Dr restores sampled coverage lost under Dr.GRPO.}
Across all six benchmarks, Re:Dr has higher \texttt{pass@32} than the matched Dr.GRPO control.

\begin{table}[!htbp]
  \centering
  \caption{\textbf{Re:Dr restores sampled mathematical coverage lost under Dr.GRPO.}
  \texttt{pass@32} is the probability that at least one of 32 sampled responses is correct, while \texttt{avg@32} is mean per-response correctness across the same samples.}
  \label{tab:standard-math-redr}
  \small
  \setlength{\tabcolsep}{5pt}
  \begin{tabular}{@{}lrrrrr@{}}
    \toprule
    Benchmark & Base \texttt{pass@32} & Dr.GRPO & Re:Dr &
    Dr.GRPO \texttt{avg@32} & Re:Dr \texttt{avg@32} \\
    \midrule
    MATH-500      & .894 & .889 & .909 & .470 & .537 \\
    AMC23         & .875 & .767 & .850 & .261 & .339 \\
    OlympiadBench & .628 & .546 & .620 & .196 & .234 \\
    Minerva       & .474 & .491 & .540 & .135 & .167 \\
    AIME24        & .367 & .333 & .383 & .057 & .081 \\
    AIME25        & .200 & .211 & .217 & .021 & .025 \\
    \bottomrule
  \end{tabular}
\end{table}

The largest losses under Dr.GRPO occur on AMC23 and OlympiadBench.
On AMC23, \texttt{pass@32} falls from $.875$ for the base model to $.767$ after Dr.GRPO, while Re:Dr reaches $.850$.
On OlympiadBench, the corresponding values are $.628$, $.546$, and $.620$.
Thus Re:Dr recovers most of the sampled-coverage reduction on both benchmarks.

The change is not obtained by sacrificing per-response correctness.
Re:Dr also increases \texttt{avg@32} on every benchmark, including from $.470$ to $.537$ on MATH-500, $.261$ to $.339$ on AMC23, and $.196$ to $.234$ on OlympiadBench.
Greedy correctness remains essentially unchanged, so the principal difference appears in the sampled distribution rather than in the policy's highest-probability answer.
Majority-vote accuracy changes substantially less than \texttt{pass@32}, consistent with replay increasing access to correct alternatives without necessarily changing which final answer receives the most probability mass.

\paragraph{MaxRL largely avoids the sampled-coverage squeeze without replay.}
The corresponding MaxRL comparison gives a different pattern.
On four benchmarks with matched MaxRL and Re:Max evaluations, MaxRL already retains \texttt{pass@32} close to the base model and close to Re:Dr's recovered coverage.

\begin{table}[!htbp]
  \centering
  \caption{\textbf{MaxRL largely preserves sampled mathematical coverage before replay is added.}
  Results use one matched seed per MaxRL--Re:Max arm and the same 32-sample evaluation.}
  \label{tab:standard-math-remax}
  \small
  \setlength{\tabcolsep}{6pt}
  \begin{tabular}{@{}lrrrr@{}}
    \toprule
    Benchmark & Base & Dr.GRPO & MaxRL & Re:Max \\
    \midrule
    MATH-500 & .894 & .889 & .902 & .904 \\
    AMC23    & .875 & .767 & .850 & .850 \\
    AIME24   & .367 & .333 & .333 & .300 \\
    AIME25   & .200 & .211 & .167 & .200 \\
    \bottomrule
  \end{tabular}
\end{table}

On MATH-500, MaxRL reaches \texttt{pass@32} $.902$ and per-response accuracy $.498$, compared with $.889$ and $.470$ for Dr.GRPO.
On AMC23, MaxRL reaches $.850$ \texttt{pass@32}, recovering most of the Dr.GRPO reduction without replay.
This pattern is consistent with MaxRL's stronger weighting of successes when they are rare.

Adding replay to MaxRL produces little further change in this final-answer setting.
Re:Max remains within $.033$ \texttt{pass@32} of MaxRL on each of the four paired benchmarks, with differences in both directions.
The distinction from \mb{} is structural: standard mathematical answer keys assign all correct derivations of a problem to the same verified outcome.
The replay buffer therefore contains only one verified solution mode per problem and cannot balance rehearsal across distinct correct derivations.
On \mb{}, where the verifier distinguishes different correct solutions, Re:Max can instead retain and rehearse multiple verified modes and produces substantial gains over MaxRL.

Together, the two objective pairs separate two effects.
MaxRL can itself preserve sampled correctness by emphasizing rare successes, while verified replay provides an additional mechanism when the verifier distinguishes multiple correct solution modes.
For Dr.GRPO, which exhibits a pronounced reduction in sampled mathematical coverage, replay restores much of that coverage even under final-answer-only verification.

\clearpage
\phantomsection
\begin{center}\large\textbf{Part IV. What diversity represents}\end{center}
\addcontentsline{toc}{part}{Part IV. What diversity represents}

\section{Prompt-Guidance Ablation}
\label{app:prompt-hint-ablation}
\label{sec:prompt-hint-ablation}

We compare benchmark and neutral prompts on the same 32 \texttt{Python}, \texttt{MathIR}, and \texttt{PantryPlan} problems at Levels~2--3, with eight responses per condition.
The neutral prompts remove task-specific solution guidance while leaving the problem, answer format, verifier, and remaining interface unchanged.
Because the \texttt{Python} intervention also removes an executable example and associated divisor guidance, it should be interpreted as a combined prompt-interface change rather than an isolated hint removal.

\begin{table}[!htbp]
  \centering
  \caption{\textbf{Removing guidance leaves \texttt{MathIR} concentrated but sharply reduces \texttt{Python} correctness.}
  Results use matched Qwen2.5-0.5B checkpoints and strict verification on 32 paired problems per level.
  \pmd{} is reported only where sufficient jointly eligible problems remain.}
  \label{tab:prompt-hints-local-strict}
  \small
  \setlength{\tabcolsep}{5pt}
  \begin{tabular}{llrr}
    \toprule
    Model & Cell & \texttt{pass@8}: benchmark $\to$ neutral & \pmd{}: benchmark $\to$ neutral \\
    \midrule
    Dr.GRPO & \texttt{Python} L2 & $.812 \to .000$ & -- \\
    Dr.GRPO & \texttt{Python} L3 & $1.000 \to .000$ & -- \\
    Re:Dr & \texttt{Python} L2 & $.887 \to .556$ & $.231 \to .173$ \\
    Re:Dr & \texttt{Python} L3 & $1.000 \to .494$ & $.256 \to .323$ \\
    Dr.GRPO & \texttt{MathIR} L2 & $.531 \to .506$ & $.000 \to .000$ \\
    Re:Dr & \texttt{MathIR} L2 & $.994 \to .975$ & $.000 \to .000$ \\
    Re:Dr & \texttt{MathIR} L3 & $.312 \to .275$ & $.000 \to .000$ \\
    \bottomrule
  \end{tabular}
\end{table}

Removing the operation-order guidance does not recover alternative verified \texttt{MathIR} solution modes.
Every eligible trained \texttt{MathIR} comparison has $\pmd{}=.000$ under both benchmark and neutral wording at Levels~2--3.
At Level~2, Re:Dr remains highly successful under neutral wording, with \texttt{pass@8} $.975$ compared with $.994$ under benchmark wording.

The \texttt{Python} intervention instead causes a large loss of correctness, particularly for Dr.GRPO.
Re:Dr retains neutral-wording \texttt{pass@8} of $.556$ and $.494$ at Levels~2 and~3, but its two reportable \pmd{} changes have opposite signs.
The \texttt{Python} result therefore shows that prompt interface can materially change both correctness and the observed diversity available for comparison.

The two-seed \texttt{PantryPlan} comparison shows smaller, mixed effects and is reported as a supplementary prompt-sensitivity check.

\paragraph{Hosted prompt-guidance sensitivity.}
\phantomsection
\label{app:hosted-discovery-complete}
For fixed hosted deployments, we evaluate benchmark and neutral wording with 64 responses on 16 \texttt{MathIR} problems per level.
Among problems with at least eight verified responses under both wordings, neutral wording increases eight-correct-draw solution breadth for GPT-5.4 by $.794$ at Level~2 and $.204$ at Level~3, and for GPT-5.6 Sol by $.872$ and $.860$, respectively.
The corresponding Grok~4.3 comparisons are inconclusive.
Thus some hosted-model concentration is sensitive to prompt guidance, even though the local trained \texttt{MathIR} checkpoints remain concentrated after the benchmark hint is removed.

\section{Inference-Time Decoding Controls}
\label{app:decoding-objection}

We test whether the low solution diversity of RLVR-only policies can be recovered by changing inference-time decoding at fixed checkpoints.

\paragraph{Temperature.}
We evaluate eight-pass Qwen2.5-0.5B Dr.GRPO checkpoints at temperatures $T\in\{.5,.7,1.0,1.3,1.6,2.0\}$ with top-$p=1$ and the same prompts and response budget.
Across the six temperatures, the largest control \pmd{} is $.021$ in \texttt{Graph}, $.016$ in \texttt{Countdown}, $.002$ in \texttt{MathIR}, and $.004$ in \texttt{PantryPlan}.
For comparison, Re:Dr at $T=1$ reaches $.562$, $.485$, $.020$, and $.330$, respectively.
\texttt{Python} controls do not produce enough verified responses to estimate conditional diversity under our reporting threshold.
Thus temperature changes alone do not recover the solution diversity retained by replay in the domains where the comparison is defined.

\paragraph{Sampling budget and nucleus sampling.}
A twelve-pass control cohort evaluates eleven decoding settings spanning temperature, top-$p$, and response-group size.
Increasing the group size from $K=8$ to $K=32$ changes control \pmd{} by at most $.010$ across matched settings, and no $K=32$ control mean exceeds $.054$.
The broader grid likewise leaves control diversity low across the eligible domains.
These comparisons change inference-time sampling at fixed checkpoints and therefore do not test changes to rollout sampling during training.

Overall, the tested decoding interventions do not eliminate the diversity gap between RLVR-only and replay-trained policies.
Missing \pmd{} estimates indicate insufficient verified responses rather than evidence that no alternative correct modes exist.

\section{Sampling-Budget Ablation}
\label{app:sampling-budget-ablation}

\subsection{Sampling-budget design}
\label{sec:sampling-budget-design}
\label{sec:discovery-curves}

We test whether small sampling budgets understate observed solution breadth by collecting up to 64 responses per prompt in the local evaluation and 512 responses per prompt in the hosted evaluation.
For a response pool of size $N$ with $c$ verified responses and solution-mode counts $n_m$, the expected \texttt{pass@}$k$ and \texttt{distinct@}$k$ of a uniformly sampled $k$-response subset are
\[
  P_k
  =
  1-
  \frac{\binom{N-c}{k}}{\binom Nk},
  \qquad
  D_k
  =
  \sum_m
  \left[
    1-
    \frac{\binom{N-n_m}{k}}{\binom Nk}
  \right].
\]
Incorrect, empty, refused, and truncated responses remain in the sampling denominator.

To compare breadth at a fixed number of verified responses, we also use
\[
  R_m
  =
  \sum_j
  \left[
    1-
    \frac{\binom{c-n_j}{m}}{\binom cm}
  \right],
\]
computed only when at least $m$ verified responses are available.
These quantities describe discovery within finite response pools; throughout this section, breadth refers to modes observed under the stated sampling budget.

\subsection{Discovery at larger sampling budgets}
\label{app:sampling-budget-effects}

We extend the local-model evaluation from eight to 64 responses per prompt on 16 held-out problems per domain and level.
The trained checkpoints were optimized on Level~2, so Level~3 evaluation measures transfer.

\paragraph{\texttt{MathIR}.}
Increasing the budget from eight to 64 responses reveals essentially no additional solution modes after training.
At Level~2, every measured trained condition produces at most one verified solution mode per checkpoint and problem.
At Level~3, only one Dr.GRPO checkpoint--problem pair produces a second mode.
The near-single-mode behavior observed at eight draws therefore largely persists at 64.

\paragraph{\texttt{Python}.}
For Re:Dr under benchmark wording, increasing the budget from eight to 64 responses adds $.058$ observed solution modes per prompt at Level~2 and $.153$ at Level~3.
Neutral wording reveals more additional solutions but also changes correctness substantially, so the magnitude of discovery depends on both sampling budget and prompt interface.

\paragraph{\texttt{PantryPlan}.}
Across trained checkpoints, increasing the budget from eight to 64 responses adds approximately $.42$--$1.02$ observed solution modes per prompt.
The ordering between Dr.GRPO and Re:Dr can change with sampling budget in this two-seed comparison, so an eight-response result need not determine which checkpoint exposes more modes at substantially larger budgets.

Overall, larger budgets reveal additional modes in some domains while leaving the strongest concentration patterns intact.
The result is therefore not an artifact of using only eight evaluation draws.

\subsection{Discovery at 512 draws}
\label{app:gpt56-all-levels-discovery}

We evaluate GPT-5.6 Sol on 32 problems from each of the five domains at Levels~1--3, with 512 responses per prompt.
This yields 480 prompt pools and 245,760 responses.

Even at this budget, observed breadth remains far below the certified or enumerated solution support in several domains.
Mean \texttt{distinct@512} is $2.53$, $2.63$, and $1.72$ for \texttt{Graph} across Levels~1--3.
For \texttt{MathIR}, it is $1.91$, $1.00$, and $.97$, despite at least five certified solution modes per problem.
For \texttt{Python}, it is $3.75$, $1.44$, and $1.34$, compared with enumerated averages of roughly $222$, $336$, and $355$ solution modes.
For \texttt{PantryPlan}, it ranges from $4.88$ to $6.91$, compared with certified-support means near $16$--$17$.

Late discovery also differs by domain.
Doubling the budget from 256 to 512 responses adds approximately $.14$--$.27$ modes in \texttt{Graph}, at most $.07$ in Level-2/3 \texttt{Python}, approximately zero in Level-2/3 \texttt{MathIR}, and $.6$--$.7$ in \texttt{PantryPlan}.
Thus larger budgets continue to reveal alternatives in some domains while leaving others close to single-mode behavior, even after 512 draws per prompt.

\section{Inference-Time Alternatives and Mode Definitions}
\label{app:inference-followups}

\subsection{Complementarity across hosted models}
\label{app:cross-model-overlap}

Seven hosted deployments each produce eight responses to the same 1,920 prompts.
For each of the 21 model pairs, we compare a mixed portfolio containing four responses from each model with the two constituent eight-response portfolios.

Every pair increases expected distinct solution modes relative to the average of its two constituents, with gains ranging from $.115$ to $.266$ modes.
Ten pairs exceed both constituent point estimates, and nine have intervals above both.
When the comparison instead fixes the portfolio to four verified responses, with two drawn from each model, every pair again exceeds the average constituent by $.073$--$.177$ modes.

The gains persist under coarser solution-mode definitions and formatting-normalized grading.
Thus different deployments expose reproducibly complementary sets of verified solutions under the tested sampling budgets.

\subsection{Sensitivity to solution-mode identity}
\label{app:coarser-keys}

We repeat selected analyses under coarser task-relevant definitions of solution identity.
For \texttt{Graph}, we group solutions related by global color renaming while respecting anchored colors.
For \texttt{Python}, we identify opposite members $d$ and $n/d$ of the same unordered factor pair.
For \texttt{Countdown}, we ignore associative reorderings of addition and multiplication while preserving operands and subtraction or division boundaries.
The \texttt{MathIR} and \texttt{PantryPlan} definitions are unchanged.

Coarsening reduces observed solution counts and increases concentration, as expected.
The largest sensitivity occurs in \texttt{Python}: under benchmark wording, the Re:Dr minus Dr.GRPO \texttt{distinct@8} gain falls from $.431$ to $.075$ at Level~2 and from $.419$ to $.000$ at Level~3.
Under neutral wording, the corresponding coarse-mode gains remain $.556$ and $.500$.
The measured diversity gain therefore depends in part on what distinctions are considered meaningful, particularly for \texttt{Python}.

\section{Portfolio Survival under Withdrawn Options}
\label{app:portfolio-withdrawals}
\label{app:pantry-adaptation}

We test whether retaining multiple verified solution modes makes a sampled portfolio more useful after the task changes.
For each prompt, we withdraw one available option: a color choice in \texttt{Graph}, an operation in \texttt{Countdown}, a divisor value in \texttt{Python}, an action in \texttt{MathIR}, or an ingredient in \texttt{PantryPlan}.
We restrict the primary analysis to binding withdrawals that remove at least one certified solution mode while leaving at least one certified mode feasible.

To reduce dependence on overall correctness, we compare survival after sampling a fixed number of verified responses.
At four verified draws, replay improves survival most strongly in \texttt{Graph}, \texttt{Countdown}, \texttt{Python}, and \texttt{PantryPlan}, while \texttt{MathIR} changes little.

\begin{table}[!htbp]
  \centering
  \caption{\textbf{Replay generally improves survival after one or two withdrawn options.}
  Entries are replay minus matched control survival in percentage points at four verified draws.}
  \label{tab:withdrawal-survival-summary}
  \small
  \setlength{\tabcolsep}{6pt}
  \begin{tabular}{@{}lrrrr@{}}
    \toprule
    & \multicolumn{2}{c}{One withdrawal} & \multicolumn{2}{c}{Two withdrawals} \\
    \cmidrule(lr){2-3}\cmidrule(l){4-5}
    Domain & Re:Dr & Re:Max & Re:Dr & Re:Max \\
    \midrule
    \texttt{Graph} & 19.65 & 10.99 & 23.65 & 12.51 \\
    \texttt{Countdown} & 9.58 & 0.10 & 6.11 & -4.87 \\
    \texttt{Python} & 7.67 & 5.68 & 10.10 & 8.16 \\
    \texttt{MathIR} & 0.14 & 0.09 & 0.20 & 0.12 \\
    \texttt{PantryPlan} & 9.78 & 6.61 & 7.23 & 4.94 \\
    \bottomrule
  \end{tabular}
\end{table}

\subsection{Two withdrawals at once}
\label{app:withdrawal-severity}
The more severe withdrawal condition removes pairs of binding options that leave at least one certified solution mode feasible.
The pattern remains strongest in \texttt{Graph}, \texttt{Python}, and \texttt{PantryPlan}, while the Re:Max \texttt{Countdown} point estimate reverses.
These results show that survival gains can persist under a more severe change but are not uniform across domains and objectives.

\subsection{Decoding settings as a substitute}
\label{app:withdrawal-decoding}
We compare replay with eleven inference-time control settings spanning temperature, nucleus sampling, and response-group size on a Level-1 cohort.
For each domain, we select the control setting with the highest four-verified-draw survival.
Replay exceeds this selected control in all five domains, with paired gains ranging from $1.37$ to $36.97$ percentage points and pointwise intervals above zero.
Thus the tested inference-time decoding changes do not reproduce the portfolio-survival gains of replay in this cohort.

\subsection{Recovery after a withdrawal}
\label{app:withdrawal-recovery}
The recovery evaluation allows up to eight additional responses when the initial eight-response portfolio contains no surviving solution.
This evaluation uses Level-1 Qwen2.5-0.5B Dr.GRPO and Re:Dr checkpoints from two matched seeds and a common recovery procedure across all five domains.

\begin{table}[!htbp]
  \centering
  \caption{\textbf{Replay improves ordinary-sampling recovery after a withdrawal.}
  Entries are Re:Dr minus Dr.GRPO on the same withdrawn tasks.
  Saved and resolved-by-eight differences are percentage points; calls are additional recovery responses.}
  \label{tab:withdrawal-recovery-summary}
  \small
  \setlength{\tabcolsep}{7pt}
  \begin{tabular}{@{}lrrr@{}}
    \toprule
    Domain & $\Delta$ saved & $\Delta$ by 8 & $\Delta$ calls \\
    \midrule
    \texttt{Graph} & 67.58 & 68.36 & -5.37 \\
    \texttt{Countdown} & 6.76 & 8.11 & -0.55 \\
    \texttt{Python} & 55.86 & 63.67 & -4.67 \\
    \texttt{MathIR} & 20.54 & 20.98 & -1.68 \\
    \texttt{PantryPlan} & 16.41 & 20.70 & -1.44 \\
    \bottomrule
  \end{tabular}
\end{table}

All ordinary-sampling recovery intervals exclude zero in the favorable direction.
Temperature-tuned portfolios show the same qualitative pattern.
Diversity prompting is less consistent and reverses the comparison in \texttt{PantryPlan}, so the recovery advantage depends on how the initial portfolio is constructed.

These withdrawal experiments connect sampled solution breadth to a concrete downstream benefit: retaining or recovering a feasible response after the task changes.
Because replay changes the trained policy as a whole, the survival gain should be interpreted as a property of replay-trained portfolios rather than as a mediation estimate for diversity alone.

\clearpage
\phantomsection
\begin{center}\large\textbf{Part V. Hosted deployment analyses}\end{center}
\addcontentsline{toc}{part}{Part V. Hosted deployment analyses}

\section{Inference-Only Concentration in GPT-5.6 Sol}
\label{app:hosted-concentration}

We evaluate GPT-5.6 Sol on 128 held-out prompts from each of five domains at Levels~1--3, with eight stateless responses per prompt.
The resulting 15,360 responses use medium reasoning, an 8,192-token output limit, no tools, and no conversation state.
These inference-only measurements characterize one hosted deployment and are separate from the matched training interventions.

\begin{table}[!htbp]
  \centering
  \caption{\textbf{GPT-5.6 Sol produces few verified solution modes at each benchmark level.}
  Values average the five domains equally under strict grading.
  $P$ is \texttt{pass@8}, $D$ is mean \texttt{distinct@8}, and $C$ is correct-pair collision.}
  \label{tab:gpt56-hosted-summary}
  \small
  \setlength{\tabcolsep}{8pt}
  \begin{tabular}{@{}lrrr@{}}
    \toprule
    Level & $P$ (\%) & $D$ & $C$ (\%) \\
    \midrule
    1 & 95.8 & 1.71 & 74.0 [71.9, 76.1] \\
    2 & 77.5 & 1.16 & 82.3 [80.3, 84.4] \\
    3 & 80.3 & 1.18 & 84.5 [82.8, 86.2] \\
    \bottomrule
  \end{tabular}
\end{table}

Concentration can remain high even when correctness is nearly perfect.
On Level-3 \texttt{Graph}, all 1,024 responses verify, yet mean \texttt{distinct@8} is $1.16$ and $95.1\%$ of correct-response pairs share the same solution mode, compared with a uniform reference near $19\%$.
Level-2 and Level-3 \texttt{MathIR} are similarly concentrated, while \texttt{PantryPlan} exhibits greater conditional diversity.

Strict formatting affects observed success most strongly in \texttt{Python}, so the hosted results characterize the tested deployment and interface.
The concentration pattern nevertheless appears across multiple domains and benchmark levels.

\section{Comparison across Hosted Deployments}
\label{app:hosted-comparison}

We compare seven hosted deployments on the same five executable domains, three levels, and eight responses per prompt.
The common-prompt cohort contains 107,520 responses.
Provider interfaces, reasoning implementations, sampling defaults, and compute differ across systems, so the comparison characterizes deployment-level solution concentration rather than isolating model identity.

\subsection{Hosted-model concentration}
\label{app:hosted-mode-diversity}

Table~\ref{tab:hosted-cohort-mode-diversity} reports promptwise \pmd{} under strict grading and benchmark wording.
Each cell averages prompts with at least two verified responses, and cells with fewer than thirty eligible prompts are omitted.

GPT-5.6 Sol reaches $\pmd{}=.000$ on \texttt{MathIR} at Levels~2--3 and $.049$ on Level-3 \texttt{Graph}, while \texttt{PantryPlan} reaches $.531$, $.386$, and $.455$ across Levels~1--3.
Across the thirteen domain--level cells reportable for all seven deployments, macro \pmd{} ranges from $.142$ to $.320$.
The corresponding effective-mode transforms range from $1.17$ to $1.47$.
Removing \texttt{MathIR} raises the summaries but leaves the same qualitative concentration pattern, with the largest transformed value reaching $1.58$.
Because these summaries depend on the averaging population, we use them to characterize concentration rather than to rank deployments.

\begin{table}[!htbp]
  \centering
  \caption{\textbf{Hosted deployments remain concentrated among verified solutions.}
  Macro \pmd{} averages the thirteen domain--level cells reportable for every deployment.
  Cells require at least thirty prompts with two verified responses.}
  \label{tab:hosted-cohort-mode-diversity}
  \small
  \begin{tabular}{@{}lccccc@{}}
    \toprule
    & & & \multicolumn{3}{c}{Macro \pmd{} by level} \\
    \cmidrule(lr){4-6}
    Deployment & Macro \pmd{} & $\neff$ & L1 & L2 & L3 \\
    \midrule
    DeepSeek V4 Pro & .320 & 1.47 & .289 & .319 & .232 \\
Kimi K3 & .302 & 1.43 & .374 & .307 & .253 \\
GPT-5.4 & .264 & 1.36 & .338 & .247 & .224 \\
Grok 4.3 & .239 & 1.31 & .275 & .179 & .171 \\
GPT-5.6 Sol & .226 & 1.29 & .256 & .229 & .183 \\
Claude Opus 4.8 & .215 & 1.27 & .217 & .229 & .168 \\
Claude Opus 5 & .142 & 1.17 & .140 & .164 & .121 \\
    \bottomrule

  \end{tabular}
\end{table}

A separate common-prompt comparison excludes \texttt{Python}, whose interface differs for one deployment.
Under formatting-normalized grading, all seven deployments solve most prompts while producing relatively few verified solution modes.

\begin{table}[!htbp]
  \centering
  \caption{\textbf{Hosted deployments produce few verified modes on common prompts.}
  Equal-domain means over \texttt{Graph}, \texttt{Countdown}, \texttt{MathIR}, and \texttt{PantryPlan}, with 128 prompts per domain and level and eight responses per prompt.
  \texttt{pass@8} is prompt success and \# modes is mean \texttt{distinct@8}, including unsolved prompts.}
  \label{tab:hosted-common-prompts}
  \small
  \setlength{\tabcolsep}{4.5pt}
  \begin{tabular}{@{}lrrrrrr@{}}
    \toprule
    & \multicolumn{2}{c}{Level 1} & \multicolumn{2}{c}{Level 2} & \multicolumn{2}{c}{Level 3} \\
    \cmidrule(lr){2-3}\cmidrule(lr){4-5}\cmidrule(l){6-7}
    Deployment & \texttt{pass@8} & \# modes & \texttt{pass@8} & \# modes & \texttt{pass@8} & \# modes \\
    \midrule
    GPT-5.6 Sol & 97.5 & 1.90 & 98.2 & 1.59 & 99.4 & 1.54 \\
    Claude Opus 5 & 97.9 & 1.42 & 99.6 & 1.45 & 99.8 & 1.35 \\
    GPT-5.4 & 97.7 & 2.00 & 98.0 & 1.68 & 99.4 & 1.62 \\
    Grok 4.3 & 99.4 & 2.12 & 100.0 & 1.70 & 100.0 & 1.68 \\
    Kimi K3 & 98.0 & 2.26 & 100.0 & 1.95 & 100.0 & 1.79 \\
    Claude Opus 4.8 & 98.4 & 1.70 & 100.0 & 1.75 & 99.8 & 1.56 \\
    DeepSeek V4 Pro & 98.0 & 2.16 & 98.8 & 2.04 & 99.6 & 1.75 \\
    \bottomrule
  \end{tabular}
\end{table}

Cross-level summaries combine changes in task population and guidance, so we use them to characterize concentration across benchmark regimes rather than as an isolated difficulty intervention.

\subsection{\texttt{Python} prompt-interface sensitivity}
\label{app:hosted-python-prompt}
\label{app:python-wording-second-deployment}

We test the same \texttt{Python} problems under the benchmark interface and a direct-expression interface without a system message.
The effect of this interface change differs sharply across the two deployments on which it is tested.

\begin{table}[!htbp]
  \centering
  \caption{\textbf{\texttt{Python} prompt-interface effects differ across deployments.}
  Entries show benchmark $\to$ direct-expression results under formatting-normalized grading.
  Accuracy is per-response verified accuracy, and $D$ is mean \texttt{distinct@8}.}
  \label{tab:hosted-python-interface}
  \small
  \setlength{\tabcolsep}{6pt}
  \begin{tabular}{@{}llrrr@{}}
    \toprule
    Deployment & Level & Refusals & Accuracy (\%) & $D$ \\
    \midrule
    Claude Opus 5 & 1 & $409\to0$ & $60.0\to100.0$ & $1.12\to1.49$ \\
    Claude Opus 5 & 2 & $1021\to0$ & $.3\to100.0$ & $.02\to1.38$ \\
    Claude Opus 5 & 3 & $1024\to3$ & $.0\to99.7$ & $.00\to1.41$ \\
    GPT-5.6 Sol & 1 & $0\to0$ & $89.8\to20.9$ & $1.36\to.41$ \\
    GPT-5.6 Sol & 2 & $0\to0$ & $58.7\to16.7$ & $1.15\to.45$ \\
    GPT-5.6 Sol & 3 & $0\to0$ & $56.3\to19.3$ & $1.13\to.48$ \\
    \bottomrule
  \end{tabular}
\end{table}

For Claude Opus 5, the direct-expression interface largely removes provider refusals and sharply increases verified success.
For GPT-5.6 Sol, the same interface change lowers both verified accuracy and sampled solution breadth.
A non-verifying LaTeX wrapper also appears more frequently in the GPT-5.6 Sol direct-expression condition, so part of that gap is associated with output formatting.
The GPT-5.6 Sol interface comparison uses a separate deployment configuration, so inference focuses on the within-deployment benchmark-to-direct change.

Formatting normalization corrects typography before applying the same executable verifier.
It does not alter mathematical values, program logic, or the definition of a solution mode.

\subsection{Reasoning-control comparison}
\label{app:hosted-reasoning-control}

Five deployments answer the same 480 prompts under medium reasoning and with explicit reasoning disabled through the provider interface.
Each condition uses 32 prompts in each domain--level cell and eight responses per prompt.

\begin{table}[!htbp]
  \centering
  \caption{\textbf{Disabling explicit reasoning lowers sampled coverage for all five tested deployments.}
  Entries show medium reasoning $\to$ disabled reasoning.
  \pmd{} uses jointly eligible prompts and therefore need not move with raw mode counts.}
  \label{tab:hosted-reasoning-control}
  \small
  \setlength{\tabcolsep}{6pt}
  \begin{tabular}{@{}lrrr@{}}
    \toprule
    Deployment & \texttt{pass@8} (\%) & \# modes & \pmd{} \\
    \midrule
    GPT-5.6 Sol & $98.1\to75.2$ & $1.62\to.99$ & $.226\to.145$ \\
    GPT-5.4 & $98.5\to74.6$ & $1.85\to1.18$ & $.255\to.224$ \\
    Grok 4.3 & $100.0\to86.5$ & $1.70\to1.50$ & $.154\to.364$ \\
    Kimi K3 & $99.0\to93.5$ & $2.06\to1.62$ & $.293\to.426$ \\
    Claude Opus 4.8 & $99.4\to97.1$ & $1.62\to1.53$ & $.202\to.209$ \\
    \bottomrule
  \end{tabular}
\end{table}

All five deployments have lower overall \texttt{pass@8} and \texttt{distinct@8} when explicit reasoning is disabled, but the change in conditional diversity is not uniform.
The reasoning-control intervention changes the provider-exposed reasoning setting rather than internal compute directly.
A separate DeepSeek V4 Pro comparison gives strict-grading \pmd{} of $.282$ with medium reasoning and $.369$ with reasoning disabled on 437 jointly eligible prompts.
Claude Opus 5 is omitted because the provider's nominally disabled condition still returns thinking tokens.

\subsection{Temperature sensitivity across deployments}
\label{app:hosted-temperature}
\label{app:gpt56-temperature}

We first sweep requested temperature for GPT-5.6 Sol with explicit reasoning disabled on the same 480 prompts.
Formatting-normalized \texttt{pass@8} and sampled solution breadth both increase over much of the sweep, while per-response accuracy decreases.

\begin{table}[!htbp]
  \centering
  \caption{\textbf{GPT-5.6 Sol exposes more sampled modes at higher requested temperature in this sweep.}
  Each temperature uses the same 480 prompts and eight responses per prompt.
  $A$ is per-response verified accuracy.}
  \label{tab:gpt56-temperature-summary}
  \small
  \setlength{\tabcolsep}{8pt}
  \begin{tabular}{@{}rrrr@{}}
    \toprule
    $T$ & \texttt{pass@8} (\%) & \texttt{distinct@8} & $A$ (\%) \\
    \midrule
    0.0 & 65.6 & .673 & 62.1 \\
    0.5 & 72.3 & .840 & 61.2 \\
    1.0 & 75.2 & .963 & 59.1 \\
    1.5 & 79.8 & 1.125 & 57.7 \\
    2.0 & 79.4 & 1.146 & 55.9 \\
    \bottomrule
  \end{tabular}
\end{table}

Between $T=0$ and $T=2$, \texttt{pass@8} increases by $13.8$ percentage points ($[10.8,16.7]$) and \texttt{distinct@8} by $.473$ modes ($[.417,.529]$).
This increase in finite-sample coverage accompanies a decrease in per-response accuracy from $62.1\%$ to $55.9\%$.

For Grok~4.3 and Kimi~K3, we compare requested temperatures $T=1.0$ and $T=1.5$ on 120 common prompts.

\begin{table}[!htbp]
  \centering
  \caption{\textbf{Temperature effects are deployment-dependent.}
  Entries show $T=1.0\to1.5$ under strict grading.}
  \label{tab:hosted-temperature-summary}
  \small
  \setlength{\tabcolsep}{8pt}
  \begin{tabular}{@{}lrr@{}}
    \toprule
    Deployment & Accuracy (\%) & \texttt{distinct@8} \\
    \midrule
    Grok 4.3 & $98.65\to96.35$ & $1.608\to1.658$ \\
    Kimi K3 & $96.56\to21.25$ & $1.933\to.983$ \\
    \bottomrule
  \end{tabular}
\end{table}

For Grok~4.3, the mode-count change is $+.050$ with interval $[-.058,+.158]$, while accuracy decreases by $2.29$ percentage points.
For Kimi~K3, 696 of 960 responses terminate at the token limit at $T=1.5$, compared with none at $T=1.0$, accompanying the large drops in both accuracy and sampled modes.
Temperature can therefore increase sampled solution coverage in some settings, but it is not a deployment-independent substitute for preserving solution modes during training.

\subsection{Native provider outcomes}
\label{app:hosted-provider-outcomes}

Provider metadata distinguish declared refusals from executable verification failures.
In the hosted cohort, Claude Opus 5 is the only deployment with provider-declared refusals, totaling 2,522 of 15,360 responses.
Most occur in \texttt{Python}: 409 of 1,024 responses at Level~1, 1,021 at Level~2, and all 1,024 at Level~3; the remaining 68 occur in \texttt{MathIR}.
These refusals explain why several Claude Opus 5 conditional-diversity cells are undefined.
Other deployments have no provider-declared refusals in this cohort, although empty responses, truncations, and other verification failures still count as unsuccessful outputs.

\clearpage
\phantomsection
\begin{center}
    \large\textbf{Part VI. Mechanism and efficiency}
\end{center}
\addcontentsline{toc}{part}{Part VI. Mechanism and efficiency}

\section{Mode Collapse and Verified Replay}
\label{app:theory}

Binary RLVR optimizes how much probability a policy assigns to correct outputs. However, it does not tell us how that probability is distributed among different correct solutions \citep{lochab2026ucpo}.
Under a softmax parameterization, on-policy expected-reward updates can amplify existing probability differences among equally rewarded outcomes, leading to concentration \citep{sinha2026expected,lochab2026ucpo}.
Here we use a simple mode-level model to connect this mechanism to our setting and to show why replay changes it.

\subsection{A mode-level model of RLVR}
\label{app:theory-grpo-mean}
\label{app:theory-collapse}
\label{app:theory-objective-inertness}

Fix a prompt $x$ and suppress it from the notation.
Let $\mathcal A=\mathcal C\cup\mathcal I$ be a finite set of possible output modes, where $\mathcal C$ contains verifier-accepted modes and $\mathcal I$ contains incorrect modes.
We assign each mode $a\in\mathcal A$ an independent logit $z_a$, with
\[
    \pi_a
    :=
    \frac{e^{z_a}}
         {\sum_{a'\in\mathcal A} e^{z_{a'}}}.
\]
The total probability of a correct response is
\[
    P
    :=
    \sum_{c\in\mathcal C}\pi_c,
\]
and, for $P>0$, the distribution over modes conditional on correctness is
\[
    q_c
    :=
    \frac{\pi_c}{P},
    \qquad c\in\mathcal C.
\]
Thus $P$ measures correctness, while $q$ describes how the correct probability
mass is distributed across solution modes.

Throughout this section we study infinitesimal expected-gradient updates in the logits $z$.
Rollouts are sampled on policy, the update is evaluated at the rollout policy, and all responses receive the same sequence-level scaling.
We omit reference-KL and entropy regularization, finite-step and clipping effects, optimizer dynamics, and parameter sharing across modes.
This is intentionally a mode-level abstraction rather than a model of every detail of neural LLM training.

For a rollout group of size $G$, let
\[
    Y_1,\ldots,Y_G\sim\pi,
    \qquad
    r_i:=\mathbf 1\{Y_i\in\mathcal C\},
    \qquad
    M:=\sum_{i=1}^G r_i,
    \qquad
    \bar r:=\frac{M}{G}.
\]

\begin{lemma}[Dr.GRPO and MaxRL share the same mean direction]
\label{thm:objective-inertness}
In the mode-level model, the expected Dr.GRPO update is
\[
    \mathbb E[\widehat g_{\mathrm{Dr}}]
    =
    \frac{G-1}{G}\nabla_z P.
\]
For the practical MaxRL estimator \citep{tajwar2026maxrl},
\[
    \mathbb E[\widehat g_{\mathrm{MaxRL}}]
    =
    \frac{1-(1-P)^{G-1}}{P}\nabla_z P,
\]
with continuous extension $G-1$ at $P=0$.
Hence both expected updates are positive scalar multiples of
$\nabla_zP$ whenever $0<P<1$.
\end{lemma}

\begin{proof}
For Dr.GRPO, the mode-level advantage is
\[
    A_i^{\mathrm{Dr}}=r_i-\bar r.
\]
Using $\mathbb E[\nabla_z\log\pi(Y_i)]=0$
and
$\mathbb E[r_i\nabla_z\log\pi(Y_i)]=\nabla_zP$,
\[
\begin{aligned}
    \mathbb E[
        (r_i-\bar r)\nabla_z\log\pi(Y_i)
    ]
    &=
    \frac{G-1}{G}\nabla_zP.
\end{aligned}
\]
Averaging over the $G$ samples gives the first result.

For practical MaxRL, the all-failure group receives zero update and successful groups use the mean-reward-normalized advantage
\[
    \frac{r_i-\bar r}{\bar r}.
\]
Its expected gradient is
\[
    \frac{1-(1-P)^{G-1}}{P}\nabla_zP,
\]
the order-$(G-1)$ MaxRL gradient
\citep{tajwar2026maxrl}.
\end{proof}

Lemma~\ref{thm:objective-inertness} means that, for a fixed prompt, Dr.GRPO and practical MaxRL traverse the same mean trajectory in logit space at different speeds. More generally, consider
\begin{equation}
    \dot z
    =
    \kappa(P)\nabla_zP,
    \qquad
    \kappa(P)>0.
    \label{eq:binary-mean-flow}
\end{equation}
Under the change of time
\[
    \frac{d\tau}{dt}
    =
    \kappa(P(t)),
\]
this becomes
\[
    \frac{dz}{d\tau}
    =
    \nabla_zP.
\]
The positive scalar reweighting changes the speed of the mean dynamics but does not change the direction of the update.

Prior work analyzes this softmax-logit expected-return flow and shows that it can amplify probability differences among equally rewarded outcomes \citep{sinha2026expected,lochab2026ucpo}.
If one correct mode $c^\star$ initially has strictly greater probability than every other correct mode, the idealized flow has the winner-take-all limit
\[
    P(t)\longrightarrow1,
    \qquad
    q_{c^\star}(t)\longrightarrow1.
\]
We note that the binary reward itself does not prefer this collapsed solution: once $P=1$, every policy supported entirely on $\mathcal C$ achieves the same expected reward.
The concentration comes from the optimization dynamics, not from a distinct reward assigned to the winning mode.

\subsection{Rare modes become difficult to recover}
\label{app:theory-gradient-availability}

The same on-policy mechanism makes a mode increasingly difficult to recover once its probability becomes small.
Let mode $b$ have current probability $\pi_b$.
Its probability of appearing at least once in a group of $G$ rollouts is
\[
    1-(1-\pi_b)^G
    \le
    G\pi_b,
\]
which vanishes as $\pi_b\to0$.

The gradient signal associated with that logit also vanishes.

\begin{lemma}[Vanishing on-policy signal for a rare mode]
\label{lem:sampled-score-bound}
Suppose the detached advantages satisfy $|A_i|\le B$ almost surely.
For the logit coordinate of mode $b$,
\[
    \mathbb E
    \left|
        [\widehat g]_b
    \right|
    \le
    2B\pi_b(1-\pi_b).
\]
Hence the expected absolute score-gradient coordinate approaches zero as $\pi_b\to0$.
\end{lemma}

\begin{proof}
For a categorical softmax,
\[
    \frac{\partial}{\partial z_b}
    \log\pi(Y_i)
    =
    \mathbf 1\{Y_i=b\}-\pi_b.
\]
Therefore
\[
    [\widehat g]_b
    =
    \frac1G
    \sum_{i=1}^G
    A_i
    \bigl(
        \mathbf 1\{Y_i=b\}-\pi_b
    \bigr).
\]
Using the triangle inequality and $|A_i|\le B$,
\[
    \mathbb E|[\widehat g]_b|
    \le
    B\,
    \mathbb E
    \left|
        \mathbf 1\{Y=b\}-\pi_b
    \right|
    =
    2B\pi_b(1-\pi_b).
\]
\end{proof}

Thus both the chance of observing a rare mode and the magnitude of its on-policy score-gradient coordinate shrink with its current probability.
This is the dependence on rediscovery that replay removes.

\subsection{Replay preserves stored modes}
\label{app:theory-replay}

Let
\[
    \varnothing\neq\mathcal B\subseteq\mathcal C,
    \qquad
    k:=|\mathcal B|,
\]
be a fixed replay buffer containing one retained example from each stored solution mode.
At the mode level, uniform replay corresponds to the target distribution
\[
    \mu_{\mathcal B}(a)
    :=
    \frac1k\mathbf 1\{a\in\mathcal B\}
\]
and the negative-log-likelihood objective
\begin{equation}
    \mathcal L_{\mathrm{rep}}(z)
    :=
    -\frac1k
    \sum_{b\in\mathcal B}
    \log\pi_b.
    \label{eq:verified-replay-loss}
\end{equation}

For categorical logits,
\[
    -\nabla_z\mathcal L_{\mathrm{rep}}
    =
    \mu_{\mathcal B}-\pi.
\]
A stored mode $b\in\mathcal B$ therefore receives replay update
\[
    \lambda_{\mathrm{rep}}
    \left(
        \frac1k-\pi_b
    \right),
\]
which approaches $\lambda_{\mathrm{rep}}/k>0$ as $\pi_b\to0$.
Unlike fresh on-policy RLVR, replay does not require the current policy to sample the stored mode again.

Combining replay with the binary-reward mean flow gives
\begin{equation}
    \dot z
    =
    \kappa(P)\nabla_zP
    -
    \lambda_{\mathrm{rep}}
    \nabla_z\mathcal L_{\mathrm{rep}}(z),
    \qquad
    \lambda_{\mathrm{rep}}>0,
    \label{eq:replay-mean-flow}
\end{equation}
where $\kappa$ is continuous and bounded.

\begin{theorem}[Fixed replay keeps stored modes at positive probability]
\label{thm:replay-retention}
Suppose that from time $T$ onward the nonempty buffer $\mathcal B$ is fixed and receives constant replay weight $\lambda_{\mathrm{rep}}>0$.
Then there exists $\epsilon>0$ such that
\[
    \pi_b(t)\ge\epsilon
    \qquad
    \text{for every }b\in\mathcal B
    \text{ and every }t\ge T.
\]
Thus no stored mode can vanish under the fixed-buffer mean flow.
\end{theorem}

\begin{proof}
Let
\[
    F(P)
    :=
    \int_0^P \kappa(s)\,ds
\]
and define
\[
    \mathcal E(z)
    :=
    \lambda_{\mathrm{rep}}
    \mathcal L_{\mathrm{rep}}(z)
    -
    F(P(z)).
\]
Equation~\eqref{eq:replay-mean-flow} is exactly
\[
    \dot z
    =
    -\nabla_z\mathcal E(z),
\]
so
\[
    \frac{d}{dt}\mathcal E(z(t))
    =
    -\|\dot z(t)\|_2^2
    \le0.
\]
Because $P\in[0,1]$ and $\kappa$ is bounded,
\[
    \mathcal L_{\mathrm{rep}}(z(t))
    \le
    \mathcal L_{\mathrm{rep}}(z(T))
    +
    \frac{\|\kappa\|_\infty}
         {\lambda_{\mathrm{rep}}}
    =:
    C_T
    <
    \infty.
\]
Equation~\eqref{eq:verified-replay-loss} then gives
\[
    \sum_{b\in\mathcal B}
    -\log\pi_b(t)
    \le
    kC_T.
\]
Each term is nonnegative, so
\[
    -\log\pi_b(t)
    \le
    kC_T
    \qquad
    \text{for every }b\in\mathcal B.
\]
Therefore
\[
    \pi_b(t)
    \ge
    e^{-kC_T}.
\]
Taking $\epsilon=e^{-kC_T}$ proves the claim.
\end{proof}

\subsection{Scope of the analysis}
\label{app:theory-evidence-map}

Our theorem gives a retention guarantee for modes already present in a fixed buffer, not a guarantee that replay discovers unseen modes.
A strong, discovery-oriented objective is still needed to find new solution modes.

The experiments test the corresponding observable predictions directly:
(1) RLVR-only training reduces correct-solution diversity (App.~\ref{app:pre-rl-regression});
(2) rare or uniform-reward rollout groups provide little fresh task-gradient signal (App.~\ref{app:degenerate-groups}); and
(3) replay preserves more conditional solution diversity than matched controls under fresh evaluation (App.~\ref{app:fresh-concentration}).

\section{Compute and Storage Overhead}
\label{app:replay-cost}

\subsection{Matched schedule and compute controls}
\label{app:compute-matching}

Replay and its matched control share initialization, training data, optimizer settings, 3,072 optimizer updates, and $G=16$ fresh rollouts per visited prompt.
Both execute the same validation, buffer construction, scheduling, and exemplar-scoring machinery, with the replay gradient set to zero in the control.
Replay therefore adds the stored-exemplar learning signal and its scoring cost without adding optimizer steps or fresh rollouts.
We report the measured overhead below and separately compare against no-replay controls with a 50\% larger fresh-rollout budget, allocated either to more updates or wider groups (Apps.~\ref{app:compute-budget-control} and~\ref{app:compute-budget-groups}).

\subsection{Storage and runtime overhead}
\label{app:compute-accounting}

Replay adds storage for retained exemplars and computation for scoring them.
Across the 75 replay runs in the main factorial, the largest estimated mean replay-buffer footprint is 444~kB.
For comparison, storing a separate full reference model would require approximately 0.99--6.17~GB across the three tested model scales.
These values describe component storage rather than peak accelerator memory.

We measure runtime directly for the Qwen2.5-0.5B Level-1 control and Re:Dr runs executed on the same two A5000 nodes with matched prompts and seeds.
Across 25 domain--seed pairs, replay increases learner time by 2.4\% and whole-job wall-clock time by 1.1\%.
These end-to-end measurements include system overhead and show that replay adds modest runtime in this implementation.

The longer-horizon control below provides a substantially larger compute perturbation.
Relative to the standard 8-pass control, the 12-pass control uses 44.9\% more learner time and 44.7\% more wall-clock time.

\subsection{Longer-horizon control}
\label{app:compute-budget-control}

We test whether additional training alone closes the replay gap by extending the no-replay control from 8 to 12 passes.
The longer control receives 4,608 optimizer updates instead of 3,072 and 50\% more fresh rollouts, while using the same prompts, seeds, hardware, optimization settings, and evaluation procedure.

Additional training does not close either gap.
Averaged across the five domains, control \texttt{pass@8} changes from $.444$ after 8 passes to $.434$ after 12 passes, compared with $.748$ for Re:Dr after 8 passes.
Mean \pmd{} changes from $.004$ to $.013$, compared with $.376$ for Re:Dr.
Thus the substantially longer no-replay run remains near one effective verified solution mode and remains far below replay on both correctness and conditional diversity.

\texttt{Python} makes the contrast especially clear.
Only $1.5\%$ of its control updates contain a nonzero fresh task gradient (App.~\ref{app:degenerate-groups}), and increasing the horizon reduces terminal \texttt{pass@8} from $.338$ to $.182$.
At 12 passes, \texttt{Python} produces too few prompts with multiple verified responses to estimate \pmd{} under the reporting threshold.

\subsection{Wider-group control}
\label{app:compute-budget-groups}

The longer-horizon control spends a 50\% larger fresh-rollout budget on additional updates.
We test the complementary allocation by increasing the no-replay group size from $G=16$ to $G=24$ while holding the training horizon fixed at 3,072 updates.
This gives the wider-group control the same total number of fresh rollouts as the 12-pass control, but uses them to sample more candidates within each update rather than to perform more updates.

Wider groups do not close the replay gap.
Averaged across the five domains, terminal \pmd{} is $.004$ for the $G=16$ control, $.008$ for the $G=24$ control, and $.376$ for Re:Dr.
No domain-level \pmd{} change from $G=16$ to $G=24$ exceeds $.016$, and every corresponding interval includes zero.
Mean \texttt{pass@8} is $.444$ for the $G=16$ control and $.435$ for the $G=24$ control, compared with $.748$ for Re:Dr.
In \texttt{Python}, \texttt{pass@8} decreases from $.338$ to $.192$.

Wider groups modestly increase the fraction of updates containing both correct and incorrect responses, from $.161$ to $.185$ on the five-domain mean.
The increase is small in \texttt{Python}, from $.020$ to $.023$, and largest in \texttt{MathIR}, from $.351$ to $.442$.
Thus, increasing fresh sampling within each update does not substantially restore task-gradient signal in the setting where the control is most repetitive.

Together with the longer-horizon control, this experiment shows that allocating a 50\% larger fresh-rollout budget either to more updates or to wider rollout groups closes neither the correctness nor the diversity gap to replay.

\subsection{Single-slot replay}
\label{app:mode-agnostic-replay}

Re:Dr both rehearses previously verified successes and balances rehearsal across the different solution modes discovered for each prompt.
To separate these effects, we train a single-slot variant that is identical to Re:Dr except that each prompt retains only its first verified solution mode.
The single-slot condition therefore provides replay of a past success without the ability to balance rehearsal across multiple stored modes.

Replay of a single success accounts for a substantial part of the observed effect.
Averaged across the five domains, terminal \pmd{} is $.004$ for the matched control, $.218$ for single-slot replay, and $.372$ for Re:Dr.
Single-slot replay recovers 58\% of the aggregate Re:Dr--control diversity improvement.
Mean \texttt{pass@8} similarly increases from $.444$ for the control to $.707$ for single-slot replay, compared with $.796$ for Re:Dr.
Thus, retaining and rehearsing even one verified success can preserve additional solution breadth without balancing rehearsal across multiple stored modes.

Balancing across discovered modes provides a further diversity gain.
Across the 25 matched domain--seed cells, Re:Dr exceeds single-slot replay by $.155$ \pmd{} with a 95\% interval of $[+.071,+.239]$.
The additional balancing benefit is largest in \texttt{Graph} and \texttt{Countdown}, with smaller domain-specific estimates elsewhere.
Together, these results distinguish two contributions of replay: retaining past successes provides a substantial benefit, while balancing rehearsal across multiple discovered solution modes provides a further aggregate diversity gain.

\clearpage
\phantomsection
\begin{center}\large\textbf{Part VII. Related-method comparison}\end{center}
\addcontentsline{toc}{part}{Part VII. Related-method comparison}

\section{Comparison with Diversity-Preserving RLVR Methods}
\label{app:related-extended}

Diversity-preserving RLVR methods differ in what they measure, which responses they update, and what distribution they encourage.
Clip-Cov and KL-Cov~\citep{cui2025entropy} address token-level entropy collapse by modifying token updates or applying selective KL regularization.
UCPO~\citep{lochab2026ucpo} redistributes advantage among correct responses in the current rollout group.
DMPO~\citep{li2026distribution} uses group-level distribution matching between reward-derived targets and sampled responses.
Verified replay instead stores previously discovered verified solution modes and rehearses one exemplar per stored mode.

\paragraph{Different notions of diversity.}
These methods do not measure interchangeable quantities.
Token entropy measures uncertainty over generated tokens, while other methods evaluate diversity over sampled responses, equations, or feasible solutions.
\mb{} instead defines solution-mode identity through executable verification and measures diversity among correct solution modes.
Multiple response strings can therefore represent the same \mb{} solution mode, while token-level diversity can also arise from incorrect responses.
Our conditional-diversity metric $\pmd(q)=1-\sum_c q_c^2$ depends only on how probability is allocated among correct solution modes.
By contrast, \texttt{pass@}$K$ depends on total correctness probability, while \texttt{distinct@}$K$ depends on both correctness and the allocation among correct modes.

\paragraph{Fresh rollouts versus stored replay.}
Clip-Cov, KL-Cov, UCPO, and DMPO operate on responses available in the current rollout group.
Verified replay can directly update a previously discovered solution mode even when that mode is absent from the current group because its exemplar remains in the buffer.
App.~\ref{app:theory-gradient-availability} formalizes this distinction in the solution-mode model: the expected on-policy signal for a mode vanishes as its probability approaches zero, while a stored replay target continues to receive a direct likelihood signal.
This does not imply that fresh-rollout methods cannot preserve diversity, since shared parameters, regularization, and later samples can also affect an unsampled mode.
The narrower distinction is that replay does not require rediscovery before a stored solution mode can receive another direct update.

\paragraph{Different targets.}
The methods also differ in what they attempt to balance.
UCPO favors a more even allocation among correct responses represented in the current group, while DMPO matches sampled responses to a reward-derived target.
Verified replay instead gives equal replay weight to the discovered solution modes represented in its buffer.
Uniformity over response strings or sampled rows therefore need not imply uniformity over executable solution modes.
Our replay-weighting ablation isolates this distinction within the same replay procedure: weighting stored modes by fresh discovery frequency preserves fewer verified alternatives than weighting them uniformly (App.~\ref{app:frequency-replay-progress}).
The single-slot ablation separately shows that retaining one verified success provides a substantial benefit even without multi-mode balancing, while balancing multiple discovered modes contributes additional diversity (App.~\ref{app:mode-agnostic-replay}).

\paragraph{Evaluated comparators.}
The direct \mb{} comparator suite includes UCPO and sparse RLEP-Dr at the smaller model scales (App.~\ref{app:ucpo-progress}), together with fixed Semantic-MaxEnt, GAPO, and SetPO (Apps.~\ref{app:semantic-maxent} and~\ref{app:diversity-comparators}).
Semantic-MaxEnt adds a frequency-based bonus to fresh verified responses, while sparse RLEP-Dr reuses verified trajectories without balancing replay across solution modes.
The online-pool RLEP-Dr analysis further separates experience coverage from the replay rule: learner-generated experience recovers part, but not all, of the diversity gap to Re:Dr (App.~\ref{app:rlep-online-pool}).
Self-imitation learning~\citep{oh2018self} and rejection-sampling fine-tuning~\citep{dong2023raft} likewise revisit successful model outputs, placing verified replay broadly in the same family.
Its distinguishing feature is how stored experience is organized: one exemplar is retained per discovered solution mode and stored modes are rehearsed with equal weight.

Clip-Cov, KL-Cov, and DMPO are included as mechanism-level literature comparisons because their reported benchmarks, metrics, and training setups differ from \mb{}.
The direct \mb{} experiments evaluate UCPO, RLEP-Dr, Semantic-MaxEnt, GAPO, and SetPO under the executable solution-mode framework used in this paper.

\clearpage
\phantomsection
\begin{center}\large\textbf{Part VIII. Data and reproducibility}\end{center}
\addcontentsline{toc}{part}{Part VIII. Data and reproducibility}

\section{Data and Reproducibility}
\label{app:reproducibility}

\subsection{Evaluation pairing}
\label{app:source-integrity}

Paired method comparisons match model, domain, training seed, and training endpoint.
Comparisons to initialization use each method's corresponding initial evaluation.
In the fresh 64-response evaluation, each training seed is paired with one initial sampling replica shared across the four methods, and all five replicas use the same initial model weights.
Terminal replay effects compare trained checkpoints on matched evaluation populations.
Each reported comparison states its evaluation population, sample size, and conditional-diversity eligibility rule.

\subsection{Released datasets and checkpoints}
\label{app:repro-artifacts}
\label{app:repro-code}
\label{app:archived-models}

The \href{https://github.com/liv-daliberti/modeBench}{\mb{} repository}
contains the benchmark implementation and documentation, while the
\href{https://huggingface.co/datasets/od2961/ModeBench}{frozen \mb{} datasets}
contain the released benchmark splits.
The dataset card includes the
\href{https://huggingface.co/datasets/od2961/ModeBench/blob/main/README.md#loading}{loading guide},
and the
\href{https://huggingface.co/datasets/od2961/ModeBench/blob/main/EVALUATION.md}{evaluation guide}
documents executable verification and solution-mode identity.
The \href{https://github.com/liv-daliberti/remax}{Re:Max repository}
contains the replay and training implementation.
The \href{https://huggingface.co/od2961/maxent-grpo-models}{model and research-artifact archive}
contains released checkpoints organized by model, domain, method, and seed,
together with their configurations, tokenizers, and applicable base-model licenses.

\subsection{Analysis resources}
\label{app:repro-recovery}

The
\href{https://huggingface.co/od2961/maxent-grpo-models/blob/main/EXPERIMENTS.md}{experiment index}
maps the main factorials, scale comparisons, difficulty experiments,
replay-weighting ablation, and supporting controls to their corresponding runs.
The
\href{https://huggingface.co/od2961/maxent-grpo-models/blob/main/DATA.md}{analysis-input index}
documents the numerical inputs used in the analyses.
The
\href{https://huggingface.co/od2961/maxent-grpo-models/blob/main/RESTORE.md}{model-loading guide}
provides instructions for restoring released checkpoints.
Figure inputs and rendering scripts are listed in \mbox{\path{paper/FIGURE_MANIFEST.md}}.

\end{document}